%% file: main.tex
\documentclass{article}

\input{math_commands}
\usepackage{microtype}
\usepackage{graphicx}
\usepackage{subcaption}
\usepackage{booktabs} 
\usepackage{ragged2e}
\usepackage{hyperref}

\usepackage[accepted]{icml2024}

\usepackage{amsmath}
\usepackage{amssymb}
\usepackage{mathtools}
\usepackage{amsthm}

\usepackage[capitalize,noabbrev]{cleveref}

\theoremstyle{plain}
\newtheorem{theorem}{Theorem}[section]
\newtheorem{proposition}[theorem]{Proposition}
\newtheorem{lemma}[theorem]{Lemma}

\theoremstyle{definition}
\newtheorem{definition}[theorem]{Definition}

\theoremstyle{remark}

\usepackage[textsize=tiny]{todonotes}

\icmltitlerunning{Is Temperature Sample Efficient for Softmax Gaussian Mixture of Experts?}

\begin{document}

\twocolumn[
\icmltitle{Is Temperature Sample Efficient for Softmax Gaussian Mixture of Experts?}



\icmlsetsymbol{equal}{*}

\begin{icmlauthorlist}
\icmlauthor{Huy Nguyen}{sds}
\icmlauthor{Pedram Akbarian}{ece}
\icmlauthor{Nhat Ho}{sds}
\end{icmlauthorlist}

\icmlaffiliation{sds}{Department of Statistics and Data Sciences}
\icmlaffiliation{ece}{Department of Electrical and Computer Engineering, The University of Texas at Austin}
\icmlcorrespondingauthor{Huy Nguyen}{huynm@utexas.edu}

\icmlkeywords{Machine Learning, ICML}

\vskip 0.3in
]



\printAffiliationsAndNotice{} 

\begin{abstract}
    Dense-to-sparse gating mixture of experts (MoE) has recently become an effective alternative to a well-known sparse MoE. Rather than fixing the number of activated experts as in the latter model, which could limit the investigation of potential experts, the former model utilizes the temperature to control the softmax weight distribution and the sparsity of the MoE during training in order to stabilize the expert specialization. Nevertheless, while there are previous attempts to theoretically comprehend the sparse MoE, a comprehensive analysis of the dense-to-sparse gating MoE has remained elusive. Therefore, we aim to explore the impacts of the dense-to-sparse gate on the maximum likelihood estimation under the Gaussian MoE in this paper. We demonstrate that due to interactions between the temperature and other model parameters via some partial differential equations, the convergence rates of parameter estimations are slower than any polynomial rates, and could be as slow as $\mathcal{O}(1/\log(n))$, where $n$ denotes the sample size. To address this issue, we propose using a novel activation dense-to-sparse gate, which routes the output of a linear layer to an activation function before delivering them to the softmax function. By imposing linearly independence conditions on the activation function and its derivatives, we show that the parameter estimation rates are significantly improved to polynomial rates. Finally, we conduct a simulation study to empirically validate our theoretical results.
\end{abstract}

\vspace{-0.5em}
\section{Introduction}
\label{sec:introduction}
Mixture of experts (MoE) \cite{Jacob_Jordan-1991,Jordan-1994} is a statistical machine learning framework that aggregates the power of multiple expert networks using softmax as a gating function (weight function) to create a more sophisticated model than a single network. To scale up the model capacity (the number of model parameters) given a fixed computational cost, \cite{shazeer2017topk} have introduced a sparse variant of the MoE model, which turns on only one or a few experts for each input. Thanks to its scalability, sparse MoE models have been widely used in several applications, namely large language models \cite{lepikhin_gshard_2021,Du_Glam_MoE,fedus2021switch,zhou2023brainformers,jiang2024mixtral}, computer vision \cite{dosovitskiy_image_2021,liang_m3vit_2022,Riquelme2021scalingvision}, multi-task learning \cite{hazimeh_dselect_k_2021} and speech recognition \cite{gulati_conformer_2020,You_Speech_MoE}. 


In sparse MoE models, gating functions are concurrently trained with expert networks to route inputs effectively. However, early in training, gating networks may exhibit instability in expert selection due to their inexperience. Additionally, pre-determining the number of activated experts per input can limit exploration of potential experts. To address this, \cite{nie2022evomoe} introduced a dense-to-sparse gate, initially routing inputs to all experts and gradually becoming sparser. This approach strategically controls the temperature of a softmax-based gating function to adjust weight distribution among experts and regulate sparsity in MoE models, promoting stability in expert specialization.

From a theoretical perspective, there have been attempts to comprehend the properties of MoE models. Firstly, \cite{chen2022theory} studied how sparse MoE
layers enhanced the efficacy of neural network learning and explained why they would not collapse into a single model. Another line of work tried to understand the effects of gating functions on the convergence rates of maximum likelihood estimation under Gaussian MoE models. In particular, when the gating function was independent of input, \cite{ho2022gaussian} showed an interaction among expert parameters, which made those rates inversely proportional to the number of over-specified experts. Next, \cite{nguyen2023demystifying} considered a dense softmax gating function, and demonstrated that parameter estimation rates were determined by the solvability of an intricate system of polynomial equations due to another interaction between gating and expert parameters. Subsequently, \cite{nguyen2024statistical} explored a general Top-K sparse softmax gating function. They revealed that activating only one expert, i.e., $K=1$, makes the previous interaction between expert and gating parameters disappear, and thus, improved the parameter estimation rates significantly. However, a comprehensive theoretical analysis of the dense-to-sparse gate has remained missing in the literature.

In this work, we focus on investigating whether the temperature in the dense-to-sparse gate is sample efficient or not under the parameter estimation problem of the Gaussian MoE model. For that purpose, we now present the formulation of that model formally.

\textbf{Problem setup.} Suppose that the data $\{(X_i,Y_i)\}_{i=1}^{n}\subset\mathbb{R}^d\times\mathbb{R}$ are i.i.d sampled from the dense-to-sparse gating Gaussian mixture of experts, which is associated with the conditional density function $g_{G_*}(Y|X)$ defined as:
\begin{align}
    \label{eq:standard_density}
    \sum_{i=1}^{k_*}\softmax\Big(\dfrac{(\beta^*_{1i})^{\top}X+\beta^*_{0i}}{\tau^*}\Big)\cdot f(Y|(a^*_i)^{\top}X+b^*_i,\nu^*_i),
\end{align}
where $G_*:=\sum_{i=1}^{k_*}\exp(\bzi/\tau^*)\delta_{(\boi,\tau^*,\ai,\bi,\vi)}$ is a true but unknown \emph{mixing measure} (i.e., a weighted sum of Dirac measures $\delta$) associated with true parameters $(\bzi,\boi,\tau^*,\ai,\bi,\vi)$ for $i\in\{1,2,\ldots,k_*\}$. Here, $\tau^*$ is the softmax temperature which adjusts the sparsity of the MoE models. When $\tau^*$ increases, the weight distribution becomes more uniform. On the other hand, when $\tau^*$ approaches zero, that distribution turns into one-hot. Meanwhile, $f(\cdot|\mu,\nu)$ stands for a univariate Gaussian density function with mean $\mu$ and variance $\nu$. For ease of presentation, we consider $k_*$ linear experts of the form $a^{\top}X+b$ as the results for general expert settings, including deep neural network, can be achieved in a similar fashion. Additionally, we define for any vector $v=(v_1,\ldots,v_{k_*})\in\mathbb{R}^{k_*}$ that $\softmax(v_i):={\exp(v_i)}/{\sum_{j=1}^{k_*}\exp(v_j)}$.

\textbf{Maximum likelihood estimation.} To estimate the parameters of model~\eqref{eq:standard_density}, we propose using the maximum likelihood method as follows:
\begin{align}
    \label{eq:MLE}
    \widehat{G}_n:=\argmax_{G}\frac{1}{n}\sum_{i=1}^{n}\log(g_{G}(Y_i|X_i)).
\end{align}
When the true number of expert $k_*$ is known (\emph{exact-specified settings}), the above maximum is taken over the set of all mixing measures of order $k_*$ denoted by $\mathcal{E}_{k_*}(\Theta):=\{G=\sum_{i=1}^{k_*}\exp(\beta_{0i}/\tau)\delta_{(\beta_{1i},\tau,a_i,b_i,\nu_i)}:(\beta_{0i},\beta_{1i},\tau,a_i,b_i,\nu_i)\in\Theta\}$. 
Conversely, when $k_*$ is unknown and the true model~\eqref{eq:standard_density} is over-specified by a Gaussian mixture of $k$ experts where $k>k_*$ (\emph{over-specified settings}), the maximum is subject to the set of all mixing measures of order at most $k$, i.e., $\mathcal{G}_{k}(\Theta):=\{G=\sum_{i=1}^{k'}\exp(\beta_{0i}/\tau)\delta_{(\beta_{1i},\tau,a_i,b_i,\nu_i)}: 1\leq k'\leq k, \ (\beta_{0i},\beta_{1i},\tau,a_i,b_i,\nu_i)\in\Theta\}$.

\textbf{Assumptions.} In our analysis, we have four main assumptions on the parameters: 

\textbf{(A.1)} The parameter space $\Theta$ is a compact subset of $\mathbb{R}\times\mathbb{R}^d\times\mathbb{R}_+\times\mathbb{R}^d\times\mathbb{R}\times\mathbb{R}_+$, and the input space $\mathcal{X}$ is bounded; 

\textbf{(A.2)} $\beta^*_{1k_*}=\zerod$ and $\beta^*_{0k_*}=0$; 

\textbf{(A.3)} $(a^*_1,b^*_1,\nu^*_1),\ldots,(a^*_{k_*},b^*_{k_*},\nu^*_{k_*})$ are pairwise distinct; 

\textbf{(A.4)} At least one among $\beta^*_{11},\ldots,\beta^*_{1k_*}$ is non-zero. 

Above, the first assumption helps ensure the convergence of parameter estimation, while the second guarantees that the dense-to-sparse gating Gaussian MoE model is identifiable (see Appendix~\ref{sec:identifiability}). Next, the third one is to make experts in model~\eqref{eq:standard_density} pairwise distinct. Finally, the last assumption makes sure that the gating function hinges on the input $X$.

\textbf{Technical challenges.} The softmax temperature leads to two fundamental challenges in theory: 

\textbf{(C.1) Temperature's interaction with other parameters.} To establish a parameter estimation rate given a density estimation rate, we need to decompose the density discrepancy $g_{\widehat{G}_n}(Y|X)-g_{G_*}(Y|X)$ into a combination of linearly independent terms. This can be done by applying Taylor expansions to the softmax's numerator $F(Y|X,\omega):=\exp(\frac{\beta_{1}^{\top}X}{\tau})f(Y|a^{\top}X+b,\nu)$, where $\omega:=(\beta_{1},\tau,a,b,\nu)$. However, we realize that the temperature interacts with both gating and expert parameters via two following partial differential equations (PDEs), which induce a number of linearly dependent terms:
\begin{align}
    \label{eq:PDE_1}
    \frac{\partial F}{\partial\tau}&=\frac{1}{\tau}\cdot\beta_{1}^{\top}\frac{\partial F}{\partial\beta_1},\\
    \label{eq:PDE_2}
    \frac{\partial^2 F}{\partial\tau~\partial b}&=\frac{1}{\tau^2}\cdot\beta_{1}^{\top}\frac{\partial F}{\partial a}.
\end{align}
Intuitively, the first PDE reveals that there is an intrinsic interaction between the temperature $\tau$ and the gating parameter $\beta_1$. Meanwhile, the second PDE indicates that the temperature also interacts with the expert parameters $a$ and $b$. Although parameter interactions expressed in the language of PDEs have been observed in \cite{nguyen2023demystifying}, the structures of the above interactions are 
Furthermore, these interactions are
substantially more serious than those in \cite{ho2022gaussian,nguyen2023demystifying,nguyen2024statistical}. More specifically, we will show in Section~\ref{sec:standard_gate} that due to the above PDEs, the parameter estimation rates are slower than any polynomial rates, and thus, could be as slow as $1/\log(n)$, where $n$ denotes the sample size. Such phenomenon has never been observed in previous work. 

\textbf{(C.2) Rate improvement.} From the previous observation, it is essential to propose a method to accelerate the parameter estimation rates. To enhance slow rates caused by the interaction between gating and expert parameters, \cite{nguyen2024general} suggested transforming the inputs using a 'modified' function $M$, e.g. $\log(|\cdot|)$, $\cos(\cdot)$, prior to delivering them to the gating network, i.e. $\softmax\Big(\frac{(\beta^*_{1i})^{\top}M(X)+\beta^*_{0i}}{\tau^*}\Big)$. However, it can be verified that the PDEs~\eqref{eq:PDE_1} and \eqref{eq:PDE_2} still holds true with the corresponding function $F(Y|X,\omega):=\exp(\frac{\beta_{1}^{\top}M(X)}{\tau})f(Y|a^{\top}X+b,\nu)$. Therefore, we have to come up with a novel solution in this work.

\begin{table*}[!ht]
\caption{Summary of density estimation rates and parameter estimation rates under the (activation) dense-to-sparse gating Gaussian MoE. In this table, the function $\bar{r}(\cdot)$ represents for the solvability of the system of polynomial equations~\eqref{eq:system_r_bar} with $\bar{r}(2)=4$ and $\bar{r}(3)=6$. Additionally, $\mathcal{A}_j$ denotes a Voronoi cell given in equation~\eqref{eq:Voronoi_cells}.}
\textbf{}\\
\centering
\begin{tabular}{ | c | c | c |c|c|c|} 
\hline
\multicolumn{6}{|c|}{\textbf{Dense-to-sparse gating Gaussian MoE}}\\
\hline
\textbf{Setting}  & $g_{G_{*}}(Y|X)$ & $\beta^*_{1j},\tau^*$ & $a^*_{j}$& $b_{j}^{*}$ & $\nu_{j}^{*}$\\
\hline 
{Exact-specified}  & $\widetilde{\mathcal{O}}(n^{-1/2})$ & Slower than $\widetilde{\mathcal{O}}(n^{-1/2r}), \forall r\geq 1$ &$\widetilde{\mathcal{O}}(n^{-1/2})$ &$\widetilde{\mathcal{O}}(n^{-1/2})$ &$\widetilde{\mathcal{O}}(n^{-1/2})$ \\
\hline
{Over-specified}  &$\widetilde{\mathcal{O}}(n^{-1/2})$  & \multicolumn{2}{c|}{Slower than $\widetilde{\mathcal{O}}(n^{-1/2r}),\forall r\geq 1$}  & $\widetilde{\mathcal{O}}(n^{-1/2\bar{r}(|\mathcal{A}_{j}|)})$ &$\widetilde{\mathcal{O}}(n^{-1/\bar{r}(|\mathcal{A}_{j}|)})$ \\
\hline
\multicolumn{6}{|c|}{\textbf{Activation Dense-to-sparse gating Gaussian MoE}}\\
\hline
\textbf{Setting}  & $p_{G_{*}}(Y|X)$ & $\beta^*_{1j},\tau^*$ & $a^*_{j}$& $b_{j}^{*}$ & $\nu_{j}^{*}$\\
\hline 
{Exact-specified}  & $\widetilde{\mathcal{O}}(n^{-1/2})$ & \multicolumn{4}{c|}{$\widetilde{\mathcal{O}}(n^{-1/2})$} \\
\hline
{Over-specified}  &$\widetilde{\mathcal{O}}(n^{-1/2})$  & \multicolumn{2}{c|}{$\widetilde{\mathcal{O}}(n^{-1/4})$}  & $\widetilde{\mathcal{O}}(n^{-1/2\bar{r}(|\mathcal{A}_{j}|)})$ &$\widetilde{\mathcal{O}}(n^{-1/\bar{r}(|\mathcal{A}_{j}|)})$ \\
\hline
\end{tabular}
\label{table:parameter_rates}
\end{table*}

\textbf{Main contributions.} In this paper, we conduct a convergence analysis of density and parameter estimations under the dense-to-sparse gating Gaussian MoE. Our contributions are two-fold and can be summarized as follows:

\textbf{1. Dense-to-sparse gating function:} Equipped with this gating function, we first establish the convergence rate of density estimation under the Total Variation distance $\bbE_X[V(g_{\widehat{G}_n}(\cdot|X),g_{G_*}(\cdot|X))]=\widetilde{\mathcal{O}}(n^{-1/2})$, which is parametric on the sample size $n$. Given this result, we then demonstrate that under the exact-specified settings, the estimation rates for $\beta^*_{1i},\tau^*$ are slower than any polynomial rates owing to the PDE~\eqref{eq:PDE_1}, and therefore, could be $1/\log(n)$.  Meanwhile, those for $a^*_{i},b^*_{i},\nu^*_{i}$ are significantly faster, standing at $\widetilde{\mathcal{O}}(n^{-1/2})$. Under the over-specified settings, we show that the rates for estimating $\beta^*_{1i},\tau^*$ remain unchanged, whereas that for $a^*_{i}$ becomes slower than any polynomial rates due to the PDE~\eqref{eq:PDE_2}. Additionally, the estimation rates for $b^*_{i},\nu^*_{i}$ depend on the solvability of a system of polynomial equations.

\textbf{2. Activation dense-to-sparse gating function.} To enhance the previous slow rates, we propose a novel class of gating functions called \emph{activation dense-to-sparse} given by $\softmax\Big(\frac{\sigma((\beta^*_{1i})^{\top}X)+\beta^*_{0i}}{\tau^*}\Big)$. Here, $\sigma(\cdot)$ is an activation function satisfying conditions in Definition~\ref{def:modified_function_exact} (resp. Definition~\ref{def:modified_function_over}), which make the interactions of temperature with other parameters in equations~\eqref{eq:PDE_1} and \eqref{eq:PDE_2} vanish under the exact-specified (resp. over-specified) settings. As a consequence, we rigorously prove that $\beta^*_{1i}$, $\tau^*$ and $a^*_{i}$ share the same considerably improved estimation rates of orders $\widetilde{\mathcal{O}}(n^{-1/2})$ and $\widetilde{\mathcal{O}}(n^{-1/4})$ under those settings, respectively.

\textbf{Outline.} The paper proceeds as follows. In Section~\ref{sec:standard_gate}, we derive the convergence rates of density estimation and parameter estimation under the dense-to-sparse gating Gaussian MoE. Subsequently, we carry out the previous analysis for the Gaussian MoE with the novel activation dense-to-sparse gate in Section~\ref{sec:modified_gate}. Then, we run some numerical experiments in Section~\ref{sec:implications} to empirically verify our theoretical results before concluding the paper in Section~\ref{sec:conclusion}. Finally, rigorous proofs and further details of experiments are provided in the supplementary material.

\textbf{Notations.} We denote $[n]: = \{1, 2, \ldots, n\}$ for any positive integer $n$. Additionally, the notation $|S|$ indicates the cardinality of any set $S$. For any vectors $v:=(v_1,v_2,\ldots,v_d) \in \mathbb{R}^{d}$ and $\alpha:=(\alpha_1,\alpha_2,\ldots,\alpha_d)\in\mathbb{N}^d$, we let $v^{\alpha}=v_{1}^{\alpha_{1}}v_{2}^{\alpha_{2}}\ldots v_{d}^{\alpha_{d}}$, $|v|:=v_1+v_2+\ldots+v_d$ and $\alpha!:=\alpha_{1}!\alpha_{2}!\ldots \alpha_{d}!$, while $\|v\|$ stands for its $2$-norm value. Given any two positive sequences $\{a_n\}_{n\geq 1}$ and $\{b_n\}_{n\geq 1}$, we write $a_n = \mathcal{O}(b_n)$ or $a_{n} \lesssim b_{n}$ if $a_n \leq C b_n$ for all $ n\in\mathbb{N}$, where $C > 0$ is some universal constant. Furthermore, we write $a_n = \widetilde{\mathcal{O}}(b_n)$ to indicate $a_{n} \lesssim b_{n}$ up to some logarithmic factors. Finally, for any two probability density functions $p,q$ dominated by the Lebesgue measure $\mu$, we denote $h(p,q) = \Big(\frac 1 2 \int (\sqrt p - \sqrt q)^2 d\mu\Big)^{1/2}$ as their Hellinger distance and $V(p,q) = \frac 1 2 \int |p-q| d\mu$ as their Total Variation distance.


\vspace{-0.5em}
\section{Dense-to-sparse Gating Function}
\label{sec:standard_gate}
In this section, we characterize the density and parameter estimation rates for the dense-to-sparse gating Gaussian MoE under both the exact-specified and over-specified settings.

We start with providing the convergence rate of the density estimation $g_{\widehat{G}_n}$ to the true density $g_{G_*}$ under the Total Variation distance in the following theorem:

\begin{theorem}
    \label{theorem:density_rate}
    Under the Total Variation distance, the density estimation $g_{\widehat{G}_n}(Y|X)$ converges to the true density $g_{G_*}(Y|X)$ at the following rate:
    \begin{align*}
        \bbE_X[V(g_{\widehat{G}_n}(\cdot|X),g_{G_*}(\cdot|X))]=\widetilde{\mathcal{O}}(n^{-1/2}).
    \end{align*}
\end{theorem}
We leverage fundamental results on density estimation for M-estimator in \cite{vandeGeer-00} to prove Theorem~\ref{theorem:density_rate} in Appendix~\ref{appendix:density_rate}.  It follows from the above bound that the density estimation rate is parametric on the sample size $n$. This results also indicates that if the Total Variation lower bound $\bbE_X[V(g_{\widehat{G}_n}(\cdot|X),g_{G_*}(\cdot|X))]\gtrsim\mathcal{D}(\widehat{G}_n,G_*)$, where  $\mathcal{D}$ is some loss function among parameters, then we obtain the parameter estimation rate $\mathcal{D}(\widehat{G}_n,G_*)=\widetilde{\mathcal{O}}(n^{-1/2})$. Now, we are ready to precisely capture those rates under the exact-specified and over-specified settings in Section~\ref{sec:standard_exact} and Section~\ref{sec:standard_over}, respectively.

\vspace{-0.5em}
\subsection{Exact-specified Settings}
\vspace{-0.5em}
\label{sec:standard_exact}
Before diving deeper into the parameter estimation problem under the exact-specified settings, let us introduce a notion of Voronoi cells \cite{manole22a}, which are then used to construct our loss functions.

\textbf{Voronoi cells.} Assume that a mixing measure $G$ has $k'$ components $\omega_i:=(\beta_{1i},\tau,a_i,b_i,\nu_i)$. Then, we distribute these components to the Voronoi cells $\mathcal{A}_j\equiv\mathcal{A}_j(G)$ generated by the components $\omega^*_j:=(\beta^*_{1j},\tau^*,a^*_j,b^*_j,\nu^*_j)$ of $G_*$, which are defined as
\begin{align}
    \label{eq:Voronoi_cells}
    \mathcal{A}_j:=\{i\in[k']:\|\omega_i-\omega^*_j\|\leq\|\omega_i-\omega^*_{\ell}\|,\forall \ell\neq j\}.
\end{align}
For instance, since the MLE $\widehat{G}_n$ has $k_*$ components under this setting, each Voronoi cell $\mathcal{A}_{j}(\widehat{G}_n)$ has exactly one element when the sample size $n$ is sufficiently large.

\textbf{Voronoi loss.} Let us define $K_{ij}(\kappa_1,\kappa_2,\kappa_3,\kappa_4,\kappa_5):=\norm{\Delta \beta_{1ij}}^{\kappa_1}+|\Delta\tau|^{\kappa_2}+\norm{\Delta a_{ij}}^{\kappa_3}+|\Delta b_{ij}|^{\kappa_4}+|\Delta\nu_{ij}|^{\kappa_5}$, where $\Delta\beta_{1ij}:=\beta_{1i}-\beta_{1j}$, $\Delta\tau:=\tau-\tau^*$, $\Delta a_{ij}:=a_{i}-a^*_{j}$, $\Delta b_{ij}:=b_{i}-b^*_{j}$ and $\Delta\nu_{ij}:=\nu_{i}-\nu^*_{j}$. Then, the Voronoi loss of interest is given by
\begin{align}
    \label{eq:loss_log_exact}
    \mathcal{D}_{1,r}(G,G_*):=
    \sum_{j=1}^{k_*}\Big|\sum_{i\in\mathcal{A}_j}\exp\Big(\frac{\beta_{0i}}{\tau}\Big)-\exp\Big(\frac{\beta^*_{0j}}{\tau^*}\Big)\Big|\nonumber\\
    +\sum_{j=1}^{k_*}\sum_{i\in\mathcal{A}_j}\exp\Big(\frac{\beta_{0i}}{\tau}\Big)K_{ij}(r,r,r,r,r)
\end{align}
Next, let us recall that when using the dense-to-sparse gate, there are two interactions of the softmax temperature $\tau$ with gating parameter $\beta_{1}$ and expert parameters $a,b$:
\begin{align}
    \label{eq:recall_PDEs}
    \frac{\partial F}{\partial\tau}=\frac{1}{\tau}\cdot\beta_{1}^{\top}\frac{\partial F}{\partial\beta_1};\quad 
    \frac{\partial^2 F}{\partial\tau~\partial b}=\frac{1}{\tau^2}\cdot\beta_{1}^{\top}\frac{\partial F}{\partial a},
\end{align}
where $F(Y|X,\omega):=\exp(\frac{\beta_{1}^{\top}X}{\tau})f(Y|a^{\top}X+b,\nu)$. Unfortunately, such interactions are so serious that the Total Variation lower bound $\bbE_X[V(g_{\widehat{G}_n}(\cdot|X),g_{G_*}(\cdot|X))]\gtrsim\mathcal{D}_{1,r}(\widehat{G}_n,G_*)$ does not hold true, and thus, we cannot achieve the bound $\mathcal{D}_{1,r}(\widehat{G}_n,G_*)=\widetilde{\mathcal{O}}(n^{-1/2})$ as discussed below Theorem~\ref{theorem:density_rate}. Instead, we show in Appendix~\ref{appendix:log_rate_exact} that 
\begin{align*}
    \inf_{G\in\mathcal{E}_{k_*}(\Theta):\mathcal{D}_{1,r}(G,G_*)\leq\varepsilon}\frac{\bbE_X[V(g_{G}(\cdot|X),g_{G_*}(\cdot|X))]}{\mathcal{D}_{1,r}(G,G_*)}\to0,
\end{align*}
as $\varepsilon\to0$, for any $r\geq 1$. This result leads to the following minimax lower bound of parameter estimation:
\begin{theorem}
    \label{theorem:log_rate_exact}
    Under the exact-specified settings, the following minimax lower bound of estimating $G_*$ holds true for any $r\geq 1$:
    \begin{align*}
        \inf_{\overline{G}_n\in\mathcal{E}_{k_*}(\Theta)}\sup_{G\in\mathcal{E}_{k_*}(\Theta)}\bbE_{g_{G}}[\mathcal{D}_{1,r}(\overline{G}_n,G)]\gtrsim n^{-1/2}.
    \end{align*}
    Here, the notation $\bbE_{g_{G}}$ indicates the expectation taken w.r.t the product measure with mixture density $g^n_{G}$.
\end{theorem}
Proof of Theorem~\ref{theorem:log_rate_exact} is in Appendix~\ref{appendix:log_rate_exact}. The above minimax lower bound suggests that the rates for estimating parameters $\beta^*_{1j},\tau^*,a^*_{j},b^*_{j},\nu^*_{j}$ are slower than any polynomial rates $\widetilde{\mathcal{O}}(n^{-1/2r})$, and therefore, could be as slow as $1/\log(n)$. This convergence behavior has never been captured in previous work on Gaussian MoE models, including \cite{ho2022gaussian,nguyen2023demystifying,nguyen2024gaussian,nguyen2024statistical}. Nevertheless, in our arguments, since the true number of experts $k_*$ is known, it is sufficient to apply the first-order Taylor expansion to the gating numerator $F$. Therefore, the second PDE in equation~\eqref{eq:recall_PDEs} should not affect the parameter estimation rates under this setting. In other words, parameters $a^*_{i},b^*_{i},\nu^*_{i}$ should enjoy faster estimation rates than their counterparts $\beta^*_{1i},\tau^*$. To illustrate this point, let us take into account another Voronoi loss function.

\textbf{Voronoi loss.} To capture the rates for estimating $a^*_{j}$, $b^*_{j}$ and $\nu^*_{j}$ more accurately, it is essential to consider the projections of previous mixing measures onto the space of those parameters $\Psi:=\mathbb{R}^d\times\mathbb{R}\times\mathbb{R}_+$. In particular, for each $G=\sum_{i=1}^{k}\exp(\beta_{0i}/\tau)\delta_{(\beta_{1i},\tau,a,b,\nu)}$, we define $G^{|\Psi}:=\sum_{i=1}^{k}\exp(\beta_{0i}/\tau)\delta_{(a_{i},b_{i},\nu_{i})}$. Then, the loss function between these projected mixing measures is given by: 
\begin{align}
    \label{eq:loss_tight_exact}
    &\mathcal{D}_2(G^{|\Psi},G^{|\Psi}_*):=\sum_{j=1}^{k_*}\Big|\sum_{i\in\mathcal{A}_j}\exp\Big(\frac{\beta_{0i}}{\tau}\Big)-\exp\Big(\frac{\beta^*_{0j}}{\tau^*}\Big)\Big|\nonumber\\
    &+\sum_{j=1}^{k_*}\sum_{i\in\mathcal{A}_j}\exp\Big(\frac{\beta_{0i}}{\tau}\Big)\Big[\|\Delta a_{ij}\|+|\Delta b_{ij}|+|\Delta\nu_{ij}|\Big].
\end{align}
\begin{theorem}
    \label{theorem:tight_rate_exact}
    Under the exact-specified settings, the following Total Variation lower bound holds true for any $G\in\mathcal{E}_{k_*}(\Theta)$:
    \begin{align*}
        \bbE_X[V(g_{G}(\cdot|X),g_{G_*}(\cdot|X))]\gtrsim \mathcal{D}_2(G^{|\Psi},G^{|\Psi}_*).
    \end{align*}
    This bound together with Theorem~\ref{theorem:density_rate} leads to the parametric convergence rate of MLE: $\mathcal{D}_{2}(\widehat{G}^{|\Psi}_n,G^{|\Psi}_*)=\widetilde{\mathcal{O}}(n^{-1/2})$.
\end{theorem}
Proof of Theorem~\ref{theorem:tight_rate_exact} is in Appendix~\ref{appendix:tight_rate_exact}. It follows from the above result that $\aj,\bj,\vj$ share the same estimation rate of order $\widetilde{\mathcal{O}}(n^{-1/2})$, which are significantly faster than those resulting from Theorem~\ref{theorem:log_rate_exact}.

\vspace{-0.5em}
\subsection{Over-specified Settings}
\vspace{-0.5em}
\label{sec:standard_over}
Analogous to the previous section, we first need to design a Voronoi loss function used for the over-specified settings. 

\textbf{Voronoi loss.} Let us define for each $r\geq 1$ that
\begin{align}
    \label{eq:loss_log_over}
    \mathcal{D}_{3,r}(G,G_*):=
    \sum_{j=1}^{k_*}\Big|\sum_{i\in\mathcal{A}_j}\exp\Big(\frac{\beta_{0i}}{\tau}\Big)-\exp\Big(\frac{\beta^*_{0j}}{\tau^*}\Big)\Big|\nonumber\\
    +\sum_{j=1}^{k_*}\sum_{i\in\mathcal{A}_j}\exp\Big(\frac{\beta_{0i}}{\tau}\Big)K_{ij}(r,r,r,r,r).
\end{align}
Recall that under this setting, the true number of experts $k_*$ is unknown, and we assume that the MLE $\widehat{G}_n$ belongs to the set of mixing measures with at most $k>k_*$ components $\mathcal{G}_{k}(\Theta)$. Thus, from the definition of Voronoi cells, there could be some cells $\mathcal{A}_{j}(\widehat{G}_n)$ having more than one element. On the other hand, each Voronoi cell $\mathcal{A}_{j}(\widehat{G}_n)$ under the exact-specified settings has exactly one element. This is the main difference between the Voronoi losses $\mathcal{D}_{3,r}$ and $\mathcal{D}_{1,r}$.
\begin{theorem}
    \label{theorem:log_rate_over}
    Under the over-specified settings, the following minimax lower bound of estimating $G_*$ holds true for any $r\geq 1$:
    \begin{align*}
    \inf_{\overline{G}_n\in\mathcal{G}_{k}(\Theta)}\sup_{G\in\mathcal{G}_{k}(\Theta)\setminus\mathcal{O}_{k_*-1}(\Theta)}\bbE_{g_{G}}[\mathcal{D}_{3,r}(\overline{G}_n,G)]\gtrsim n^{-1/2}.
    \end{align*}
    Here, the notation $\bbE_{g_{G}}$ indicates the expectation taken w.r.t the product measure with mixture density $g^n_{G}$.
\end{theorem}
Proof of Theorem~\ref{theorem:log_rate_over} is in Appendix~\ref{appendix:log_rate_over}. The above minimax lower bound indicates that the estimation rates for parameters $\beta^*_{1j},\tau^*,a^*_{j},b^*_{j},\nu^*_{j}$ are all slower than $\widetilde{\mathcal{O}}(n^{-1/2r})$ for any $r\geq 1$. This means that those rates cannot be faster than polynomial rates and could be as slow as $1/\log(n)$. Such slow rates are caused by the interaction between the softmax temperature and other parameters via the PDEs in equation~\eqref{eq:recall_PDEs}. Despite the issue, not all parameters are negatively impacted. Specifically, our constructed loss function, detailed in Theorem~\ref{theorem:tight_rate_over}, demonstrates polynomial estimation rates for $b^*_{j}$ and $\nu^*_{j}$.

\textbf{Voronoi loss.} Similar to Section~\ref{sec:standard_exact}, for each mixing measure $G=\sum_{i=1}^{k}\exp(\beta_{0i}/\tau)\delta_{(\beta_{1i},\tau,a,b,\nu)}$, we consider its projection on the space $\Upsilon:=\mathbb{R}\times\mathbb{R}_+$ of parameters $b^*_{j},\nu^*_{j}$, that is, $G^{|\Upsilon}:=\sum_{i=1}^{k}\exp(\beta_{0i}/\tau)\delta_{(b_{i},\nu_{i})}$. Then, the loss function of interest is defined as
\begin{align}
    \label{eq:loss_tight_over}
    &\mathcal{D}_4(G^{|\Upsilon},G^{|\Upsilon}_*):=\sum_{j=1}^{k_*}\Big|\sum_{i\in\mathcal{A}_j}\exp\Big(\frac{\beta_{0i}}{\tau}\Big)-\exp\Big(\frac{\beta^*_{0j}}{\tau^*}\Big)\Big|\nonumber\\
    &
    +\sum_{j:|\mathcal{A}_j|>1}\sum_{i\in\mathcal{A}_j}\exp\Big(\frac{\beta_{0i}}{\tau}\Big)\Big[|\Delta b_{ij}|^{\bar{r}(|\mathcal{A}_j|)}+|\Delta\nu_{ij}|^{\frac{\bar{r}(|\mathcal{A}_j|)}{2}}\Big]\nonumber\\
    &+\sum_{j:|\mathcal{A}_j|=1}\sum_{i\in\mathcal{A}_j}\exp\Big(\frac{\beta_{0i}}{\tau}\Big)\Big[|\Delta b_{ij}|+|\Delta\nu_{ij}|\Big].
\end{align}
Here, $\bar{r}(|\mathcal{A}_{j}|)$ stands for the smallest positive integer $r$ such that the following system does not have any non-trivial solutions for the unknown variables $\{p_{l},q_{1l},q_{2l}\}_{l=1}^{m}$. :
\begin{align}
\label{eq:system_r_bar}
\sum_{l=1}^{|\mathcal{A}_{j}|}\sum_{\substack{n_1,n_2\in\mathbb{N}:\\n_1+2n_2=s}}\dfrac{p^2_{l}~q^{n_1}_{1l}~q^{n_2}_{2l}}{n_1!~n_2!}=0, \quad s=1,2,\ldots,r,
\end{align}
A solution is called non-trivial if all the values of $p_{l}$ are different from zero, whereas at least one among $q_{1l}$ is non-zero. \cite{Ho-Nguyen-Ann-16} demonstrate that $\bar{r}(2)=4$, $\bar{r}(3)=6$ and $\bar{r}(m)\geq 7$ when $m\geq 4$.
\begin{theorem}
    \label{theorem:tight_rate_over}
    Under the over-specified settings, the following Total Variation lower bound holds true for any $G\in\mathcal{G}_{k}(\Theta)$:
    \begin{align*}
        \bbE_X[V(g_{G}(\cdot|X),g_{G_*}(\cdot|X))]\gtrsim \mathcal{D}_4(G^{|\Upsilon},G^{|\Upsilon}_*).
    \end{align*}
    This bound together with Theorem~\ref{theorem:density_rate} leads to the parametric convergence rate of MLE: $\mathcal{D}_{4}(\widehat{G}^{|\Upsilon}_n,G^{|\Upsilon}_*)=\widetilde{\mathcal{O}}(n^{-1/2})$.
    
\end{theorem}
Proof of Theorem~\ref{theorem:tight_rate_over} is in Appendix~\ref{appendix:tight_rate_over}. The above result reveals that the MLE $\widehat{G}_n$ converges to the true mixing measure $G_*$ under the loss $\mathcal{D}_{4}$ at the parametric rate $\widetilde{\mathcal{O}}(n^{-1/2})$, which implies the followings:

\textbf{(i)} The rates for estimating parameters $\beta^*_{1j},\nu^*_{j}$ which are fitted by one component, i.e. $|\mathcal{A}_{j}(\widehat{G}_n)|=1$, are of the same order $\widetilde{\mathcal{O}}(n^{-1/2})$. Compared to the rates resulted from Theorem~\ref{theorem:tight_rate_exact} under the exact-specified settings, those rates remain unchanged under the over-specified settings.

\textbf{(ii)} For parameters $\beta^*_{1j},\nu^*_{j}$ which are approximated by more than one component, i.e. $|\mathcal{A}_{j}(\widehat{G}_n)|>1$, their estimation rates are of orders $\widetilde{\mathcal{O}}(n^{-1/2\bar{r}(|\mathcal{A}_{j}(\widehat{G}_n)|)})$ and $\widetilde{\mathcal{O}}(n^{-1/\bar{r}(|\mathcal{A}_{j}(\widehat{G}_n)|)})$, respectively. For instance, if those parameters are fitted by two components, that is, $|\mathcal{A}_{j}(\widehat{G}_n)|=2$, then the previous rates become $\widetilde{\mathcal{O}}(n^{-1/8})$ and $\widetilde{\mathcal{O}}(n^{-1/4})$. On the other hand, if $|\mathcal{A}_{j}(\widehat{G}_n)|=3$, then the rates for estimating $\beta^*_{1j},\nu^*_{j}$ are of orders $\widetilde{\mathcal{O}}(n^{-1/12})$ and $\widetilde{\mathcal{O}}(n^{-1/6})$.

\vspace{-0.5em}
\section{Activation Dense-to-sparse Gating Function}
\label{sec:modified_gate}
In this section, we propose a novel class of gating functions named activation dense-to-sparse in order to improve the slow parameter estimation rates when using the dense-to-sparse gate in Section~\ref{sec:standard_gate}. 

To begin with, let us present the formulation of a Gaussian MoE with the activation dense-to-sparse gating function.

\textbf{Problem setup}. Suppose that the data $\{(X_i,Y_i)\}_{i=1}^{n}\subset\mathbb{R}^d\times\mathbb{R}$ are i.i.d sampled from the activation dense-to-sparse Gaussian MoE, whose conditional density function $p_{G_*}(Y|X)$ is defined as:
\begin{align}
\label{eq:modified_density}
    &\sum_{i=1}^{k_*}\softmax\Big(\dfrac{\sigma((\beta^*_{1i})^{\top}X)+\beta^*_{0i}}{\tau^*}\Big)\nonumber\\
    &\hspace{3cm}\times f(Y|(a^*_i)^{\top}X+b^*_i,\nu^*_i).
\end{align}
In the model's gating network, the output of a linear layer undergoes an activation function $\sigma$ before entering the softmax function. The activation function $\sigma:\mathbb{R}\to\mathbb{R}$ must satisfy the conditions in Definition~\ref{def:modified_function_exact} and Definition~\ref{def:modified_function_over} for exact-specified and over-specified settings, respectively. These conditions mitigate the interaction of the softmax temperature with other parameters in equation~\eqref{eq:recall_PDEs}, addressing the slow rates discussed in Section~\ref{sec:standard_gate}. We maintain the assumptions on the parameters outlined in Section~\ref{sec:introduction}, unless explicitly stated otherwise.

\textbf{Maximum likelihood estimation.} According to the change of gating function, let us re-define the MLE corresponding to the model~\eqref{eq:modified_density} as follows:
\begin{align}
    \label{eq:modified_MLE}
    \widetilde{G}_n:=\argmax_{G}\frac{1}{n}\sum_{i=1}^{n}\log(p_{G}(Y_i|X_i)).
\end{align}

Subsequently, we provide in the following theorem a convergence rate of density estimation under the Gaussian MoE model with the activation dense-to-sparse gate.

\begin{theorem}
    \label{theorem:modified_density_rate}
    Under the Total Variation distance, the density estimation $p_{\widehat{G}_n}(Y|X)$ converges to the true density $p_{G_*}(Y|X)$ at the following rate:
    \begin{align*}
        \bbE_X[V(p_{\widetilde{G}_n}(\cdot|X),p_{G_*}(\cdot|X))]=\widetilde{\mathcal{O}}(n^{-1/2}).
    \end{align*}
\end{theorem}
Proof of Theorem~\ref{theorem:modified_density_rate} is in Appendix~\ref{appendix:modified_density_rate}. The above bound confirms that the density estimation rate under the Gaussian MoE with the activation dense-to-sparse gate is of the same order as that with the standard dense-to-sparse gate in Theorem~\ref{theorem:density_rate}, which is $\widetilde{\mathcal{O}}(n^{-1/2})$. Next, we utilize this result to derive the parameter estimation rates for the model~\eqref{eq:modified_density} under the exact-specified and over-specified settings.

\vspace{-0.5em}
\subsection{Exact-specified Setting}
\vspace{-0.5em}
\label{sec:modified_exact}
First of all, we introduce the conditions on the activation function $\sigma$ in the model~\eqref{eq:modified_density} under this setting.
\begin{definition}[First-order Independence]
\label{def:modified_function_exact}
    Let $\sigma:\mathbb{R}\to\mathbb{R}$ be a differentiable function, and $\bar{\sigma}(X,w):=\sigma(w^{\top}X)$. We say that $\sigma$ is first-order independent if the set
    \begin{align}
        \Big\{1, \ \bar{\sigma}(X,w), \ \frac{\partial\bar{\sigma}}{\partial w^{(u)}}(X,w) :1\leq u\leq d\Big\}
    \end{align}
    is linearly independent w.r.t $X$ for any $w\in\mathbb{R}^d$.
\end{definition}
\textbf{Example.} It can be verified that the functions $\sigmoid(\cdot)$ and Gaussian error linear units $\gelu(\cdot)$ \cite{hendrycks2023gaussian} are first-order independent. On the other hand, the function $z\mapsto z^{p}$ for $p\geq 1$ does not satisfy the first-order independence condition in Definition~\ref{def:modified_function_exact}.

Denote $\widetilde{F}(Y|X,\omega):=\exp(\frac{\sigma(\beta_{1}^{\top}X)}{\tau})f(Y|a^{\top}X+b,\nu)$. Then, the first-order independence condition on the activation function $\sigma$ guarantees that the interaction between $\tau$ and $\beta_1$ in equation~\eqref{eq:recall_PDEs} no longer holds true, that is,
\begin{align*}
    \frac{\partial \widetilde{F}}{\partial\tau}\neq\frac{1}{\tau}\cdot\beta_{1}^{\top}\frac{\partial \widetilde{F}}{\partial\beta_1}.
\end{align*}
As a result, the estimation rates for parameters $\beta^*_{1j}$ and $\tau^*$ should be improved in comparison with those in Section~\ref{sec:standard_gate}. To certify this point, we design the following Voronoi loss function, and then provide in Theorem~\ref{theorem:modified_exact} the convergence rate of the MLE under the exact-specified settings.

\textbf{Voronoi loss.} 
The Voronoi loss of interest is given by
\begin{align}
    \label{eq:loss_modified_exact}
    \mathcal{D}_5(G,G_*)&:=
    \sum_{j=1}^{k_*}\Big|\sum_{i\in\mathcal{A}_j}\exp\Big(\frac{\beta_{0i}}{\tau}\Big)-\exp\Big(\frac{\beta^*_{0j}}{\tau^*}\Big)\Big|\nonumber\\
    +&\sum_{j=1}^{k_*}\sum_{i\in\mathcal{A}_j}\exp\Big(\frac{\beta_{0i}}{\tau}\Big)K_{ij}(1,1,1,1,1).
\end{align}
\begin{theorem}
    \label{theorem:modified_exact}
    Under the exact-specified settings, the following Total Variation lower bound holds true for any $G\in\mathcal{E}_{k_*}(\Theta)$:
    \begin{align*}
        \bbE_X[V(p_{G}(\cdot|X),p_{G_*}(\cdot|X))]\gtrsim \mathcal{D}_5(G,G_*).
    \end{align*}
    This bound together with Theorem~\ref{theorem:modified_density_rate} leads to the parametric convergence rate of MLE: $\mathcal{D}_{5}(\widetilde{G}_n,G_*)=\widetilde{\mathcal{O}}(n^{-1/2})$.
\end{theorem}
Proof of Theorem~\ref{theorem:modified_exact} is in Appendix~\ref{appendix:modified_exact}. Theorem~\ref{theorem:modified_exact} implies that all the rates for estimating parameters $\beta^*_{1j},\tau^*,a^*_{j},b^*_{j},\nu^*_{j}$ are parametric on the sample size, standing at order $\widetilde{\mathcal{O}}(n^{-1/2})$. It can be seen that the estimation rates for $\beta^*_{1j}$ and $\tau^*$ when using the activation dense-to-sparse gate become substantially faster than their counterparts when using the standard dense-to-sparse gate, which are slower than $\widetilde{\mathcal{O}}(n^{-1/2r})$ for any $r\geq 1$. This highlights the benefits of our proposed activation dense-to-sparse gate.

\vspace{-0.5em}
\subsection{Over-specified Setting}
\vspace{-0.5em}
\label{sec:modified_over}
In this section, we continue to characterize conditions for the activation function $\sigma$ under the over-specified settings for the sake of enhancing the parameter estimation rates.

\begin{definition}[Second-order Independence]
\label{def:modified_function_over}
    Let $\sigma:\mathbb{R}\to\mathbb{R}$ be a twice differentiable function and $\bar{\sigma}(X,w):=\sigma(\omega^{\top}X)$. We say that $\sigma$ is second-order independent if the set
    \begin{align*}
        &\Bigg\{1, \ \bar{\sigma}(X,w), \ \bar{\sigma}^2(X,w), \ \frac{\partial\bar{\sigma}}{\partial w^{(u)}}(X,w), \nonumber\\
        &\qquad \Big(\bar{\sigma}\cdot\frac{\partial\bar{\sigma}}{\partial w^{(u)}}\Big)(X,w), \  \Big(\frac{\partial\bar{\sigma}}{\partial w^{(u)}}\cdot\frac{\partial\sigma}{\partial w^{(v)}}\Big)(X,w) ,\\
        &\hspace{2.5cm}\frac{\partial^2\bar{\sigma}}{\partial w^{(u)}\partial w^{(v)}}(X,w): 1\leq u,v\leq d\Bigg\}
    \end{align*}
    is linearly independent for almost surely $X$ for any $w\in\mathbb{R}^d$.
\end{definition}
\textbf{Example.} We can validate that the function $\sigmoid(\cdot)$ and Gaussian error linear units $\gelu(\cdot)$ \cite{hendrycks2023gaussian} also meet the second-order independence condition. Additionally, since the second-order independence condition implies the first-order one, the function $z\mapsto z^p$, for $p\geq 1$, is still not second-order independent.

The second-order independence condition on the activation function $\sigma$ ensures that there are no interactions of the softmax temperature with other parameters as in equation~\eqref{eq:recall_PDEs}, i.e.
\begin{align*}
    \frac{\partial \widetilde{F}}{\partial\tau}\neq\frac{1}{\tau}\cdot\beta_{1}^{\top}\frac{\partial \widetilde{F}}{\partial\beta_1};\quad 
    \frac{\partial^2 \widetilde{F}}{\partial\tau~\partial b}\neq\frac{1}{\tau^2}\cdot\beta_{1}^{\top}\frac{\partial \widetilde{F}}{\partial a}.
\end{align*}
Consequently, not only the rates for estimating $\beta^*_{1j}$ and $\tau^*$ should be improved under the over-specified settings as in Section~\ref{sec:modified_exact} but also those for parameters $a^*_{j}$. Now, it is necessary to build a Voronoi loss function to give a theoretical guarantee for that claim in Theorem~\ref{theorem:modified_over}.

\textbf{Voronoi loss.} 
Then, the Voronoi loss of interest is given by
\begin{align}
    \label{eq:loss_modified_over}
    &\mathcal{D}_6(G,G_*):=\sum_{j=1}^{k_*}\Big|\sum_{i\in\mathcal{A}_j}\exp\Big(\frac{\beta_{0i}}{\tau}\Big)-\exp\Big(\frac{\beta^*_{0j}}{\tau^*}\Big)\Big|\nonumber\\
    +&\sum_{j:|\mathcal{A}_j|>1}\sum_{i\in\mathcal{A}_j}\exp\Big(\frac{\beta_{0i}}{\tau}\Big)K_{ij}\Big(2,2,2,\bar{r}(|\mathcal{A}_j|),\frac{\bar{r}(|\mathcal{A}_j|)}{2}\Big)\nonumber\\
    +&\sum_{j:|\mathcal{A}_j|=1}\sum_{i\in\mathcal{A}_j}\exp\Big(\frac{\beta_{0i}}{\tau}\Big)K_{ij}(1,1,1,1,1).
\end{align}
\begin{theorem}
    \label{theorem:modified_over}
    Under the over-specified settings, the following Total Variation lower bound holds true for any $G\in\mathcal{G}_{k}(\Theta)$:
    \begin{align*}
        \bbE_X[V(p_{G}(\cdot|X),p_{G_*}(\cdot|X))]\gtrsim \mathcal{D}_6(G,G_*)
    \end{align*}
    This bound together with Theorem~\ref{theorem:modified_density_rate} leads to the parametric convergence rate of MLE: $\mathcal{D}_{6}(\widetilde{G}_n,G_*)=\widetilde{\mathcal{O}}(n^{-1/2})$.
\end{theorem}
Proof of Theorem~\ref{theorem:modified_over} is in Appendix~\ref{appendix:modified_over}. The above parametric convergence rate $\widetilde{\mathcal{O}}(n^{-1/2})$ of the MLE $\widetilde{G}$ to $G_*$ under the loss function $\mathcal{D}_{6}$ gives us the followings:

\textbf{(i)} Under the over-specified settings, parameters $\beta^*_{1j},\tau^*,a^*_{j},b^*_{j},\nu^*_{j}$ which are fitted by one component, i.e. $|\mathcal{A}_{j}(\widetilde{G}_n)|=1$, enjoy the same estimation rate of order $\widetilde{\mathcal{O}}(n^{-1/2})$. This result aligns with that under the exact-specified settings in Theorem~\ref{theorem:modified_exact}.

\textbf{(ii)} On the other hand, those for parameters approximated by more than one component, i.e. $|\mathcal{A}_{j}(\widetilde{G}_n)|>1$, are no longer homogeneous. In particular, the rates for estimating $\beta^*_{1j},\tau^*,a^*_{j}$ are of order $\widetilde{\mathcal{O}}(n^{-1/4})$. At the same time, the estimation rates for $b^*_{j}$ and $\nu^*_{j}$ become $\widetilde{\mathcal{O}}(n^{-1/2\bar{r}(|\mathcal{A}_{j}(\widetilde{G}_n)|)})$ and $\widetilde{\mathcal{O}}(n^{-1/\bar{r}(|\mathcal{A}_{j}(\widetilde{G}_n)|)})$, respectively.

\vspace{-0.5em}
\section{Numerical Experiments}  
\label{sec:implications}
In this section, we perform numerical experiments to empirically confirm the theoretical convergence rates of maximum likelihood estimation (MLE) in both standard and activation dense-to-sparse gating MoE models. We employ an Expectation-Maximization (EM) algorithm \cite{dempster_maximum_1977} for parameter estimation, utilizing synthetic datasets generated from the true models in equation~\eqref{eq:standard_density} and equation \eqref{eq:modified_density}. Further details on the experimental setups are deferred to Appendix~\ref{appendix:experiments} due to the space limit.
\begin{figure*}[!ht]
    \centering
    \begin{subfigure}{.32\textwidth}
        \centering
        \includegraphics[scale = .34]{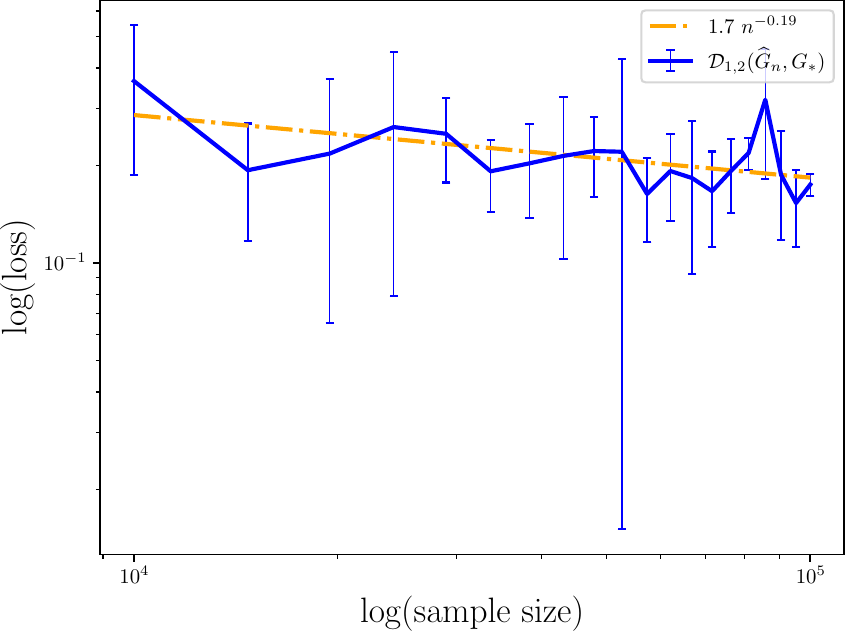}
        \caption{(Exact-specified) Convergence rate of $\mathcal{D}_{1,2}(\widehat{G}_n, G_{*})$}
        \label{fig:exact-D1r}
    \end{subfigure}
    \begin{subfigure}{.32\textwidth}
        \centering
        \includegraphics[scale = .34]{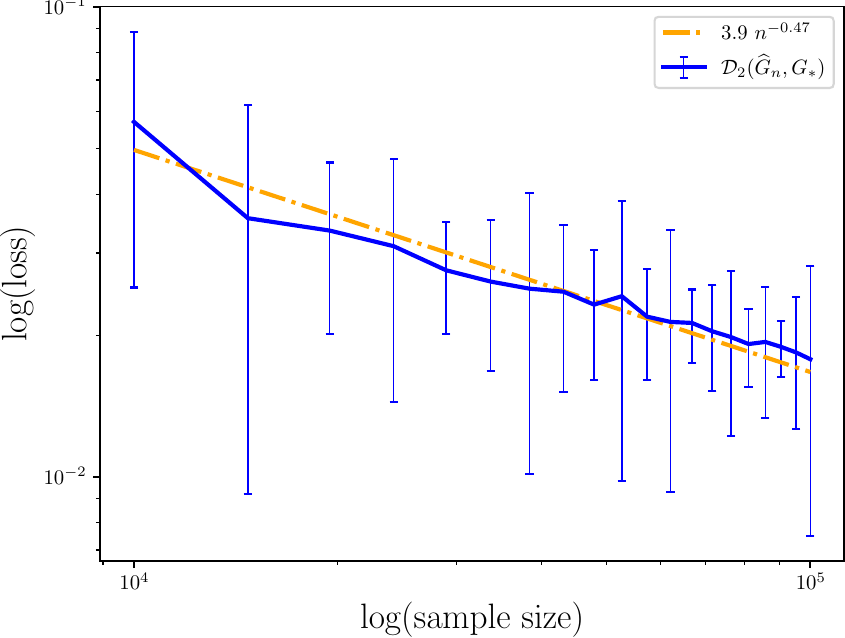}
        \caption{(Exact-specified) Convergence rate of $\mathcal{D}_{2}(\widehat{G}_n, G_{*})$}
        \label{fig:exact-D2}
    \end{subfigure}
    \begin{subfigure}{.32\textwidth}
        \centering
        \includegraphics[scale = .34]{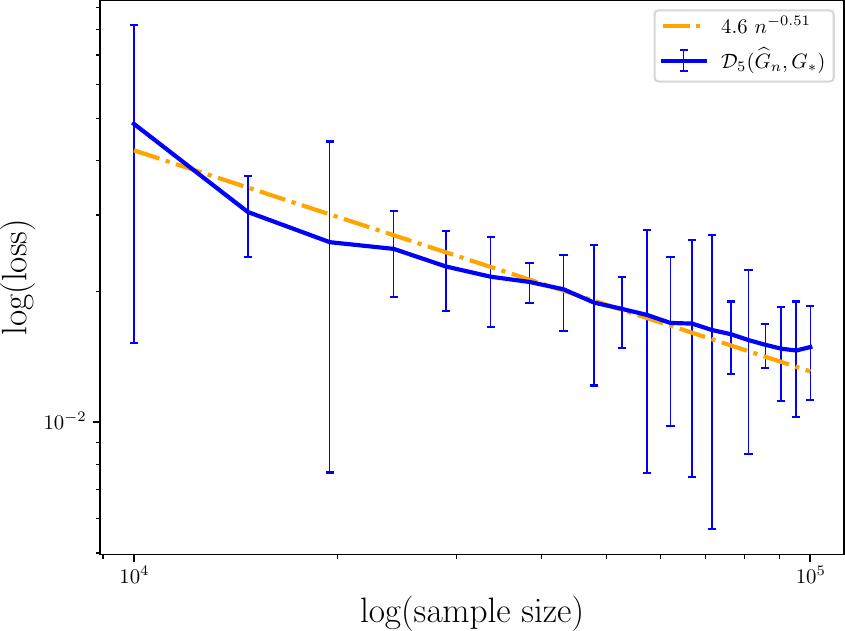}
        \caption{(Exact-specified with sigmoid activation) Convergence rate of $\mathcal{D}_{5}(\widehat{G}_n, G_{*})$}	
        \label{fig:exact-D5}
    \end{subfigure}    
    \vskip\baselineskip
    \begin{subfigure}{.32\textwidth}
        \centering
        \includegraphics[scale = .34]{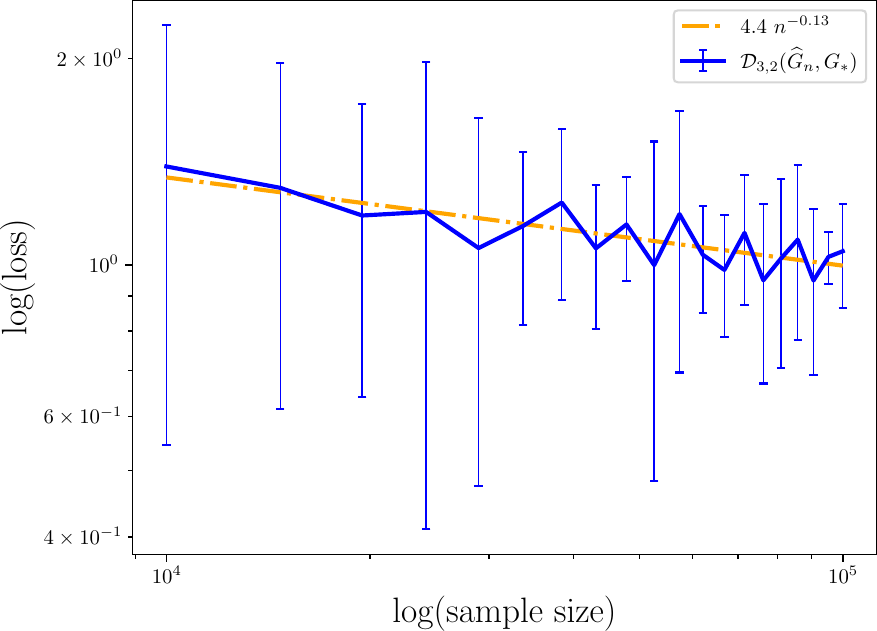}
        \caption{(Over-specified) Convergence rate of $\mathcal{D}_{3,2}(\widehat{G}_n, G_{*})$}	
        \label{fig:over-D3r}
    \end{subfigure}    
    \begin{subfigure}{.32\textwidth}
        \centering
        \includegraphics[scale = .34]{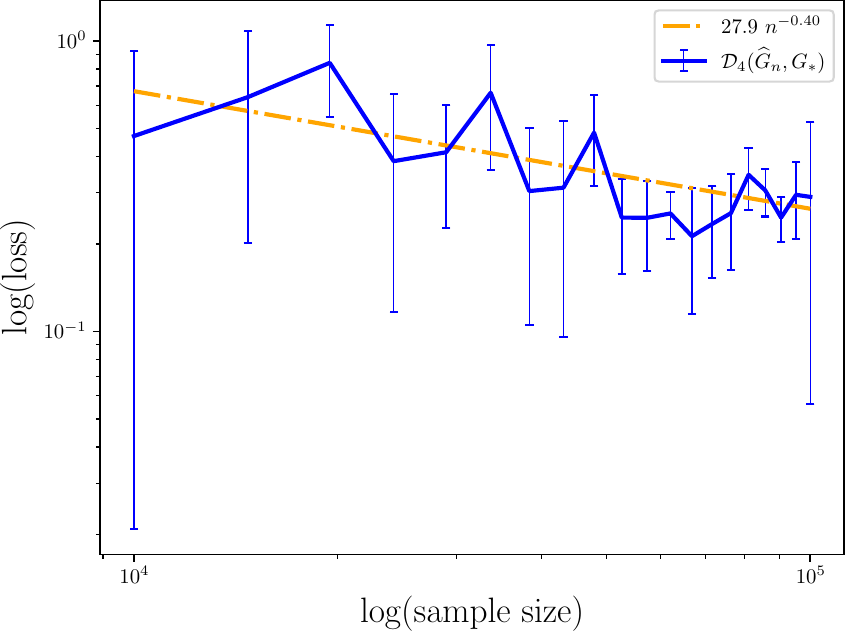}
        \caption{(Over-specified) Convergence rate of $\mathcal{D}_{4}(\widehat{G}_n, G_{*})$}	
        \label{fig:over-D4}
    \end{subfigure}
    \begin{subfigure}{.32\textwidth}
        \centering
        \includegraphics[scale = .34]{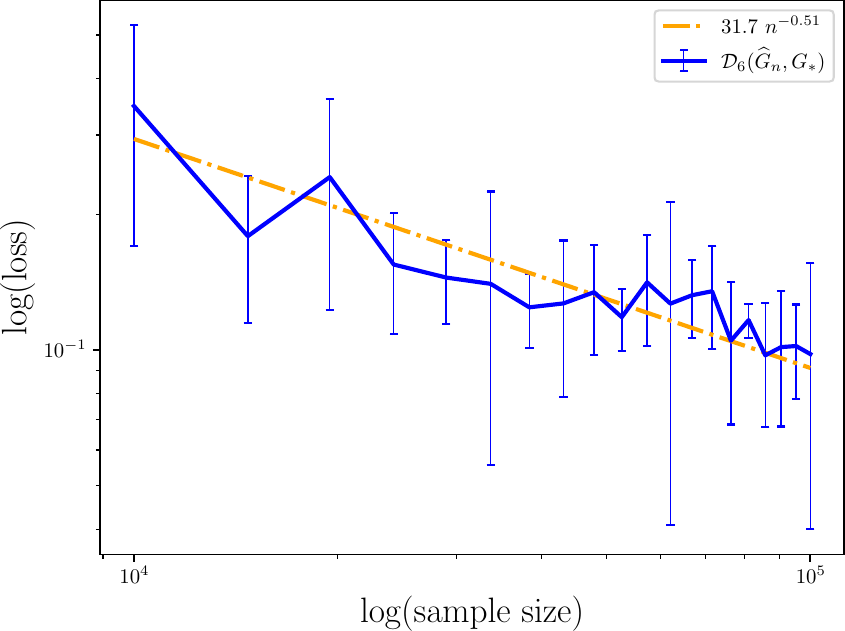}
        \caption{(Over-specified with sigmoid activation) Convergence rate of $\mathcal{D}_{6}(\widehat{G}_n, G_{*})$}	
        \label{fig:over-D6}
    \end{subfigure}
    
    \caption{Log-log scaled plots for the empirical convergence rates of the MLE $\widehat{G}_n$ for exact and over-specified settings. In these figures, the corresponding empirical discrepancies are illustrated by the blue curves, while the orange dash-dotted lines represent for the least-squares fitted linear regression lines. The error bars represent two times the empirical standard deviation under the exact-specified setting.\vspace{-0.75em}}
    \label{fig:experiments}
\end{figure*}

Throughout the experiments under the over-specified setting, we set true number of experts $k_*=2$ and the estimated number of experts $k=k_{*}+1=3$. We perform $40$ experiments for each sample size $n$ in different setups, encompassing a range of $20$ distinct sample sizes whose values vary from $n=10^4$ to $n=10^5$. Moreover, we use the sigmoid function as an activation for the activation dense-to-sparse gate. The graphs in Figure~\ref{fig:experiments} illustrate that the empirical convergence rates of the MLE $\widehat{G}_n$ to the true mixing measure $G_{*}$ under different Voronoi metrics.

\vspace{-0.5em}
\subsection{Standard Dense-to-sparse Gating Function}

\textbf{Exact-specified setting.} As illustrated in Figures~\ref{fig:exact-D1r} and~\ref{fig:exact-D2}, the convergence rate of $\mathcal{D}_{1, 2}(\widehat{G}_n, G_{*})$ is significantly slow as indicated by Theorem~\ref{theorem:log_rate_exact}. Additionally, that of $\mathcal{D}_{2}(\widehat{G}_n, G_{*})$ admits a much faster rate of order $\widetilde{\mathcal{O}}(n^{-1/2})$, which aligns with the result in Theorem~\ref{theorem:tight_rate_exact}.

\textbf{Over-specified setting.}
Similar to the exact-specified setting, as depicted in Figures~\ref{fig:over-D3r} and~\ref{fig:over-D4}, the convergence rate for $\mathcal{D}_{3, 2}(\widehat{G}_n, G_{*})$ is notably slower than that of $\mathcal{D}_{4}(\widehat{G}_n, G_{*})$, where $\mathcal{D}_{4}(\widehat{G}_n, G_{*})$ exhibits a parametric rate of order $\widetilde{\mathcal{O}}(n^{-1/2})$ per Theorem~\ref{theorem:log_rate_over} and Theorem~\ref{theorem:tight_rate_over}.

\vspace{-0.5em}
\subsection{Activation Dense-to-sparse Gating Function}
\vspace{-0.5em}
Illustrated in Figures~\ref{fig:exact-D5}~and~\ref{fig:over-D6}, the use of sigmoid activation in the softmax gate leads to an enhanced rate of $\widetilde{\mathcal{O}}(n^{-1/2})$ for both exact-specified and over-specified settings. These results totally match those in Theorem~\ref{theorem:modified_exact} and Theorem~\ref{theorem:modified_over}.

\section{Practical Implications}
\label{sec:practical_implications}
In this section, we provide four main practical implications from our theoretical results:

\textbf{1. Expert Selection:} An important application of the dense-to-sparse MoE is to select the important experts (similar to the popular top-K sparse MoE for large language models in the literature). Our theory for parameter estimation have direct indications for the sample efficiency of choosing the important experts. In particular, it suggests that we need an exponential number of data (roughly $\exp(1/\epsilon^{\eta})$ for some $\eta > 0$ where $\epsilon$ is the desired approximation error) to estimate the softmax gating function with temperature, which directly implies that we need an exponential number of data to select the important experts. This is undesirable in practice and would lead to potential wrong choice of important experts when we only have a finite number of data. On the other hand, for the proposed activation dense-to-sparse gate, we can reduce that exponential number of data to only a polynomial number of data (roughly $\epsilon^{-\bar{\eta}}$ for some $\bar{\eta} > 0$) to select the important experts, which is a considerable improvement in practice.

\textbf{2. Expert Estimation:} Apart from the benefit of sample complexity from dense-to-sparse to activation dense-to-sparse models, the parameter estimation rates also directly imply the convergence behavior of estimating expert networks, which is of practical interest. In particular, we consider linear experts of the form $a^{\top}x+b$ in our work, and show that the estimation rate of the expert $a^{\top}x+b$ is the slowest between the estimation rates of its parameters $a$ and $b$. For instance, it can be seen from Table~\ref{table:parameter_rates} that under the over-specified setting of the dense-to-sparse gate MoE model, the rate for estimating $b^*_j$ is faster than that for $a^*_j$. Consequently, the expert estimation rate is equal to the rate for estimating $a^*_j$, which is slower than any polynomial rates. By contrast, under the over-specified setting of the activation dense-to-sparse gate MoE model, since the estimation rate of $a^*_j$ is improved to be faster than that of $b^*_j$, the expert admits the same estimation rate as $b^*_j$, which is of order $\widetilde{\mathcal{O}}(n^{-1/2\bar{r}})$, faster the previous one. From this observation, we see that the parameter estimation problem also provides useful insights in designing the gating mechanism of MoE models, which helps enhance the performance of expert networks significantly.

\textbf{3. Misspecified settings:} In this paper, we consider a well-specified setting where the data are assumed to be sampled from the (activation) dense-to-sparse gating Gaussian MoE in order to lay the foundation for a more realistic yet challenging misspecified setting where the data are not necessarily generated from these models. Under that misspecified setting, we assume that the data are generated from the true but unknown conditional distribution $Q(Y|X)$, which is not either dense-to-sparse or activation dense-to-sparse mixture of experts. Then, we can demonstrate that the MLE $\widehat{G}_n$ converges to a mixing measure $\overline{G} \in \argmin_{G \in \mathcal{G}_{k}(\Theta)} \text{KL}(Q(Y|X) || g_{G}(Y|X))$ where KL stands for the Kullback-Leibler divergence, and $g_{G}(Y|X)$ is the conditional density of the (activation) dense-to-sparse gate Gaussian MoE. As $Q(Y|X)$ does not belong to the (activation) dense-to-sparse MoEs, the optimal mixing measure will be in the boundary of the parameter space, namely, the number of experts in $\bar{G}$ is $k$. Therefore, as $n$ is sufficiently large, $\widehat{G}_{n}$ also has $k$ experts. 

The insights from our theories in the well-specified setting indicate that the Voronoi losses can be used to obtain the estimation rates of individual parameters of the MLE $\widehat{G}_n$ to those of $\overline{G}$. From Table 1 in the manuscript, we can see that

\emph{(3.1) Dense-to-sparse MoE:} the worst parameter estimation rate of $\widehat{G}_{n}$ to $\overline{G}$ could be as slow as $1/\log^{\eta}(n)$ for some $\eta > 0$. It indicates that we still need exponential number of data (roughly $\exp(1/\epsilon^{\eta})$ where $\epsilon$ is the desired approximation error) to estimate $\overline{G}$. 

\emph{(3.2) Activation dense-to-sparse MoE:} Under the misspecified setting, the parameter estimation rate of $\widehat{G}_{n}$ to $\overline{G}$ is $n^{-1/2}$. It indicates that we only need roughly $\epsilon^{-2}$ where $\epsilon$ is the desired approximation error to estimate $\overline{G}$. 

As a consequence, under the misspecified settings, the parameter estimation rates achieved when we use the activation dense-to-sparse gate MoE in the above KL divergence should be faster than those obtained when we use the dense-to-sparse gate MoE. This explains why the activation dense-to-sparse gate is a solution to the parameter estimation problem, or more generally, the expert estimation and selection problem of the MoE models.

\textbf{4. Model design:} From the benefits of the activation dense-to-sparse gate for the expert estimation of MoE models when using the temperature to smooth the expert selection process, we deduce that it would be better to use a gating network with sufficiently sophisticated activation functions (e.g. $\mathrm{sigmoid}$, $\mathrm{GeLU}$, etc) rather than a simple linear gating network.


\vspace{-0.5em}
\section{Concluding Remarks}
\label{sec:conclusion}
In this paper, we investigate the effects of the dense-to-sparse gate on the convergence rates of maximum likelihood estimation under the Gaussian mixture of experts. We discover that the density estimation rate is parametric on the sample size. On the other hand, due to the interactions of the temperature with both gating and expert parameters via two partial differential equations, the rates for estimating them are slower than any polynomial rates, and therefore, could be as slow as $\mathcal{O}(1/\log(n))$. To enhance the sample efficiency of the temperature for the Gaussian mixture of experts, we design a novel gating function called activation dense-to-sparse, which routes the outputs of a linear layer to an activation function before sending them to the softmax function. We demonstrate that if the activation function meets the first-order (second-order) independence condition, then the aforementioned interactions disappear, and the parameter estimation rates become polynomial under the exact-specified (over-specified) settings.

Following from the results in this work, there are a few promising directions that we leave for future development:

Firstly, the convergence analysis of maximum likelihood estimation in this paper is carried out under the assumption that the data are sampled from the (activation) dense-to-sparse gating Gaussian mixture of experts. However, when the data are not necessarily generated from that model, such theoretical analysis has remained missing in the literature. More concretely, under that setting, the MLE converges to a mixing measure $\overline{G} \in \argmin_{G \in \mathcal{G}_{k}(\Theta)} \text{KL}(Q(Y|X) || g_{G}(Y|X))$ where $Q(Y|X)$ stands for the true conditional distribution of $Y$ given $X$ and $\text{KL}$ denotes the Kullback-Leibler divergence. It is worth noting that current techniques for analyzing the convergence of the MLE apply only for the setting when the space of mixing measures is convex. Since the space $\mathcal{G}_{k}(\Theta)$ is non-convex, it is essential to develop further technical tools to establish the convergence rate of the MLE to the set of $\overline{G}$.

Secondly, the current theories are for probabilistic settings of dense-to-sparse gating mixture of experts, namely, when the expert functions are means of the Gaussian distribution. In practical applications, we usually consider the deterministic settings of dense-to-sparse gating mixture of experts of the form: $\sum_{i = 1}^{k^{*}} \softmax \left(\frac{(\beta_{1i}^{*})^{\top} X + \beta_{0i}^{*}}{\tau}\right) h(X, \eta_{i}^{*})$ where $h(., \eta_{i}^{*})$ are general expert functions for all $i \in [k]$ and utilize the least square loss function to estimate the true parameters. Therefore, extending the current theories under the probabilistic settings to those under the deterministic settings is also practically important. 

\section*{Acknowledgements}
NH acknowledges support from the NSF IFML 2019844 and the NSF AI Institute for Foundations of Machine Learning.

\section*{Impact  Statement}

This paper presents work whose goal is to advance the field of Machine Learning.
Given the theoretical nature of the paper, we believe that there are no potential societal consequences of our work.

\bibliography{example_paper}
\bibliographystyle{icml2024}

\newpage
\appendix
\onecolumn

\centering
\textbf{\Large{Supplementary Material for \\ \vspace{.2em} 
``Is Temperature Sample Efficient for Softmax Gaussian Mixture of Experts?''}}

\justifying
\setlength{\parindent}{0pt}
In this supplementary material, we provide rigorous theoretical guarantee for the convergence rate of density estimation in Appendix~\ref{appendix:density_rates}, while we leave that for parameter estimation in Appendix~\ref{appendix:parameter_estimations}. We study the identifiability of the (activation) dense-to-sparse gating Gaussian mixture of experts (MoE) in Appendix~\ref{sec:identifiability}. Finally, additional experiment setups are presented in Appendix~\ref{appendix:experiments}.

\section{Proofs for Density Estimation Rates}
\label{appendix:density_rates}
\subsection{Proof of Theorem~\ref{theorem:density_rate}}
\label{appendix:density_rate}
In this proof, we will leverage results regarding the convergence rates of density estimation from MLE in [Theorem 7.4, \cite{vandeGeer-00}]. Prior to presenting those result here, it is necessary to introduce some notations. Firstly, we denote by $\mathcal{G}_k(\Theta)$ the set of conditional densities of all mixing measures in $\mathcal{O}_k(\Theta)$, that is, $\mathcal{G}_k(\Theta):=\{g_{G}(Y|X):G\in\mathcal{G}_{k}(\Theta)\}$. Next, we define
\begin{align*}
    \widetilde{\mathcal{G}}_k(\Theta)&:=\{g_{(G+G_*)/2}(Y|X):G\in\mathcal{O}_k(\Theta)\},\\
    \widetilde{\mathcal{G}}^{1/2}_k(\Theta)&:=\{g^{1/2}_{(G+G_*)/2}(Y|X):G\in\mathcal{O}_k(\Theta)\}.
\end{align*}
Additionally, for each $\delta>0$, the Hellinger ball centered around the conditional density $g_{G_*}(Y|X)$ and intersected with the set $ \widetilde{\mathcal{G}}^{1/2}_k(\Theta)$ is defined as
\begin{align*}
     \widetilde{\mathcal{G}}^{1/2}_k(\Theta,\delta):=\left\{g^{1/2}\in \widetilde{\mathcal{G}}^{1/2}_k(\Theta):h(g,g_{G_*})\leq\delta\right\}.
\end{align*}
In order to measure the size of the above set, Geer et. al. \cite{vandeGeer-00} suggest using the following quantity:
\begin{align}
    \label{eq:bracket_size}
    \mathcal{J}_B(\delta, \widetilde{\mathcal{G}}^{1/2}_k(\Theta,\delta)):=\int_{\delta^2/2^{13}}^{\delta}H_B^{1/2}(t, \widetilde{\mathcal{G}}^{1/2}_k(\Theta,t),\|\cdot\|)~\dint t\vee \delta,
\end{align}
where $H_B(t, \widetilde{\mathcal{G}}^{1/2}_k(\Theta,t),\|\cdot\|)$ stands for the bracketing entropy \cite{vandeGeer-00} of $ \widetilde{\mathcal{G}}^{1/2}_k(\Theta,u)$ under the $\ell_2$-norm, and $t\vee\delta:=\max\{t,\delta\}$. Now, let us recall the statement of Theorem 7.4 in \cite{vandeGeer-00} with notations being adapted to this work.
\begin{lemma}[Theorem 7.4, \cite{vandeGeer-00}]
    \label{theorem:vandegeer}
    Take $\Psi(\delta)\geq \mathcal{J}_B(\delta,\widetilde{\mathcal{G}}^{1/2}_k(\Theta,\delta))$ that satisfies $\Psi(\delta)/\delta^2$ is a non-increasing function of $\delta$. Then, for some universal constant $c$ and for some sequence $(\delta_n)$ such that $\sqrt{n}\delta^2_n\geq c\Psi(\delta_n)$, we achieve that
    \begin{align*}
        \mathbb{P}\Big(\bbE_X[h(g_{\widehat{G}_n}(\cdot|X),g_{G_*}(\cdot|X))]>\delta\Big)\leq c\exp\left(-\frac{n\delta^2}{c^2}\right),
    \end{align*}
    for all $\delta\geq \delta_n$, where $h(g_1,g_2):=\Big(\frac{1}{2}\int(\sqrt{g_1}-\sqrt{g_2})^2\dint \mu\Big)^{1/2}$ is the Hellinger distance w.r.t Lebesgue measure $\mu$.
\end{lemma}
Proof of Lemma~\ref{theorem:vandegeer} can be found in \cite{vandeGeer-00}. Subsequently, we provide below a result on the bound for the bracketing entropy, which is essential for the proof of Theorem~\ref{theorem:density_rate}.

\begin{lemma}
    \label{lemma:bracket_bound}
    Assume that $\Theta$ is a bounded set, then the following inequality holds true for any $0\leq\varepsilon\leq 1/2$:
    \begin{align*}
        H_B(\varepsilon,\mathcal{G}_k(\Theta),h)\lesssim \log(1/\varepsilon).
    \end{align*}
\end{lemma}
Proof of Lemma~\ref{lemma:bracket_bound} is deferred to Appendix~\ref{appendix:bracket_bound}. Equipped with the results in Lemma~\ref{theorem:vandegeer} and Lemma~\ref{lemma:bracket_bound}, we present the proof of Theorem~\ref{theorem:density_rate} in Appendix~\ref{appendix:density_main_proof}.

\subsubsection{Main Proof}
\label{appendix:density_main_proof}
It is worth noting that
\begin{align*}
    H_B(t, \widetilde{\mathcal{G}}^{1/2}_k(\Theta,t),\|\cdot\|)\leq H_B(t,\mathcal{G}_k(\Theta,t),h),
\end{align*}
for any $t>0$. Then, we deduce from equation~\eqref{eq:bracket_size} that 
\begin{align*}
    \mathcal{J}_B(\delta, \widetilde{\mathcal{G}}^{1/2}_k(\Theta,\delta))&\leq \int_{\delta^2/2^{13}}^{\delta}H_B^{1/2}(t, \mathcal{G}_k(\Theta,t),h)~\dint t\vee \delta\lesssim \int_{\delta^2/2^{13}}^{\delta}\log(1/t)dt\vee\delta,
\end{align*}
where the second inequality occurs due to the upper bound of a bracketing entropy in Lemma~\ref{lemma:bracket_bound}. 

Denote $\Psi(\delta)=\delta\cdot[\log(1/\delta)]^{1/2}$, it is clear that $\Psi(\delta)/\delta^2$ is a non-increasing function of $\theta$. Furthermore, it follows the above inequality that $\Psi(\delta)\geq \mathcal{J}_B(\delta,\widetilde{\mathcal{G}}^{1/2}_k(\Theta,\delta))$. Additionally, let $\delta_n=\sqrt{\log(n)/n}$, we get that $\sqrt{n}\delta^2_n\geq c\Psi(\delta_n)$ for some universal constant $c$. Now, by applying Lemma~\ref{theorem:vandegeer}, we obtain that 
\begin{align*}
    \mathbb{P}\Big(\bbE_X[h(g_{\widehat{G}_n}(\cdot|X),g_{G_*}(\cdot|X))]>C(\log(n)/n)^{1/2}\Big)\lesssim \exp(-c\log(n)),
\end{align*}
for some universal constant $C$ that depends only on $\Theta$. Since the Hellinger distance is lower bounded by the Total Variation distance, i.e. $h\geq V$, we also achieve that
\begin{align*}
    \mathbb{P}\Big(\bbE_X[V(g_{\widehat{G}_n}(\cdot|X),g_{G_*}(\cdot|X))]>C(\log(n)/n)^{1/2}\Big)\lesssim \exp(-c\log(n)).
\end{align*}
Hence, we reach the conclusion that $\bbE_X[V(g_{\widehat{G}_n}(\cdot|X),g_{G_*}(\cdot|X))]=\widetilde{\mathcal{O}}(n^{-1/2})$.

\subsubsection{Proof of Lemma~\ref{lemma:bracket_bound}}
\label{appendix:bracket_bound}
First of all, we aim to derive an upper bound for the Gaussian density $f(Y|a^{\top}X+b,\nu)$. As both $\mathcal{X}$ and $\Theta$ are bounded sets, we can find positive constants $\kappa,u,\ell$ that satisfy $-\kappa\leq a^{\top}X+b\leq\kappa$ and $\ell\leq \nu\leq u$. Therefore, we have that 
    \begin{align*}
        f(Y|a^{\top}X+b,\nu)=\frac{1}{\sqrt{2\pi \nu}}\exp\Big(-\frac{(Y-a^{\top}X-b)^2}{2\nu}\Big)\leq \frac{1}{\sqrt{2\pi \ell}}.
    \end{align*}
    For any $|Y|\geq 2\kappa$, we get that $\frac{(Y-a^{\top}X-b)^2}{2\nu}\geq \frac{Y^2}{8u}$, implying that 
    \begin{align*}
        f(Y|a^{\top}X+b,\nu)\leq \frac{1}{\sqrt{2\pi\ell}}\exp\Big(-\frac{Y^2}{8u}\Big).
    \end{align*}
    Putting the above results together, it follows that $f(Y|a^{\top}X+b,\nu)\leq B(Y|X)$, where we define
    \begin{align*}
        B(Y|X)=\begin{cases}
            \frac{1}{\sqrt{2\pi\ell}}\exp\Big(-\frac{Y^2}{8u}\Big), \hspace{1cm}|Y|\geq 2\kappa;\\
            \textbf{}\\
            \frac{1}{\sqrt{2\pi \ell}},\hspace{3cm} \text{otherwise}.
        \end{cases}
    \end{align*} 
    Let $\eta\leq\varepsilon$, we assume that the set $\mathcal{G}_k(\Theta)$ has an $\eta$-cover (under $\ell_1$-norm) denoted by $\{\pi_1,\ldots,\pi_N\}$, where $N:={N}(\eta,\mathcal{G}_k(\Theta),\|\cdot\|_1)$ is the $\eta$-covering number of the metric space $(\mathcal{G}_k(\Theta),\|\cdot\|_1)$. Then, we construct the brackets of the form $[L_i(Y|X),U_i(Y|X)]$ for all $i\in[N]$ as follows:
    \begin{align*}
        L_i(Y|X)&:=\max\{\pi_i(Y|X)-\eta,0\},\\
        U_i(Y|X)&:=\min\{\pi_i(Y|X)+\eta, B(Y|X)\}.
    \end{align*}
    We can verify that $\mathcal{G}_k(\Theta)\subset \bigcup_{i=1}^{N}[L_i(Y|X),U_i(Y|X)]$ with a note that $0\leq U_i(Y|X)-L_i(Y|X)\leq \min\{2\eta, B(Y|X)\}$. Next, for each $i\in[N]$, the term $\|U_i-L_i\|_1$ is upper bounded as follows:
    \begin{align*}
        \|U_i-L_i\|_1&= \int_{|Y|<2\kappa}(U_i(Y|X)-L_i(Y|X))~\dint(X,Y)+\int_{|Y|\geq 2\kappa}(U_i(Y|X)-L_i(Y|X))~\dint(X,Y)\\
        &\leq R\eta+\exp\Big(-\frac{R^2}{2u}\Big)\leq R'\eta,
    \end{align*}
    in which $R:=\max\{2\kappa,\sqrt{8u}\}\log(1/\eta)$ and $R'>0$ is a universal constant. From the definition of bracketing entropy, $H_B(R'\eta,\mathcal{G}_k(\Theta),\|\cdot\|_1)$ is the logarithm of the smallest number of brackets of size $R'\eta$ necessary to cover $\mathcal{G}_k(\Theta)$, which leads to
    \begin{align}
        \label{eq:bracketing_covering_bound}
        H_B(R'\eta,\mathcal{G}_k(\Theta),\|\cdot\|_1)&\leq \log N=\log{N}(\eta,\mathcal{G}_k(\Theta),\|\cdot\|_1).
    \end{align}
    As we demonstrate at the end of this proof, the covering number is bounded as $\log{N}(\eta,\mathcal{G}_k(\Theta),\|\cdot\|_{1})\lesssim \log(1/\eta)$. This bound together with the result in equation~\eqref{eq:bracketing_covering_bound} implies that
    \begin{align*}
        H_B(R'\eta,\mathcal{G}_k(\Theta),\|\cdot\|_1)\lesssim \log(1/\eta).
    \end{align*}
    By choosing $\eta=\varepsilon/R'$, we obtain that $H_B(\varepsilon,\mathcal{G}_k(\Theta),\|\cdot\|_1)\lesssim\log(1/\varepsilon)$. Moreover, since the Hellinger distance is upper bounded by the $\ell_1$-norm, we reach the desired conclusion that
    \begin{align*}
        H_B(\varepsilon,\mathcal{G}_k(\Theta),h)\lesssim\log(1/\varepsilon).
    \end{align*}
    \textbf{Upper bound of the covering number.}
    For completion, we establish the following upper bound for the covering number:
    \begin{align*}
        \log{N}(\eta,\mathcal{G}_k(\Theta),\|\cdot\|_{1})\lesssim \log(1/\eta).
    \end{align*}
    Since $\Theta$ is a compact set, it follows that $\Delta:=\{(\beta_0,\beta_1,\tau)\in\mathbb{R}\times\mathbb{R}^d\times\mathbb{R}_+:(\beta_0,\beta_1,\tau,a,b,\nu)\in\Theta\}$ and $\Omega:=\{(a,b,\nu)\in\mathbb{R}^d\times\mathbb{R}\times\mathbb{R}_+:(\beta_{0},\beta_{1},\tau,a,b,\nu)\in\Theta\}$ are also compact. Therefore, there exist $\eta$-covers $\Delta_{\eta}$ and ${\Omega}_{\eta}$ for $\Delta$ and $\Omega$, respectively. Additionally, we can validate that $|\Delta_{\eta}|\leq \mathcal{O}(\eta^{-(d+2)k})$ and $|\Omega_{\eta}|\leq \mathcal{O}(\eta^{-(d+2)k})$.
    
    Subsequently, for each mixing measure $G=\sum_{i=1}^k\exp\Big(\frac{\beta_{0i}}{\tau}\Big)\delta_{(\beta_{1i},\tau,a_i,b_i,\nu_i)}\in\mathcal{O}_k(\Theta)$, we consider another one denoted by $\widetilde{G}:=\sum_{i=1}^k\exp\Big(\frac{\beta_{0i}}{\tau}\Big)\delta_{({\beta}_{1i},\tau,\overline{a}_i,\overline{b}_i,\overline{\nu}_i)}$, where $(\overline{a}_i,\overline{b}_i,\overline{\nu}_i)\in{\Omega}_{\eta}$ such that $(\overline{a}_i,\overline{b}_i,\overline{\nu}_i)$ are the closest to $(a_i,b_i,\nu_i)$ in that set for all $i\in[k]$. Besides, we also take into account the mixing measure $\overline{G}:=\sum_{i=1}^k\exp\Big(\frac{\overline{\beta}_{0i}}{\overline{\tau}}\Big)\delta_{({\overline{\beta}}_{1i},\overline{\tau},\overline{a}_i,\overline{b}_i,\overline{\nu}_i)}$, where $(\overline{\beta}_{0i},\overline{\beta}_{1i},\overline{\tau})\in\Delta_{\eta}$ are the closest to $(\beta_{0i},\beta_{1i},\tau)$ in that set. It can be verified that the conditional density $g_{\overline{G}}$ belongs to the following set:
   \begin{align*}
        \mathcal{R}:=\left\{g_{G}\in\mathcal{G}_k(\Theta):(\beta_{0i},\beta_{1i},\tau)\in\Delta_{\eta},~(a_i,b_i,\nu_i)\in\Omega_{\eta}, \ \forall i\in[k]\right\}.
    \end{align*}
    It follows from the formulation of $\widetilde{G}$ that
    \begin{align}
        \label{eq:first_term}
        \|g_{G}-g_{\widetilde{G}}\|_{1}&\leq \sum_{i=1}^{k}\int\softmax\Big(\frac{(\beta_{1i})^{\top}X+\beta_{0i}}{\tau}\Big)\cdot\Big|f(Y|(a_i)^{\top}X+b_i,\nu_i)-f(Y|(\overline{a}_i)^{\top}X+\overline{b}_i,\overline{\nu}_i)\Big|\dint (X,Y)\nonumber\\
        &\leq \sum_{i=1}^{k}\int\Big|f(Y|(a_i)^{\top}X+b_i,\nu_i)-f(Y|(\overline{a}_i)^{\top}X+\overline{b}_i,\overline{\nu}_i)\Big|\dint (X,Y)\nonumber\\
        &\lesssim\sum_{i=1}^{k}(\|a_{i}-\overline{a}_i\|+|b_i-\overline{b}_i|+|\nu_i-\overline{\nu}_i|)\nonumber\\
        &\lesssim\eta,
    \end{align}
    Since $\softmax$ is a Lipschitz function with Lipschitz constant $L\geq 0$, we get
    \begin{align*}
        \|g_{\widetilde{G}}-g_{\overline{G}}\|_{1}&\leq \sum_{i=1}^k\int\Big|\softmax\Big(\frac{(\beta_{1i})^{\top}X+\beta_{0i}}{\tau}\Big)-\softmax\Big(\frac{(\overline{\beta}_{1i})^{\top}X+\overline{\beta}_{0i}}{\overline{\tau}}\Big)\Big|\cdot f(Y|(\overline{a}_i)^{\top}X+\overline{b}_i,\overline{\nu}_i)\dint (X,Y)\nonumber\\
        &\lesssim L\cdot\sum_{i=1}^{k}\int\Big(\Big\|\frac{\beta_{1i}}{\tau}-\frac{\overline{\beta}_{1i}}{\overline{\tau}}\Big\|\cdot\|X\|+\Big|\frac{\beta_{0i}}{\tau}-\frac{\overline{\beta}_{0i}}{\overline{\tau}}\Big|\Big)\dint (X,Y)
    \end{align*}
    where the second inequality follows from the fact that the Gaussian density $f(Y|(\overline{a}_i)^{\top}X+\overline{b}_i,\overline{\nu}_i)$ is bounded. Note that
    \begin{align*}
        \Big\|\frac{\beta_{1i}}{\tau}-\frac{\overline{\beta}_{1i}}{\overline{\tau}}\Big\|=\Big\|\beta_{1i}\Big(\frac{1}{\tau}-\frac{1}{\overline{\tau}}\Big)+\frac{\beta_{1i}-\overline{\beta}_{1i}}{\overline{\tau}}\Big\|\leq\frac{\|\beta_{1i}\|}{\tau\overline{\tau}}\cdot|\tau-\overline{\tau}|+\frac{\|\beta_{1i}-\overline{\beta}_{1i}\|}{\overline{\tau}}\lesssim\eta.
    \end{align*}
    Similarly, we also get that $\Big|\frac{\beta_{0i}}{\tau}-\frac{\overline{\beta}_{0i}}{\overline{\tau}}\Big|\lesssim\eta$. Moreover, since $\mathcal{X}$ is a bounded set, there exists a constant $B>0$ such that $\|X\|\leq B$ for any $X\in\mathcal{X}$. As a result,
    \begin{align}
         \label{eq:second_term}
         \|g_{\widetilde{G}}-g_{\overline{G}}\|_{1}&\lesssim L\cdot\sum_{i=1}^{k}\int (\eta\cdot B+\eta)\dint(X,Y)\leq Lk\eta(B+1).
    \end{align}
    Putting the bounds in equations~\eqref{eq:first_term} and~\eqref{eq:second_term} together with the triangle inequality, we receive that 
    \begin{align*}
        \|g_{G}-g_{\overline{G}}\|_{1}\leq \|g_{G}-g_{\widetilde{G}}\|_{1}+\|g_{\widetilde{G}}-g_{\overline{G}}\|_{1}\lesssim\eta,
    \end{align*}
     which means that $\mathcal{R}$ is an $\eta$-cover (not necessarily smallest) of the metric space $(\mathcal{G}_k(\Theta),\|\cdot\|_1)$. By definition of the covering number, we know that
    \begin{align*}
        N(\eta,\mathcal{G}_k(\Theta),\|\cdot\|_{1})\leq |\mathcal{R}| = |\Delta_{\eta}|\times|\Omega_{\eta}|\leq \mathcal{O}(\eta^{-(d+2)k})\cdot \mathcal{O}(\eta^{(-d+2)k})\leq\mathcal{O}(\eta^{-(2d+4)k}),
    \end{align*}
    which implies that
    \begin{align*}
        \log  N(\eta,\mathcal{G}_k(\Theta),\|\cdot\|_{1})\lesssim \log(1/\eta).
    \end{align*}
    Hence, the proof is completed.

\subsection{Proof of Theorem~\ref{theorem:modified_density_rate}}
\label{appendix:modified_density_rate}
Based on the proof of Theorem~\ref{theorem:density_rate} in Appendix~\ref{appendix:density_rate}, it suffices to establish the following upper bound for the covering number of the metric space $(\mathcal{P}_k(\Theta),\|\cdot\|_1)$, where $\mathcal{P}_k(\Theta):=\{p_{G}(Y|X):G\in\mathcal{O}_k(\Theta)\}$, while other results can be demonstrated in a similar fashion:
\begin{align*}
     \log{N}(\eta,\mathcal{P}_k(\Theta),\|\cdot\|_{1})\lesssim \log(1/\eta).
\end{align*}
Recall that $\Theta$ is a compact set, it follows that $\Delta:=\{(\beta_0,\beta_1,\tau)\in\mathbb{R}\times\mathbb{R}^d\times\mathbb{R}_+:(\beta_0,\beta_1,\tau,a,b,\nu)\in\Theta\}$ and $\Omega:=\{(a,b,\nu)\in\mathbb{R}^d\times\mathbb{R}\times\mathbb{R}_+:(\beta_{0},\beta_{1},\tau,a,b,\nu)\in\Theta\}$ are also compact. Therefore, there exist $\eta$-covers $\Delta_{\eta}$ and ${\Omega}_{\eta}$ for $\Delta$ and $\Omega$, respectively, with a note that $|\Delta_{\eta}|\leq \mathcal{O}(\eta^{-(d+2)k})$ and $|\Omega_{\eta}|\leq \mathcal{O}(\eta^{-(d+2)k})$.
    
    Next, for each mixing measure $G=\sum_{i=1}^k\exp\Big(\frac{\beta_{0i}}{\tau}\Big)\delta_{(\beta_{1i},\tau,a_i,b_i,\nu_i)}\in\mathcal{O}_k(\Theta)$, we consider another one denoted by $\widetilde{G}:=\sum_{i=1}^k\exp\Big(\frac{\beta_{0i}}{\tau}\Big)\delta_{({\beta}_{1i},\tau,\overline{a}_i,\overline{b}_i,\overline{\nu}_i)}$, where $(\overline{a}_i,\overline{b}_i,\overline{\nu}_i)\in{\Omega}_{\eta}$ such that $(\overline{a}_i,\overline{b}_i,\overline{\nu}_i)$ are the closest to $(a_i,b_i,\nu_i)$ in that set for all $i\in[k]$. Additionally, we also take into account the mixing measure $\overline{G}:=\sum_{i=1}^k\exp\Big(\frac{\overline{\beta}_{0i}}{\overline{\tau}}\Big)\delta_{({\overline{\beta}}_{1i},\overline{\tau},\overline{a}_i,\overline{b}_i,\overline{\nu}_i)}$, where $(\overline{\beta}_{0i},\overline{\beta}_{1i},\overline{\tau})\in\Delta_{\eta}$ are the closest to $(\beta_{0i},\beta_{1i},\tau)$ in that set. It can be verified that the conditional density $p_{\overline{G}}$ belongs to the following set:
   \begin{align*}
        \mathcal{R}:=\left\{p_{G}\in\mathcal{G}_k(\Theta):(\beta_{0i},\beta_{1i},\tau)\in\Delta_{\eta},~(a_i,b_i,\nu_i)\in\Omega_{\eta}, \ \forall i\in[k]\right\}.
    \end{align*}
    From the formulation of $\widetilde{G}$, we have that
    \begin{align}
        \label{eq:first_term_over}
        \|p_{G}-p_{\widetilde{G}}\|_{1}&\leq \sum_{i=1}^{k}\int\softmax\Big(\frac{\sigma((\beta_{1i})^{\top}X)+\beta_{0i}}{\tau}\Big)\cdot\Big|f(Y|(a_i)^{\top}X+b_i,\nu_i)-f(Y|(\overline{a}_i)^{\top}X+\overline{b}_i,\overline{\nu}_i)\Big|\dint (X,Y)\nonumber\\
        &\leq \sum_{i=1}^{k}\int\Big|f(Y|(a_i)^{\top}X+b_i,\nu_i)-f(Y|(\overline{a}_i)^{\top}X+\overline{b}_i,\overline{\nu}_i)\Big|\dint (X,Y)\nonumber\\
        &\lesssim\sum_{i=1}^{k}(\|a_{i}-\overline{a}_i\|+|b_i-\overline{b}_i|+|\nu_i-\overline{\nu}_i|)\nonumber\\
        &\lesssim\eta,
    \end{align}
    Since $\softmax$ is a Lipschitz function with Lipschitz constant $L_1\geq 0$, we get
    \begin{align*}
        \|p_{\widetilde{G}}-p_{\overline{G}}\|_{1}&\leq \sum_{i=1}^k\int\Big|\softmax\Big(\frac{\sigma((\beta_{1i})^{\top}X)+\beta_{0i}}{\tau}\Big)-\softmax\Big(\frac{\sigma((\overline{\beta}_{1i})^{\top}X)+\overline{\beta}_{0i}}{\overline{\tau}}\Big)\Big|\cdot f(Y|(\overline{a}_i)^{\top}X+\overline{b}_i,\overline{\nu}_i)\dint (X,Y)\nonumber\\
        &\lesssim L_1\cdot\sum_{i=1}^{k}\int\Big(\Big|\frac{\sigma((\beta_{1i})^{\top}X)}{\tau}-\frac{\sigma((\overline{\beta}_{1i})^{\top}X)}{\overline{\tau}}\Big|+\Big|\frac{\beta_{0i}}{\tau}-\frac{\overline{\beta}_{0i}}{\overline{\tau}}\Big|\Big)\dint (X,Y)
    \end{align*}
    where the second inequality follows from the fact that the Gaussian density $f(Y|(\overline{a}_i)^{\top}X+\overline{b}_i,\overline{\nu}_i)$ is bounded. Note that
    \begin{align*}
        &\Big|\frac{\sigma((\beta_{1i})^{\top}X)}{\tau}-\frac{\sigma((\overline{\beta}_{1i})^{\top}X)}{\overline{\tau}}\Big|\\
        &=\Big|\sigma((\beta_{1i})^{\top}X)\Big(\frac{1}{\tau}-\frac{1}{\overline{\tau}}\Big)+\frac{\sigma((\beta_{1i})^{\top}X)-\sigma((\overline{\beta}_{1i})^{\top}X)}{\overline{\tau}}\Big|\\
        &\leq\frac{|\sigma((\beta_{1i})^{\top}X)|}{\tau\overline{\tau}}\cdot|\tau-\overline{\tau}|+\frac{|\sigma((\beta_{1i})^{\top}X)-\sigma((\overline{\beta}_{1i})^{\top}X)|}{\overline{\tau}}.
    \end{align*}
    Since the function $\sigma$ is differentiable, it is also Lipschitz with some Lipschitz constant $L_2>0$ and $|\sigma((\beta_{1i})^{\top}X)|$ is bounded. Furthermore, as $\mathcal{X}$ is a bounded set, it follows that $\|X\|$ is also bounded. Thus, we get
    \begin{align*}
        &\Big|\frac{\sigma((\beta_{1i})^{\top}X)}{\tau}-\frac{\sigma((\overline{\beta}_{1i})^{\top}X)}{\overline{\tau}}\Big|\\
        &\leq \frac{|\sigma((\beta_{1i})^{\top}X)|}{\tau\overline{\tau}}\cdot|\tau-\overline{\tau}|+L_2\cdot\frac{\|\beta_{1i}-\overline{\beta}_{1i}\|\cdot\|X\|}{\overline{\tau}}\\
        &\leq \frac{|\sigma((\beta_{1i})^{\top}X)|}{\tau\overline{\tau}}\cdot\eta+L_2\cdot\frac{\|X\|}{\overline{\tau}}\cdot\eta\\
        &\lesssim\eta.
    \end{align*}
    Analogously, we also have that $\Big|\frac{\beta_{0i}}{\tau}-\frac{\overline{\beta}_{0i}}{\overline{\tau}}\Big|\lesssim\eta$. Consequently,
    \begin{align}
         \label{eq:second_term_over}
         \|p_{\widetilde{G}}-p_{\overline{G}}\|_{1}&\lesssim L_1\cdot\sum_{i=1}^{k}\int (\eta+\eta)\dint(X,Y)\leq 2L_1k\eta.
    \end{align}
    Putting the bounds in equations~\eqref{eq:first_term_over} and~\eqref{eq:second_term_over} together with the triangle inequality, we receive that 
    \begin{align*}
        \|p_{G}-p_{\overline{G}}\|_{1}\leq \|p_{G}-p_{\widetilde{G}}\|_{1}+\|p_{\widetilde{G}}-p_{\overline{G}}\|_{1}\lesssim\eta,
    \end{align*}
     which means that $\mathcal{R}$ is an $\eta$-cover (not necessarily smallest) of the metric space $(\mathcal{P}_k(\Theta),\|\cdot\|_1)$. By definition of the covering number, we know that
    \begin{align*}
        N(\eta,\mathcal{P}_k(\Theta),\|\cdot\|_{1})\leq |\mathcal{R}| = |\Delta_{\eta}|\times|\Omega_{\eta}|\leq \mathcal{O}(\eta^{-(d+2)k})\cdot \mathcal{O}(\eta^{(-d+2)k})\leq\mathcal{O}(\eta^{-(2d+4)k}),
    \end{align*}
    which implies that
    \begin{align*}
        \log  N(\eta,\mathcal{P}_k(\Theta),\|\cdot\|_{1})\lesssim \log(1/\eta).
    \end{align*}
    Hence, the proof is completed.

\section{Proofs for Parameter Estimation Rates}
\label{appendix:parameter_estimations}
\subsection{Proof of Theorem~\ref{theorem:log_rate_exact}}
\label{appendix:log_rate_exact}
Before going to the main proof of Theorem~\ref{theorem:log_rate_exact} in Appendix~\ref{appendix:main_proof_exact}, let us introduce a key lemma for that proof as follows:
\begin{lemma}
    \label{lemma:minimax_exact}
    For any $r\geq 1$, if the following holds :
    \begin{align*}
        \lim_{\varepsilon\to0}\inf_{G\in\mathcal{E}_{k_*}(\Theta):\mathcal{D}_{1,r}(G,G_*)\leq\varepsilon}\frac{\bbE_X[V(g_{G}(\cdot|X),g_{G_*}(\cdot|X))]}{\mathcal{D}_{1,r}(G,G_*)}=0,
    \end{align*}
    then we achieve that 
    \begin{align*}
        \inf_{\overline{G}_n\in\mathcal{E}_{k_*}(\Theta)}\sup_{G\in\mathcal{E}_{k_*}(\Theta)}\bbE_{g_{G}}[\mathcal{D}_{1,r}(\overline{G}_n,G)]\gtrsim n^{-1/2}.
    \end{align*}
\end{lemma}
Proof of Lemma~\ref{lemma:minimax_exact} is deferred to Appendix~\ref{appendix:minimax_exact}. Now, we are ready to present the main proof of Theorem~\ref{theorem:log_rate_exact}.

\subsubsection{Main Proof}
\label{appendix:main_proof_exact}
Based on the result of Lemma~\ref{lemma:minimax_exact}, it is sufficient to construct a sequence of mixing measures $G_n$ such that $\mathcal{D}_{1,r}(G_n,G_*)\to0$ and 
\begin{align}
    \label{eq:ratio_zero}
    \frac{\bbE_X[V(g_{G_n}(\cdot|X),g_{G_*}(\cdot|X))]}{\mathcal{D}_{1,r}(G_n,G_*)}\to0,
\end{align}
as $n\to\infty$. For that purpose, we choose the following sequence:  $G_n=\sum_{i=1}^{k_*}\exp\Big(\frac{\bzin}{\tau^n}\Big)\delta_{(\boin,\tau^n,\ain,\bin,\vin)}$, where 
\begin{itemize}
    \item $a^n_{i}=a^*_{i}$, $b^n_{i}=b^*_{i}$, $\nu^n_{i}=\nu^*_{i}$ for any $i\in[k_*]$;
    \item $\beta^n_{1i}=\beta^*_{1i}+s_{n,i}$, for any $i\in[k_*]$;
    \item $\tau^n=\tau^*+t_n$;
    \item $\bzin=\Big(1+\dfrac{t_n}{\tau^*}\Big)\bzi$, which implies that $\exp\Big(\frac{\bzin}{\tau^n}\Big)=\exp\Big(\frac{\bzi}{\tau^*}\Big)$, for any $i\in[k_*]$
\end{itemize}
where $s_{n,i}:=(s^{(1)}_{n,i},\ldots,s^{(d)}_{n,i})\in\mathbb{R}^d$ and $t_n\in\mathbb{R}$ will be chosen later such that $s^{(u)}_{n,i}\to0$ and $t_n\to0$ as $n\to\infty$ for any $u\in[d]$ and $i\in[k_*]$. Then, the loss function $\mathcal{D}_{1,r}$ is reduced to
\begin{align}
    \label{eq:D_r_formulation}
    \mathcal{D}_{1,r}(G_n,G_*)=\sum_{i=1}^{k_*}\exp\Big(\frac{\bzi}{\tau^*}\Big)(\|s_{n,i}\|^r+t_n^r).
\end{align}
It is clear that $\mathcal{D}_{1,r}(G_n,G_*)\to0$ as $n\to\infty$. Now, we will show that $\bbE_X[V(g_{G_n}(\cdot|X),g_{G_*}(\cdot|X))]/\mathcal{D}_{1,r}(G_n,G_*)\to0$ as $n\to\infty$. Let us consider the quantity $Q_n:=\Big[\sum_{i=1}^{k_*}\exp\Big(\dfrac{(\boi)^{\top}X+\bzi}{\tau^*}\Big)\Big]\cdot\Big[g_{G_n}(Y|X)-g_{G_*}(Y|X)\Big]$, which can be decomposed as follows:
\begin{align*}
    Q_n&=\sum_{i=1}^{k_*}\exp\Big(\frac{\bzin}{\tau^n}\Big)\Big[\exp\Big(\frac{(\boin)^{\top}X}{\tau^n}\Big)f(Y|(\ain)^{\top}X+\bin,\vin)-\exp\Big(\frac{(\boi)^{\top}X}{\tau^*}\Big)f(Y|(\ai)^{\top}X+\bi,\vi)\Big]\\
    &-\sum_{i=1}^{k_*}\exp\Big(\frac{\bzin}{\tau^n}\Big)\Big[\exp\Big(\frac{(\boin)^{\top}X}{\tau^n}\Big)g_{G_n}(Y|X)-\exp\Big(\frac{(\boi)^{\top}X}{\tau^*}\Big)g_{G_n}(Y|X)\Big]\\
    &+\sum_{i=1}^{k_*}\Big[\exp\Big(\frac{\bzin}{\tau^n}\Big)-\exp\Big(\frac{\bzi}{\tau^*}\Big)\Big]\Big[\exp\Big(\frac{(\boi)^{\top}X}{\tau^*}\Big)f(Y|(\ai)^{\top}X+\bi,\vi)-\exp\Big(\frac{(\boi)^{\top}X}{\tau^*}\Big)g_{G_n}(Y|X)\Big]\\
    &:=A_n-B_n+E_n.
\end{align*}
Given the formulation of $G_n$, the term $A_n$ can be simplified as
\begin{align*}
    A_n=\sum_{i=1}^{k_*}\exp\Big(\frac{\bzi}{\tau^*}\Big)\Big[\exp\Big(\frac{(\boin)^{\top}X}{\tau^n}\Big)-\exp\Big(\frac{(\boi)^{\top}X}{\tau^*}\Big)\Big]f(Y|(\ai)^{\top}X+\bi,\vi)
\end{align*}
By means of first-order Taylor expansions, we can rewrite $A_n$ as
\begin{align*}
    A_n=\sum_{i=1}^{k_*}\sum_{u=1}^{d}\exp\Big(\frac{\bzi}{\tau^*}\Big)\Big[\frac{s^{(u)}_{n,i}}{\tau^*}-\frac{t_n(\boi)^{(u)}}{(\tau^*)^2}\Big]\cdot X^{(u)}\exp\Big(\frac{(\boi)^{\top}X}{\tau^*}\Big) f(Y|(\ai)^{\top}X+\bi,\vi) +R_1(X,Y),
\end{align*}
where $R_1(X,Y)$ is a Taylor remainder such that $R_1(X,Y)/\mathcal{D}_{1,r}(G_n,G_*)\to0$ as $n\to\infty$. Then, by choosing 
\begin{align*}
    t_n=\frac{1}{n}; \qquad s^{(u)}_{n,i}=\dfrac{t_n(\boi)^{(u)}}{\tau^*}=\dfrac{(\boi)^{(u)}}{n\tau^*},
\end{align*}
we obtain that $A_n/\mathcal{D}_{1,r}(G_n,G_*)\to0$ as $n\to\infty$.

Next, we consider the term $B_n$:
\begin{align*}
    B_n=\sum_{i=1}^{k_*}\exp\Big(\frac{\bzi}{\tau^*}\Big)\Big[\exp\Big(\frac{(\boin)^{\top}X}{\tau^n}\Big)-\exp\Big(\frac{(\boi)^{\top}X}{\tau^*}\Big)\Big]g_{G_n}(Y|X).
\end{align*}
By arguing similarly, we also get that $B_n/\mathcal{D}_{1,r}(G_n,G_*)\to0$ as $n\to\infty$. Since we have $E_n=0$, it follows that $Q_n/\mathcal{D}_{1,r}(G_n,G_*)\to0$ as $n\to\infty$. Moreover, as the term $\Big[\sum_{i=1}^{k_*}\exp\Big(\dfrac{(\boi)^{\top}X+\bzi}{\tau^*}\Big)\Big]$ is bounded, we can deduce that $|g_{G_n}(\cdot|X)-g_{G_*}(\cdot|X)|/\mathcal{D}_{1,r}(G_n,G_*)\to0$ as $n\to\infty$ for almost surely $X$. As a consequence, we satisfy the condition in equation~\eqref{eq:ratio_zero}. Hence, the proof is completed.

\subsubsection{Proof of Lemma~\ref{lemma:minimax_exact}}
\label{appendix:minimax_exact}

For a sufficiently small $\varepsilon>0$ and a fixed constant $C_1>0$ that we will choose later, it follows from the assumption that we can find a mixing measure $G'_*\in\mathcal{E}_{k_*}(\Theta)$ that satisfies $\mathcal{D}_{1,r}(G'_*,G_*)=2\varepsilon$ and $\bbE_X[V(g_{G'_*}(\cdot|X),g_{G_*}(\cdot|X))\leq C_1\varepsilon$. Additionally, for any sequence $\overline{G}_n\in\mathcal{E}_{k_*}(\Theta)$, we have
    \begin{align*}
        2\max_{G\in\{G'_*,G_*\}}\bbE_{g_{G}}[\mathcal{D}_{1,r}(\overline{G}_n,G)]\geq \bbE_{g_{G_*}}[\mathcal{D}_{1,r}(\overline{G}_n,G_*)]+\bbE_{g_{G'_*}}[\mathcal{D}_{1,r}(\overline{G}_n,G'_*)],
    \end{align*}
    where $\bbE_{g_{G}}$ stands for the expectation taken w.r.t the product measure with density $g^n_{G}$. Furthermore, since the loss $\mathcal{D}_{1,r}$ satisfies the weak triangle inequality, we can find a constant $C_2>0$ such that
    \begin{align*}
        \mathcal{D}_{1,r}(\overline{G}_n,G_*)+\mathcal{D}_{1,r}(\overline{G}_n,G'_*)\geq C_2\mathcal{D}_{1,r}(G_*,G'_*)=2C_2\varepsilon.
    \end{align*}
    Consequently, it follows that
    \begin{align*}
        \max_{G\in\{G_*,G'_*\}}\bbE_{g_{G}}[\mathcal{D}_{1,r}(\overline{G}_n,G)]&\geq\frac{1}{2}\Big( \bbE_{g_{G_*}}[\mathcal{D}_{1,r}(\overline{G}_n,G_*)]+\bbE_{g_{G'_*}}[\mathcal{D}_{1,r}(\overline{G}_n,G'_*)]\Big)\\
        &\geq C_2\varepsilon\cdot\inf_{f_1,f_2}\left(\bbE_{g_{G_*}}[f_1]+\bbE_{g_{G'_*}}[f_2]\right).
    \end{align*}
    Here, $f_1$ and $f_2$ in the above infimum are measurable functions in terms of $X_1,X_2,\ldots,X_n$ that satisfy $f_1+f_2=1$. By the definition of Total Variation distance, the above infimum value is equal to $1-\bbE_X[V(g^n_{G_*}(\cdot|X),g^n_{G'_*}(\cdot|X))]$. Therefore, we obtain that
    \begin{align*}
        \max_{G\in\{G_*,G'_*\}}\bbE_{g_{G}}[\mathcal{D}_{1,r}(\overline{G}_n,G)]&\geq C_2\varepsilon\Big(1-\bbE_X[V(g^n_{G_*}(\cdot|X),g^n_{G'_*}(\cdot|X))]\Big)\\
        &\geq C_2\varepsilon\Big[1-\sqrt{1-(1-C_1^2\varepsilon^2)^n}\Big].
    \end{align*}
    By choosing $\varepsilon=n^{-1/2}/C_1$, we have $C_1^2\varepsilon^2=\frac{1}{n}$, which implies that
    \begin{align*}
         \sup_{G\in\mathcal{E}_{k_*}(\Theta)\setminus\mathcal{O}_{k_*-1}(\Theta)}\bbE_{g_{G}}[\mathcal{D}_{1,r}(\overline{G}_n,G)]\geq\max_{G\in\{G_*,G'_*\}}\bbE_{g_{G}}[\mathcal{D}_{1,r}(\overline{G}_n,G)]&\gtrsim n^{-1/2},
    \end{align*}
    for any mixing measure $\overline{G}_n\in\mathcal{E}_{k_*}(\Theta)$. Hence, we reach the conclusion of Lemma~\ref{lemma:minimax_exact}, that is,
    \begin{align*}
        \inf_{\overline{G}_n\in\mathcal{E}_{k_*}(\Theta)}\sup_{G\in\mathcal{E}_{k_*}(\Theta)\setminus\mathcal{O}_{k_*-1}(\Theta)}\bbE_{g_{G}}[\mathcal{D}_{1,r}(\overline{G}_n,G)]\gtrsim n^{-1/2},
    \end{align*}
for any $r\geq 1$.

\subsection{Proof of Theorem~\ref{theorem:tight_rate_exact}}
\label{appendix:tight_rate_exact}
In this proof, our main goal is to prove the following inequality:
\begin{align}
    \label{eq:tight_exact_universal_inequality}
    \inf_{G\in\mathcal{E}_{k_*}(\Theta)}\bbE_X[V(g_{G}(\cdot|X),g_{G_*}(\cdot|X))]/\mathcal{D}_2(G^{|\Psi},G^{|\Psi}_*)>0.
\end{align}
For that purpose, we divide the above inequality into local and global parts as below.

\textbf{Local part:} In this part, we aim to establish the following inequality:
\begin{align}
    \label{eq:local_tight_exact}
    \lim_{\varepsilon\to0}\inf_{G\in\mathcal{E}_{k_*}(\Theta):\mathcal{D}_2(G^{|\Psi},G^{|\Psi}_*)\leq\varepsilon}\bbE_X[V(g_{G}(\cdot|X),g_{G_*}(\cdot|X))]/\mathcal{D}_2(G^{|\Psi},G^{|\Psi}_*)>0.
\end{align}
Assume by contrary that the above inequality does not hold true, then there exists a sequence of mixing measures $G_n\in\mathcal{E}_{k_*}(\Theta)$ such that $G^{|\Psi}_n=\sum_{i=1}^{k_*}\exp(\beta^n_{0i}/\tau^n)\delta_{(a^n_i,b^n_i,\nu^n_i)}$ which satisfies $\mathcal{D}_{2n}:=\mathcal{D}_2(G^{|\Psi}_n,G^{|\Psi}_*)\to0$ and
\begin{align}
    \label{eq:ratio_tight_exact}
    \bbE_X[V(g_{G_n}(\cdot|X),g_{G_*}(\cdot|X))]/\mathcal{D}_{2n}\to0,
\end{align}
as $n\to\infty$. Recall that under the exact-specified settings, each Voronoi cell $\mathcal{A}^n_i=\mathcal{A}_i(G_n)$ has only one element. Therefore, we may assume without loss of generality (WLOG) that $\mathcal{A}^n_i=\{i\}$ for any $i\in[k_*]$. Thus, the loss function $\mathcal{D}_{2n}$ is reduced to
\begin{align}
    \label{eq:new_loss_tight_exact}
    \mathcal{D}_{2n}:=
    \sum_{i=1}^{k_*}\exp\Big(\frac{\bzin}{\tau^n}\Big)\Big[\|\dain\|+|\dbin|+|\dvin|\Big]+\sum_{i=1}^{k_*}\Big|\exp\Big(\frac{\bzin}{\tau^n}\Big)-\exp\Big(\frac{\bzj}{\tau^*}\Big)\Big|
\end{align}
Since $\mathcal{D}_{2n}\to0$, we get that $(\ain,\bin,\vin)\to(\ai,\bi,\vi)$ and $\exp(\bzin/\tau^n)\to\exp(\bzi/\tau^*)$ as $n\to\infty$. Now, we separate the proof of local part into three steps as follows:

\textbf{Step 1.} In this step, we decompose the quantity $Q_n:=\Big[\sum_{i=1}^{k_*}\exp\Big(\dfrac{(\boi)^{\top}X+\bzi}{\tau^*}\Big)\Big]\cdot[g_{G_n}(Y|X)-g_{G_*}(Y|X)]$ into a linear combination of linearly independent terms. For the ease of presentation, let us denote $F(Y;X,\beta_1,\tau,a,b,\nu):=\exp\Big(\frac{\beta_{1}^{\top}X}{\tau}\Big)f(Y|a^{\top}X+b,\nu)$ and $H(Y;X,\beta_{1},\tau):=\exp\Big(\frac{\sigma(\beta_{1}^{\top}X)}{\tau}\Big)g_{G_n}(Y|X)$. Then, it can be checked that
\begin{align*}
    Q_n&=\sum_{i=1}^{k_*}\exp\Big(\frac{\bzin}{\tau^n}\Big)\Big[F(Y;X,\boin,\tau^n,\ain,\bin,\vin)-F(Y;X,\boi,\tau^*,\ai,\bi,\vi)\Big]\\
    &-\sum_{i=1}^{k_*}\exp\Big(\frac{\bzin}{\tau^n}\Big)\Big[H(Y;X,\boin,\tau^n)-H(Y;X,\boi,\tau^*)\Big]\\
    &+\sum_{i=1}^{k_*}\Big[\exp\Big(\frac{\bzin}{\tau^n}\Big)-\exp\Big(\frac{\bzi}{\tau^*}\Big)\Big]\cdot\exp\Big(\frac{(\boi)^{\top}X}{\tau^*}\Big)f(Y|(\ai)^{\top}X+\bi,\vi)\\
    &-\sum_{i=1}^{k_*}\Big[\exp\Big(\frac{\bzin}{\tau^n}\Big)-\exp\Big(\frac{\bzi}{\tau^*}\Big)\Big]\cdot\exp\Big(\frac{(\boi)^{\top}X}{\tau^*}\Big)g_{G_n}(Y|X)\\
    :&=A_n-B_n+E_{n,1}-E_{n,2}.
\end{align*}
Next, by means of the first-order Taylor expansion, we get that
\begin{align*}
    A_n&=\sum_{i=1}^{k_*}\exp\Big(\frac{\bzin}{\tau^n}\Big)\sum_{|\alpha|=1}(\dboin)^{\alpha_1}(\dtn)^{\alpha_2}(\dain)^{\alpha_3}(\dbin)^{\alpha_4}(\dvin)^{\alpha_5}\\
    &\hspace{5cm}\times \frac{\partial F}{\partial\beta_{1}^{\alpha_1}~\partial\tau^{\alpha_2}~\partial a^{\alpha_3}~\partial b^{\alpha_4}~\partial \nu^{\alpha_5}}(Y;X,\boi,\tau^*,\ai,\bi,\vi) + R_1(X,Y),
\end{align*}
where $R_1(X,Y)$ is a Taylor remainder such that $R_1(X,Y)/\mathcal{D}_{2n}\to0$ as $n\to\infty$. Let us denote
\begin{align*}
    F^{(\eta)}(Y;X,\omega^*_i):=\exp\Big(\frac{(\boi)^{\top}X}{\tau^*}\Big)\cdot\frac{\partial^{\eta}f}{\partial h_1^{\eta}}(Y|(\ai)^{\top}X+\bi,\vi),
\end{align*}
for any $\eta\in\mathbb{N}$, where $\omega^*_i:=(\boi,\tau^*,\ai,\bi,\vi)$. Then, the first derivatives of function $F$ w.r.t its parameters are given by
\begin{align}
    \frac{\partial F}{\partial\beta_{1}^{(u)}}(Y;X,\omega^*_i)&=\frac{X^{(u)}}{\tau^*}F(Y;X,\omega^*_i),\qquad 
    \frac{\partial F}{\partial\tau}(Y;X,\omega^*_i)=-\frac{(\boi)^{\top}X}{(\tau^*)^2}F(Y;X,\omega^*_i),\nonumber\\
    \frac{\partial F}{\partial a^{(u)}}(Y;X,\omega^*_i)&=X^{(u)}F^{(1)}(Y;X,\omega^*_i),\quad 
    \frac{\partial F}{\partial b}(Y;X,\omega^*_i)=F^{(1)}(Y;X,\omega^*_i),\quad 
    \label{eq:F_first_derivatives_tight}
    \frac{\partial F}{\partial \nu}(Y;X,\omega^*_i)=\frac{1}{2}F^{(2)}(Y;X,\omega^*_i).
\end{align}
From this result, we can rewrite $A_n$ as
\begin{align*}
    A_n&=\sum_{i=1}^{k_*}\exp\Big(\frac{\bzin}{\tau^n}\Big)\Big[\sum_{u=1}^{d}\Big(\frac{(\dboin)^{(u)}}{\tau^*}-\frac{(\dtn)(\boi)^{(u)}}{(\tau^*)^2}\Big)X^{(u)}F(Y;X,\omega^*_i) +\sum_{u=1}^{d}(\dain)^{(u)}X^{(u)}F^{(1)}(Y;X,\omega^*_i)\\
    &\hspace{3cm}+(\dbin)F^{(1)}(Y;X,\omega^*_i)+\frac{1}{2}(\dvin)F^{(2)}(Y;X,\omega^*_i)\Big]+R_1(X,Y),
\end{align*}

Analogously, we also apply the first-order Taylor expansion to the term $B_n$ and get that
\begin{align*}
    B_n&=\sum_{i=1}^{k_*}\exp\Big(\frac{\bzin}{\tau^n}\Big)\sum_{|\gamma|=1}(\dboin)^{\gamma_1}(\dtn)^{\gamma_2}\cdot\frac{\partial H}{\partial \beta_1^{\gamma_1}~\partial\tau^{\gamma_2}}(Y;X,\boi,\tau^*) + R_2(X,Y),\\
    &=\sum_{i=1}^{k_*}\exp\Big(\frac{\bzin}{\tau^n}\Big)\sum_{u=1}^{d}\Big(\frac{\dboin}{\tau^*}-\frac{(\boi)^{(u)}}{(\tau^*)^2}\Big)X^{(u)}H(Y;X,\boi,\tau^*)+ R_2(X,Y)
\end{align*}
where $R_2(X,Y)$ is a Taylor remainder such that $R_2(X,Y)/\mathcal{D}_{2n}\to0$ as $n\to\infty$. 

As a result, we can represent $Q_n$ as
\begin{align}
    \label{eq:Qn_represent}
    Q_n=\sum_{i=1}^{k_*}\sum_{\eta=0}^{2}C_{n,\eta,i}(X)F^{(\eta)}(Y;X,\omega^*_i)-\sum_{i=1}^{k_*}C_{n,0,i}(X)H(Y;X,\boi,\tau^*)+R_1(X,Y)-R_2(X,Y),
\end{align}
where we define
\begin{align*}
    C_{n,0,i}(X)&:=\Big[\exp\Big(\frac{\bzin}{\tau^n}\Big)-\exp\Big(\frac{\bzi}{\tau^*}\Big)\Big]+\exp\Big(\frac{\bzin}{\tau^n}\Big)\sum_{u=1}^{d}\Big[\frac{(\dboin)^{(u)}}{\tau^*}-\frac{(\dtn)(\boi)^{(u)}}{(\tau^*)^2}\Big]\cdot X^{(u)},\nonumber\\
    C_{n,1,i}(X)&:=\exp\Big(\frac{\bzin}{\tau^n}\Big)\Big[\sum_{u=1}^{d}(\dain)^{(u)} X^{(u)}+(\dbin)\Big],\nonumber\\
    C_{n,2,i}(X)&:=\exp\Big(\frac{\bzin}{\tau^n}\Big)\cdot\frac{(\dvin)}{2}.
\end{align*}
From the above results, we can treat $[Q_n-R_1(X,Y)+R_2(X,Y)]/\mathcal{D}_{2n}$ as a combination of elements from the following set:
\begin{align*}
    \Big\{F(Y;X,\omega^*_i), \ X^{(u)}F(Y;X,\omega^*_i), \ F^{(1)}(Y;X,\omega^*_i), \ X^{(u)}F^{(1)}(Y;X,\omega^*_i), \ F^{(2)}(Y;X,\omega^*_i): u\in[d], i\in[k_*]\Big\}\\
    \cup~\Big\{H(Y;X,\boi,\tau^*), \ X^{(u)}H(Y;X,\boi,\tau^*):u\in[d], i\in[k_*]\Big\}.
\end{align*}
\textbf{Step 2.} In this step, we demonstrate that at least one among the coefficients in the representation of $[Q_n-R_1(X,Y)+R_2(X,Y)]/\mathcal{D}_{2n}$ does not converge to zero when $n\to\infty$. Assume by contrary that all of them go to 0 as $n\to\infty$. By taking the summation of the absolute values of the coefficients of $F(Y;X,\omega^*_i)$, we get that
\begin{align}
    \label{eq:limit_bias_tight_exact}
    \frac{1}{\mathcal{D}_{2n}}\cdot\sum_{i=1}^{k_*}\Big|\exp\Big(\frac{\bzin}{\tau^n}\Big)-\exp\Big(\frac{\bzi}{\tau^*}\Big)\Big|\to0.
\end{align}
Next, by taking the summation of the absolute values of the coefficients associated with 
\begin{itemize}
    \item $X^{(u)}F^{(1)}(Y;X,\omega^*_i)$: we have that $\frac{1}{\mathcal{D}_{2n}}\cdot\sum_{i=1}^{k_*}\exp\Big(\frac{\bzin}{\tau^n}\Big)\|\dain\|_1\to0$;
    \item $F^{(1)}(Y;X,\omega^*_i)$: we have that $\frac{1}{\mathcal{D}_{2n}}\cdot\sum_{i=1}^{k_*}\exp\Big(\frac{\bzin}{\tau^n}\Big)|\dbin|\to0$;
    \item $F^{(2)}(Y;X,\omega^*_i)$: we have that $\frac{1}{\mathcal{D}_{2n}}\cdot\sum_{i=1}^{k_*}\exp\Big(\frac{\bzin}{\tau^n}\Big)|\dvin|\to0$;
\end{itemize}
Due to the topological equivalence between $\ell_1$-norm and $\ell_2$-norm, it follows that
\begin{align}
    \label{eq:limit_parameter_tight_exact}
        1=\frac{1}{\mathcal{D}_{2n}}\cdot\sum_{i=1}^{k_*}\exp\Big(\frac{\bzin}{\tau^n}\Big)\Big(\|\dain\|+|\dbin|+|\dvin|\Big)\to0,
\end{align}
which is a contradiction. Consequently, not all the coefficients in the representation of $[Q_n-R_1(X,Y)+R_2(X,Y)]/\mathcal{D}_{2n}$ converge to zero when $n\to\infty$.

\textbf{Step 3.} In this step, we leverage the Fatou's lemma to show a result contradicting to that in Step 2. In particular, by the Fatou's lemma, we have
\begin{align*}
    \lim_{n\to\infty}\frac{ \bbE_X[V(g_{G_n}(\cdot|X),g_{G_*}(\cdot|X))]}{\mathcal{D}_{2n}}\geq\int\liminf_{n\to\infty}\frac{|g_{G_n}(Y|X)-g_{G_*}(Y|X)|}{2\mathcal{D}_{2n}}\dint(X,Y).
\end{align*}
Moreover, recall from the hypothesis in equation~\eqref{eq:ratio_tight_exact} that $\bbE_X[V(g_{G_n}(\cdot|X),g_{G_*}(\cdot|X))]/\mathcal{D}_{2n}\to0$ as $n\to\infty$. Therefore, we deduce that $$\frac{|g_{G_n}(Y|X)-g_{G_*}(Y|X)|}{\mathcal{D}_{2n}}\to0,$$ 
for almost surely $(X,Y)$. Since the term $\Big[\sum_{i=1}^{k_*}\exp\Big(\dfrac{(\boi)^{\top}X+\bzi}{\tau^*}\Big)\Big]$ is bounded, we also have that $\dfrac{Q_n}{\mathcal{D}_{2n}}\to0$ as $n\to\infty$. Following from the results in equation~\eqref{eq:Qn_represent}, $Q_n$ can be represented as
\begin{align*}
    Q_n&=\sum_{i=1}^{k_*}\sum_{\eta=0}^{2}C_{n,\eta,i}(X)F^{(\eta)}(Y;X,\omega^*_i)-\sum_{i=1}^{k_*}C_{n,0,i}(X)H(Y;X,\boi,\tau^*)+R_1(X,Y)-R_2(X,Y).
\end{align*}
Since $R_1(X,Y)/\mathcal{D}_{2n}\to0$ and $R_2(X,Y)/\mathcal{D}_{2n}\to0$ as $n\to\infty$, we can deduce that $C_{n,\eta,i}(X)/\mathcal{D}_{2n}$ must be bounded for any $\eta\in\{0,1,2\}$. Indeed, if at least one among them is not bounded, then that ratio will go to infinity, implying that $Q_n/\mathcal{D}_{2n}\not\to0$, which is a contradiction. Thus, for each $\eta\in\{0,1,2\}$, we can replace $C_{n,\eta,i}(X)$ by one of its subsequences such that the ratio $C_{n,\eta,i}(X)/\mathcal{D}_{2n}$ has a finite limit as $n\to\infty$. Let us denote, 
\begin{align*}
     \frac{1}{\mathcal{D}_{2n}}\cdot\Big[\exp\Big(\frac{\bzin}{\tau^n}\Big)-\exp\Big(\frac{\bzi}{\tau^*}\Big)\Big]&\to\phi_{0,i},\qquad 
    \frac{1}{\mathcal{D}_{2n}}\cdot\exp\Big(\frac{\bzin}{\tau^n}\Big)(\dboin)^{(u)}\to\phi^{(u)}_{1,i},\\
    \frac{1}{\mathcal{D}_{2n}}\cdot\exp\Big(\frac{\bzin}{\tau^n}\Big)(\dtn)&\to\phi_{2,i},\qquad
    \frac{1}{\mathcal{D}_{2n}}\cdot\exp\Big(\frac{\bzin}{\tau^n}\Big)(\dain)^{(u)}\to\phi^{(u)}_{3,i},\\
    \frac{1}{\mathcal{D}_{2n}}\cdot\exp\Big(\frac{\bzin}{\tau^n}\Big)(\dbin)&\to\phi_{4,i},\qquad
    \frac{1}{\mathcal{D}_{2n}}\cdot\exp\Big(\frac{\bzin}{\tau^n}\Big)(\dvin)\to\phi_{5,i}.
\end{align*}
Then, we have
\begin{align*}
    \frac{Q_n}{\mathcal{D}_{2n}}\to\sum_{i=1}^{k_*}\sum_{\eta=0}^{2}C^*_{\eta,i}(X)\cdot F^{(\eta)}(Y;X,\omega^*_i)-\sum_{i=1}^{k_*}C^*_{0,i}(X)\cdot H(Y;X,\boi,\tau^*),
\end{align*}
as $n\to\infty$, for almost surely $(X,Y)$, where we define
\begin{align}
    \label{eq:first_limit_tight_exact}
    C^*_{0,i}(X)&:=\phi_{0,i}+\sum_{u=1}^{d}\Big[\frac{\phi^{(u)}_{1,i}}{\tau^*}-\frac{\phi_{2,i}}{(\tau^*)^2}\cdot(\boi)^{(u)}\Big]\cdot X^{(u)},\\
    \label{eq:second_limit_tight_exact}
    C^*_{1,i}(X)&:=\sum_{u=1}^{d}\phi_{3,i}^{(u)}\cdot X^{(u)}+\phi_{4,i},\\
    \label{eq:third_limit_tight_exact}
    C^*_{2,i}(X)&:=\frac{1}{2}\phi_{5,i},
\end{align}
for any $i\in[k_*]$. In other words, we have
\begin{align*}
    \sum_{i=1}^{k_*}\sum_{\eta=0}^{2}C^*_{\eta,i}(X)\cdot F^{(\eta)}(Y;X,\omega^*_i)-\sum_{i=1}^{k_*}C^*_{0,i}(X)\cdot H(Y;X,\boi,\tau^*)=0,
\end{align*}
for almost surely $(X,Y)$. Since the set $\{F^{(\eta)}(Y;X,\omega^*_i), \ H(Y;X,\boi,\tau^*):\eta\in\{0,1,2\}, \ i\in[k_*]\}$ is linearly independent, we achieve that $C^*_{\eta,i}(X)=0$ for almost surely $X$ for any $\eta\in\{0,1,2\}$. As $C^*_{1,i}(X)=0$, we deduce that $\phi_{3,i}^{(u)}=\phi_{3,i}=0$ for any $u\in[d]$ and $i\in[k_*]$. Next, since $C^*_{2,i}(X)=0$, we have that $\phi_{5,i}=0$ for any $i\in[k_*]$. However, it follows from the results in Step 2 that at least one among $\phi_{3,i}^{(u)}$, $\phi_{4,i}$ and $\phi_{5,i}$ must be different from zero, which is a contradiction. Hence, we achieve the local inequality in equation~\eqref{eq:local_tight_exact}. Therefore, we can find a constant $\varepsilon'>0$ such that
\begin{align*}
    \inf_{G\in\mathcal{E}_{k_*}(\Theta):\mathcal{D}_2(G^{|\Psi},G^{|\Psi}_*)\leq\varepsilon'}\bbE_X[V(g_{G}(\cdot|X),g_{G_*}(\cdot|X))]/\mathcal{D}_2(G^{|\Psi},G^{|\Psi}_*)>0.
\end{align*}
\textbf{Global part.} As a consequence, it suffices to demonstrate the following inequality:
\begin{align}
    \label{global_tight_exact}
    \inf_{G\in\mathcal{E}_{k_*}(\Theta):\mathcal{D}_2(G^{|\Psi},G^{|\Psi}_*)>\varepsilon'}\bbE_X[V(g_{G}(\cdot|X),g_{G_*}(\cdot|X))]/\mathcal{D}_2(G^{|\Psi},G^{|\Psi}_*)>0.
\end{align}
Assume by contrary that the above claim does not hold true, then we can seek a sequence of mixing measures $G'_n\in\mathcal{E}_{k_*}(\Omega)$ such that $\mathcal{D}_2((G'_n)^{|\Psi},G^{|\Psi}_*)>\varepsilon'$ and
\begin{align*}
    \lim_{n\to\infty}\frac{\bbE_X[V(g_{G'_n}(\cdot|X),g_{G_*}(\cdot|X))]}{\mathcal{D}_2((G'_n)^{|\Psi},G^{|\Psi}_*)}=0,
\end{align*}
which directly implies that $\bbE_X[V(g_{G'_n}(\cdot|X),g_{G_*}(\cdot|X))]\to0$ as $n\to\infty$. Recall that $\Theta$ is a compact set, therefore, we can replace the sequence $G'_n$ by one of its subsequences that converges to a mixing measure $G'\in\mathcal{E}_{k_*}(\Omega)$. Since $\mathcal{D}_2((G'_n)^{|\Psi},G^{|\Psi}_*)>\varepsilon'$, this result induces that $\mathcal{D}_2((G')^{|\Psi},G^{|\Psi}_*)>\varepsilon'$. 

Next, by invoking the Fatou's lemma, it follows that
\begin{align*}
    0=\lim_{n\to\infty}\bbE_X[2V(g_{G'_n}(\cdot|X),g_{G_*}(\cdot|X))]\geq \int\liminf_{n\to\infty}\Big|g_{G'_n}(Y|X)-g_{G_*}(Y|X)\Big|~\dint(X,Y).
\end{align*}
Thus, we get that $p_{G'}(Y|X)=g_{G_*}(Y|X)$ for almost surely $(X,Y)$. From Proposition~\ref{prop:standard_identifiability}, we know that the model~\eqref{eq:standard_density} is identifiable, which indicates that $G'\equiv G_*$. As a consequence, we have that $\mathcal{D}_2((G')^{|\Psi},G^{|\Psi}_*)=0$, contradicting the fact that $\mathcal{D}_2((G')^{|\Psi},G^{|\Psi}_*)>\varepsilon'>0$. 

Hence, the proof is completed.

\subsection{Proof of Theorem~\ref{theorem:log_rate_over}}
\label{appendix:log_rate_over}
First of all, we provide a useful lemma that will be utilized for this proof as follows:
\begin{lemma}
    \label{lemma:minimax_over}
    For any $r\geq 1$, if the following holds :
    \begin{align*}
        \lim_{\varepsilon\to0}\inf_{G\in\mathcal{G}_{k}(\Theta):\mathcal{D}_{3,r}(G,G_*)\leq\varepsilon}\frac{\bbE_X[V(g_{G}(\cdot|X),g_{G_*}(\cdot|X))]}{\mathcal{D}_{3,r}(G,G_*)}=0,
    \end{align*}
    then we achieve that 
    \begin{align*}
        \inf_{\overline{G}_n\in\mathcal{G}_{k}(\Theta)}\sup_{G\in\mathcal{G}_{k}(\Theta)\setminus\mathcal{O}_{k_*-1}(\Theta)}\bbE_{g_{G}}[\mathcal{D}_{3,r}(\overline{G}_n,G)]\gtrsim n^{-1/2}.
    \end{align*}
\end{lemma}
The proof of Lemma~\ref{lemma:minimax_over} can be done similarly as in Appendix~\ref{appendix:minimax_exact}. Following from this lemma, it suffices to build a sequence of mixing measures $G_n$ that satisfies $\mathcal{D}_{3,r}(G_n,G_*)\to0$ and
\begin{align}
    \label{eq:ratio_zero_over}
    \frac{\bbE_X[V(g_{G_n}(\cdot|X),g_{G_*}(\cdot|X))]}{\mathcal{D}_{3,r}(G_n,G_*)}\to0,
\end{align}
as $n\to\infty$. To this end, we take into account the mixing measure sequence $G_n=\sum_{i=1}^{k_*}\exp\Big(\frac{\bzin}{\tau^n}\Big)\delta_{(\boin,\tau^n,\ain,\bin,\vin)}$, where we define for any $j\in[k_*]$ that
\begin{itemize}
    \item $a^n_{i}=a^*_{j}$, $b^n_{i}=b^*_{j}$, $\nu^n_{i}=\nu^*_{j}$ for any $i\in\mathcal{A}_j$;
    \item $\beta^n_{1i}=\beta^*_{1j}+s_{n,j}$, for any $i\in\mathcal{A}_j$;
    \item $\tau^n=\tau^*+t_n$;
    \item $\bzin=\tau^n\cdot\Big[\frac{\bzj}{\tau^*}-\log(|\mathcal{A}_j|)\Big]$, which implies that $\sum_{i\in\mathcal{A}_j}\exp\Big(\frac{\bzin}{\tau^n}\Big)=\exp\Big(\frac{\bzj}{\tau^*}\Big)$, for any $i\in\mathcal{A}_j$,
\end{itemize}
where $s_{n,j}:=(s^{(1)}_{n,j},\ldots,s^{(d)}_{n,j})\in\mathbb{R}^d$ and $t_n\in\mathbb{R}$ will be chosen later such that $s^{(u)}_{n,j}\to0$ and $t_n\to0$ as $n\to\infty$ for any $u\in[d]$ and $j\in[k_*]$. Then, the loss function $\mathcal{D}_{3,r}$ is reduced to
\begin{align*}
    \mathcal{D}_{3,r}(G_n,G_*)=\sum_{j=1}^{k_*}|\mathcal{A}_j|\cdot\exp\Big(\frac{\bzj}{\tau^*}\Big)(\|s_{n,j}\|^r+t_n^r).
\end{align*}
Obviously, we have that $\mathcal{D}_{3,r}(G_n,G_*)\to0$ as $n\to\infty$. 

Now, we will show that $\bbE_X[V(g_{G_n}(\cdot|X),g_{G_*}(\cdot|X))]/\mathcal{D}_{3,r}(G_n,G_*)\to0$ as $n\to\infty$. Let us consider the quantity $Q_n:=\Big[\sum_{j=1}^{k_*}\exp\Big(\dfrac{(\boi)^{\top}X+\bzi}{\tau^*}\Big)\Big]\cdot\Big[g_{G_n}(Y|X)-g_{G_*}(Y|X)\Big]$, which can be represented as as follows:
\begin{align*}
    Q_n&=\sum_{j=1}^{k_*}\sum_{i\in\mathcal{A}_j}\exp\Big(\frac{\bzin}{\tau^n}\Big)\Big[\exp\Big(\frac{(\boin)^{\top}X}{\tau^n}\Big)f(Y|(\ain)^{\top}X+\bin,\vin)-\exp\Big(\frac{(\boj)^{\top}X}{\tau^*}\Big)f(Y|(\aj)^{\top}X+\bj,\vj)\Big]\\
    &-\sum_{j=1}^{k_*}\sum_{i\in\mathcal{A}_j}\exp\Big(\frac{\bzin}{\tau^n}\Big)\Big[\exp\Big(\frac{(\boin)^{\top}X}{\tau^n}\Big)g_{G_n}(Y|X)-\exp\Big(\frac{(\boj)^{\top}X}{\tau^*}\Big)g_{G_n}(Y|X)\Big]\\
    &+\sum_{j=1}^{k_*}\Big[\sum_{i\in\mathcal{A}_j}\exp\Big(\frac{\bzin}{\tau^n}\Big)-\exp\Big(\frac{\bzj}{\tau^*}\Big)\Big]\Big[\exp\Big(\frac{(\boj)^{\top}X}{\tau^*}\Big)f(Y|(\aj)^{\top}X+\bj,\vj)-\exp\Big(\frac{(\boj)^{\top}X}{\tau^*}\Big)g_{G_n}(Y|X)\Big]\\
    &:=A_n-B_n+E_n.
\end{align*}
Following from the formulation of $G_n$, we can rewrite the term $A_n$ as
\begin{align*}
    A_n=\sum_{i=1}^{k_*}\sum_{i\in\mathcal{A}_j}\exp\Big(\frac{\bzin}{\tau^n}\Big)\Big[\exp\Big(\frac{(\boin)^{\top}X}{\tau^n}\Big)-\exp\Big(\frac{(\boj)^{\top}X}{\tau^*}\Big)\Big]f(Y|(\aj)^{\top}X+\bj,\vj).
\end{align*}
By means of first-order Taylor expansions, we can rewrite $A_n$ as
\begin{align*}
    A_n=\sum_{j=1}^{k_*}\sum_{i\in\mathcal{A}_j}\sum_{u=1}^{d}\exp\Big(\frac{\bzin}{\tau^n}\Big)\Big[\frac{s^{(u)}_{n,j}}{\tau^*}-\frac{t_n(\boj)^{(u)}}{(\tau^*)^2}\Big]\cdot X^{(u)}\exp\Big(\frac{(\boj)^{\top}X}{\tau^*}\Big) f(Y|(\aj)^{\top}X+\bj,\vj) +R_1(X,Y),
\end{align*}
where $R_1(X,Y)$ is a Taylor remainder such that $R_1(X,Y)/\mathcal{D}_{1,r}(G_n,G_*)\to0$ as $n\to\infty$. Then, by choosing 
\begin{align*}
    t_n=\frac{1}{n}; \qquad s^{(u)}_{n,j}=\dfrac{t_n(\boj)^{(u)}}{\tau^*}=\dfrac{(\boj)^{(u)}}{n\tau^*},
\end{align*}
we obtain that $A_n/\mathcal{D}_{3,r}(G_n,G_*)\to0$ as $n\to\infty$.

By arguing in the same fashion, we also get that $B_n/\mathcal{D}_{3,r}(G_n,G_*)\to0$ as $n\to\infty$. As we have $E_n=0$, it follows that $Q_n/\mathcal{D}_{3,r}(G_n,G_*)\to0$ as $n\to\infty$. Moreover, since the term $\Big[\sum_{j=1}^{k_*}\exp\Big(\dfrac{(\boj)^{\top}X+\bzj}{\tau^*}\Big)\Big]$ is bounded, we can deduce that $|g_{G_n}(Y|X)-g_{G_*}(Y|X)|/\mathcal{D}_{3,r}(G_n,G_*)\to0$ as $n\to\infty$ for almost surely $(X,Y)$. As a consequence, we satisfy the condition in equation~\eqref{eq:ratio_zero_over}. Hence, the proof is completed.

\subsection{Proof of Theorem~\ref{theorem:tight_rate_over}}
\label{appendix:tight_rate_over}

Analogous to the proof of Theorem~\ref{theorem:tight_rate_exact} in Appendix~\ref{appendix:tight_rate_exact}, we aim to prove the following inequality:
\begin{align}
    \label{eq:inequality_tight_over}
    \inf_{G\in\mathcal{G}_{k}(\Theta)}\bbE_X[V(g_{G}(\cdot|X),g_{G_*}(\cdot|X))]/\mathcal{D}_4(G^{|\Upsilon},G^{|\Upsilon}_*)>0.
\end{align}
Moreover, we also divide the above inequality into local and global parts. Since the global part can be argued in the same fashion as in Appendix~\ref{appendix:tight_rate_exact}, we will demonstrate only the local part, that is
\begin{align}
    \label{eq:local_tight_over}
    \lim_{\varepsilon\to0}\inf_{G\in\mathcal{G}_{k}(\Theta):\mathcal{D}_4(G^{|\Upsilon},G^{|\Upsilon}_*)\leq\varepsilon}\bbE_X[V(g_{G}(\cdot|X),g_{G_*}(\cdot|X))]/\mathcal{D}_4(G^{|\Upsilon},G^{|\Upsilon}_*)>0.
\end{align}
Assume by contrary that the above claim does not hold true, then we can find a sequence of mixing measures $G_n=\sum_{i=1}^{k_*}\exp(\beta^n_{0i}/\tau^n)\delta_{(\beta^n_{1i},\tau^n,a^n_i,b^n_i,\nu^n_i)}$ in $\mathcal{G}_{k}(\Theta)$ that satisfies $\mathcal{D}_{4n}:=\mathcal{D}_4(G^{|\Upsilon}_n,G^{|\Upsilon}_*)\to0$ and
\begin{align}
    \label{eq:ratio_tight_over}
    \bbE_X[V(g_{G_n}(\cdot|X),g_{G_*}(\cdot|X))]/\mathcal{D}_{4n}\to0,
\end{align}
as $n\to\infty$. Let us denote $\mathcal{A}^n_j=\mathcal{A}_j(G_n)$, then the loss function $\mathcal{D}_{6n}$ is reduced to
\begin{align}
    \label{eq:new_loss_tight_over}
    \mathcal{D}_{4n}:=&\sum_{j:|\mathcal{A}_j|>1}\sum_{i\in\mathcal{A}_j}\exp\Big(\frac{\bzin}{\tau^n}\Big)\Big[|\dbijn|^{\brj}+|\dvijn|^{\brj/2}\Big]\nonumber\\
    &+\sum_{j:|\mathcal{A}_j|=1}\sum_{i\in\mathcal{A}_j}\exp\Big(\frac{\bzin}{\tau^n}\Big)\Big[|\dbijn|+|\dvijn|\Big]+\sum_{j=1}^{k_*}\Big|\sum_{i\in\mathcal{A}_j}\exp\Big(\frac{\bzin}{\tau^n}\Big)-\exp\Big(\frac{\bzj}{\tau^*}\Big)\Big|
\end{align}
As $\mathcal{D}_{4n}\to0$, we get that $(\bin,\vin)\to(\bj,\vj)$ and $\sum_{i\in\mathcal{A}_j}\exp(\bzin/\tau^n)\to\exp(\bzj/\tau^*)$ as $n\to\infty$ for any $i\in\mathcal{A}_j$ and $j\in[k_*]$. Now, we divide the proof of local part into three steps as follows:

\textbf{Step 1.} In this step, we decompose the quantity $Q_n:=\Big[\sum_{j=1}^{k_*}\exp\Big(\dfrac{\boj X+\bzj}{\tau^*}\Big)\Big]\cdot[g_{G_n}(Y|X)-g_{G_*}(Y|X)]$ into a linear combination of linearly independent terms. Firstly, let $F(Y;X,\beta_{1},\tau,a,b,\nu):=\exp\Big(\frac{\beta_{1}X}{\tau}\Big)f(Y|aX+b,\nu)$ and $H(Y;X,\beta_{1},\tau):=\exp\Big(\frac{\beta_{1}X}{\tau}\Big)g_{G_n}(Y|X)$. Then, it can be verified that
\begin{align}
    \label{eq:Q_n_tight_formulation}
    Q_n&=\sum_{j=1}^{k_*}\sum_{i\in\mathcal{A}_j}\exp\Big(\frac{\bzin}{\tau^n}\Big)\Big[F(Y;X,\boin,\tau^n,\ain,\bin,\vin)-F(Y;X,\boj,\tau^*,\aj,\bj,\vj)\Big]\nonumber\\
    &-\sum_{j=1}^{k_*}\sum_{i\in\mathcal{A}_j}\exp\Big(\frac{\bzin}{\tau^n}\Big)\Big[H(Y;X,\boin,\tau^n)-H(Y;X,\boj,\tau^*)\Big]\nonumber\\
    &+\sum_{j=1}^{k_*}\Big[\sum_{i\in\mathcal{A}_j}\exp\Big(\frac{\bzin}{\tau^n}\Big)-\exp\Big(\frac{\bzi}{\tau^*}\Big)\Big]\cdot\exp(\boj X)f(Y|\aj X+\bj,\vj)\nonumber\\
    &-\sum_{j=1}^{k_*}\Big[\sum_{i\in\mathcal{A}_j}\exp\Big(\frac{\bzin}{\tau^n}\Big)-\exp\Big(\frac{\bzi}{\tau^*}\Big)\Big]\cdot\exp(\boj X)g_{G_n}(Y|X)\nonumber\\
    :&=A_n-B_n+E_{n,1}-E_{n,2}.
\end{align}
Next, we continue to separate $A_n$ into two terms as follows:
\begin{align*}
    A_n:&=\sum_{j:|\mathcal{A}_j|=1}\sum_{i\in\mathcal{A}_j}\exp\Big(\frac{\bzin}{\tau^n}\Big)\Big[F(Y;X,\boin,\tau^n,\ain,\bin,\vin)-F(Y;X,\boj,\tau^*,\aj,\bj,\vj)\Big]\\
    &+\sum_{j:|\mathcal{A}_j|>1}\sum_{i\in\mathcal{A}_j}\exp\Big(\frac{\bzin}{\tau^n}\Big)\Big[F(Y;X,\boin,\tau^n,\ain,\bin,\vin)-F(Y;X,\boj,\tau^*,\aj,\bj,\vj)\Big]\\
    :&=A_{n,1}+A_{n,2}.
\end{align*}
Let us denote
\begin{align*}
    F^{(\eta)}(Y;X,\omega^*_j):=\exp\Big(\frac{\boj X}{\tau^*}\Big)\frac{\partial^{\eta}f}{\partial h_1^{\eta}}(Y|\aj X+\bj,\vj),
\end{align*}
for any $\eta\in\mathbb{N}$, where $\omega^*_j:=(\boj,\tau^*,\aj,\bj,\vj)$. Then, by applying the first-order Taylor expansion as in equation~\eqref{eq:modified_first_order_exact}, the term $A_{n,1}$ can be decomposed as
\begin{align*}
    A_{n,1}&=\sum_{j:|\mathcal{A}_j|=1}\sum_{i\in\mathcal{A}_j}\exp\Big(\frac{\bzin}{\tau^n}\Big)\sum_{|\alpha|=1}\frac{1}{\alpha!}(\dboijn)^{\alpha_1}(\dtn)^{\alpha_2}(\daijn)^{\alpha_3}(\dbijn)^{\alpha_4}(\dvijn)^{\alpha_5}\nonumber\\
    &\hspace{5cm}\times \frac{\partial F}{\partial\beta_{1}^{\alpha_1}~\partial\tau^{\alpha_2}~\partial a^{\alpha_3}~\partial b^{\alpha_4}~\partial \nu^{\alpha_5}}(Y;X,\omega^*_j) + R_1(X,Y)\nonumber\\
    &=\sum_{j:|\mathcal{A}_j|=1}\sum_{i\in\mathcal{A}_j}\exp\Big(\frac{\bzin}{\tau^n}\Big)\sum_{|\alpha|=1}\frac{1}{\alpha!}(\dboijn)^{\alpha_1}(\dtn)^{\alpha_2}(\daijn)^{\alpha_3}(\dbijn)^{\alpha_4}(\dvijn)^{\alpha_5}\nonumber\\
    &\hspace{2cm}\times \frac{X^{\alpha_1}}{(\tau^*)^{\alpha_1}}\cdot\Big(\sum_{w=1}^{\alpha_2}\frac{c_{w,\boj,\tau^*}}{(\tau^*)^{\alpha_2}}X^{w}\Big)\cdot X^{\alpha_3}F^{(\alpha_3+\alpha_4+2\alpha_5)}(Y;X,\omega^*_j) + R_1(X,Y),
\end{align*}
where $R_1(X,Y)$ is a Taylor remainder such that $R_1(X,Y)/\mathcal{D}_{4n}\to0$ as $n\to\infty$. By letting $\ell_1=\alpha_1+w+\alpha_3$ and $\ell_2=\alpha_3+\alpha_4+2\alpha_5$, we obtain that
\begin{align}
    \label{eq:A_1_tight_over}
    A_{n,1}&=\sum_{j:|\mathcal{A}_j|=1}\sum_{\ell_1+\ell_2=1}^{2}\sum_{i\in\mathcal{A}_j}\sum_{\alpha\in\mathcal{I}_{\ell_1,\ell_2}}\exp\Big(\frac{\bzin}{\tau^n}\Big)\frac{c_{\ell_1-\alpha_1-\alpha_3,\boj,\tau^*}}{\alpha!2^{\alpha_5}(\tau^*)^{\alpha_1+\alpha_2}}(\dboijn)^{\alpha_1}(\dtn)^{\alpha_2}(\daijn)^{\alpha_3}(\dbijn)^{\alpha_4}(\dvijn)^{\alpha_5}\nonumber\\
    &\hspace{9cm}\times X^{\ell_1}F^{(\ell_2)}(Y;X,\omega^*_j)+R_1(X,Y),
\end{align}
where
\begin{align*}
    \mathcal{I}_{\ell_1,\ell_2}:=\{(\alpha_1,\alpha_2,\alpha_3,\alpha_4,\alpha_5)\in\mathbb{N}^5:\alpha_1+\alpha_2+\alpha_3\geq\ell_1, \ \alpha_3+\alpha_4+\alpha_5=\ell_2\}.
\end{align*}
Regarding $A_{n,2}$, for each $j:|\mathcal{A}_j|>1$, by means of Taylor expansion of order $\brj$, we have
\begin{align}
    \label{eq:A_2_tight_over}
    A_{n,2}&=\sum_{j:|\mathcal{A}_j|>1}\sum_{\ell_1+\ell_2=1}^{2\brj}\sum_{i\in\mathcal{A}_j}\sum_{\alpha\in\mathcal{I}_{\ell_1,\ell_2}}\exp\Big(\frac{\bzin}{\tau^n}\Big)\frac{c_{\ell_1-\alpha_1-\alpha_3,\boj,\tau^*}}{\alpha!2^{\alpha_5}(\tau^*)^{\alpha_1+\alpha_2}}(\dboijn)^{\alpha_1}(\dtn)^{\alpha_2}(\daijn)^{\alpha_3}(\dbijn)^{\alpha_4}(\dvijn)^{\alpha_5}\nonumber\\
    &\hspace{9cm}\times X^{\ell_1}F^{(\ell_2)}(Y;X,\omega^*_j)+R_2(X,Y),
\end{align}
where $R_2(X,Y)$ is a Taylor remainder such that $R_2(X,Y)/\mathcal{D}_{4n}\to0$ as $n\to\infty$. From the results in equations~\eqref{eq:A_1_tight_over}, \eqref{eq:A_2_tight_over} and the definition of $E_{n,1}$, we get
\begin{align}
    \label{eq:A_n_tight_over}
    A_{n}+E_{n,1}&=\sum_{j=1}^{k_*}\sum_{\ell_1+\ell_2=0}^{2\brj}Z^n_{\ell_1,\ell_2,j}\cdot X^{\ell_1}F^{(\ell_2)}(Y;X,\omega^*_j)+R_1(X,Y)+R_2(X,Y),
\end{align}
where
\begin{align*}
    Z^n_{\ell_1,\ell_2,j}:=\begin{cases}
        \sum_{i\in\mathcal{A}_j}\sum_{\alpha\in\mathcal{I}_{\ell_1,\ell_2}}\exp\Big(\dfrac{\bzin}{\tau^n}\Big)\dfrac{c_{\ell_1-\alpha_1-\alpha_3,\boj,\tau^*}}{\alpha!2^{\alpha_5}(\tau^*)^{\alpha_1+\alpha_2}}(\dboijn)^{\alpha_1}(\dtn)^{\alpha_2}(\daijn)^{\alpha_3}(\dbijn)^{\alpha_4}(\dvijn)^{\alpha_5}, \\
        \hspace{11.1cm}(\ell_1,\ell_2)\neq (0,0);\\
        \textbf{}\\
        \sum_{i\in\mathcal{A}_j}\exp\Big(\dfrac{\bzin}{\tau^n}\Big)-\exp\Big(\dfrac{\bzi}{\tau^*}\Big), \hspace{6.3cm} (\ell_1,\ell_2)=(0,0).
    \end{cases}
\end{align*}
Subsequently, we also separate $B_n$ into two terms:
\begin{align*}
    B_n&:=\sum_{j:|\mathcal{A}_j|=1}\sum_{i\in\mathcal{A}_j}\exp\Big(\frac{\bzin}{\tau^n}\Big)\Big[H(Y;X,\boin,\tau^n)-H(Y;X,\boj,\tau^*)\Big]\\
    &+\sum_{j:|\mathcal{A}_j|>1}\sum_{i\in\mathcal{A}_j}\exp\Big(\frac{\bzin}{\tau^n}\Big)\Big[H(Y;X,\boin,\tau^n)-H(Y;X,\boj,\tau^*)\Big]\\
    &:=B_{n,1}+B_{n,2}.
\end{align*}
By applying the first-order Taylor expansion to $B_{n,1}$, we have
\begin{align*}
    B_{n,1}&=\sum_{j:|\mathcal{A}_j|=1}\sum_{i\in\mathcal{A}_j}\exp\Big(\frac{\bzin}{\tau^n}\Big)\sum_{|\gamma|=1}\frac{1}{\gamma!}(\dboijn)^{\gamma_1}(\dtn)^{\gamma_2}\cdot\frac{\partial^{\gamma_1+\gamma_2}H}{\partial\beta_1^{\gamma_1}\partial\tau^{\gamma_2}}(Y;X,\boj,\tau^*)+R_3(X,Y)\nonumber\\
    &=\sum_{j:|\mathcal{A}_j|=1}\sum_{i\in\mathcal{A}_j}\exp\Big(\frac{\bzin}{\tau^n}\Big)\sum_{|\gamma|=1}\frac{1}{\gamma!}(\dboijn)^{\gamma_1}(\dtn)^{\gamma_2}\nonumber\\
    &\hspace{4cm}\times \frac{X^{\gamma_1}}{(\tau^*)^{\gamma_1}}\cdot\Big(\sum_{w=1}^{\gamma_2}\frac{c_{w,\boj,\tau^*}}{(\tau^*)^{\gamma_2}}X^{w}\Big)H(Y;X,\boj,\tau^*) + R_3(X,Y),
\end{align*}
where $R_3(X,Y)$ is a Taylor remainder such that $R_3(X,Y)/\mathcal{D}_{4n}\to0$ as $n\to\infty$. By letting $\ell=\gamma_1+w$, we rewrite $B_{n,1}$ as
\begin{align}
     \label{eq:B_1_tight_over}
    B_{n,1}&=\sum_{j:|\mathcal{A}_j|=1}\sum_{\ell=1}\sum_{i\in\mathcal{A}_j}\sum_{\gamma\in\mathcal{J}_{\ell}}\frac{c_{\ell-\gamma_1,\boj,\tau^*}}{\gamma!(\tau^*)^{\gamma_1+\gamma_2}}\exp\Big(\frac{\bzin}{\tau^n}\Big)(\dboijn)^{\gamma_1}(\dtn)^{\gamma_2}\cdot X^{\ell}H(Y;X,\boj,\tau^*) +R_3(X,Y).
\end{align}
Regarding $B_{n,2}$, by means of the second-order Taylor expansion, we get
\begin{align}
     \label{eq:B_2_tight_over}
    B_{n,2}=\sum_{j:|\mathcal{A}_j|>1}\sum_{\ell=1}^{2}\sum_{i\in\mathcal{A}_j}\sum_{\gamma\in\mathcal{J}_{\ell}}\frac{c_{\ell-\gamma_1,\boj,\tau^*}}{\gamma!(\tau^*)^{\gamma_1+\gamma_2}}\exp\Big(\frac{\bzin}{\tau^n}\Big)(\dboijn)^{\gamma_1}(\dtn)^{\gamma_2}\cdot X^{\ell}H(Y;X,\boj,\tau^*) +R_4(X,Y),
\end{align}
where $R_4(X,Y)$ is a Taylor remainder such that $R_4(X,Y)/\mathcal{D}_{4n}\to0$ as $n\to\infty$. From the results in equations~\eqref{eq:B_1_tight_over}, \eqref{eq:B_2_tight_over} and the definition of $E_{n,2}$, we obtain that
\begin{align}
    \label{eq:B_n_tight_over}
    B_n+E_{n,2}=\sum_{j=1}^{k_*}\sum_{\ell=0}^{1+\mathbf{1}_{|\mathcal{A}_j|>1}}Z^n_{\ell,0,j}\cdot X^{\ell}H(Y;X,\boj,\tau^*)+R_3(X,Y) +R_4(X,Y).
\end{align}
Combine equation~\eqref{eq:A_n_tight_over} with equation~\eqref{eq:B_n_tight_over}, we have
\begin{align}
    \label{eq:Q_n_tight_over}
    Q_n=\sum_{j=1}^{k_*}\sum_{\ell_1+\ell_2=0}^{2\brj}Z^n_{\ell_1,\ell_2,j}\cdot X^{\ell_1}F^{(\ell_2)}(Y;X,\omega^*_j)-\sum_{j=1}^{k_*}\sum_{\ell=0}^{1+\mathbf{1}_{|\mathcal{A}_j|>1}}Z^n_{\ell,0,j}\cdot X^{\ell}H(Y;X,\boj,\tau^*)\nonumber\\
    +R_1(X,Y)+R_2(X,Y)-R_3(X,Y)-R_4(X,Y).
\end{align}
As a consequence, we can view $[Q_n-R_1(X,Y)-R_2(X,Y)+R_3(X,Y)+R_4(X,Y)]/\mathcal{D}_{4n}$ as a combination of elements from the following set:
\begin{align*}
    \mathcal{S}:=\Big\{X^{\ell_1}F^{(\ell_2)}(Y;X,\omega^*_j), \ X^{\ell}H(Y;X,\boj,\tau^*):j\in[k_*], \ 0\leq\ell_1+\ell_2\leq 2\brj, \ 0\leq\ell\leq 1+\mathbf{1}_{|\mathcal{A}_j|>1}\Big\}.
\end{align*}
\textbf{Step 2.} In this step, we show that at least one among the ratios $Z^n_{\ell_1,\ell_2,j}/\mathcal{D}_{4n}$ does not converge to zero as $n\to\infty$. Assume by contrary that all of them go to zero. Then, by taking the summation of the absolute values of 
\begin{itemize}
    \item $Z^n_{0,0,j}/\mathcal{D}_{4n}$ for $j\in[k_*]$, we have $\frac{1}{\mathcal{D}_{4n}}\cdot\sum_{j=1}^{k_*}\Big|\sum_{i\in\mathcal{A}_j}\exp\Big(\frac{\bzin}{\tau^n}\Big)-\exp\Big(\frac{\bzj}{\tau^*}\Big)\Big|\to0$;
    \item $Z^n_{0,1,j}/\mathcal{D}_{4n}$ for $j:|\mathcal{A}_j|=1$, we have $\frac{1}{\mathcal{D}_{4n}}\cdot\sum_{j:|\mathcal{A}_j|=1}\sum_{i\in\mathcal{A}_j}\exp\Big(\frac{\bzin}{\tau^n}\Big)|\dbijn|\to0$;
    \item $Z^n_{0,2,j}/\mathcal{D}_{4n}$ for $j:|\mathcal{A}_j|=1$, we have $\frac{1}{\mathcal{D}_{4n}}\cdot\sum_{j:|\mathcal{A}_j|=1}\sum_{i\in\mathcal{A}_j}\exp\Big(\frac{\bzin}{\tau^n}\Big)|\dvijn|\to0$.
\end{itemize}
From the above limits and the formulation of $\mathcal{D}_{4n}$ in equation~\eqref{eq:new_loss_tight_over}, we deduce that
\begin{align*}
    \frac{1}{\mathcal{D}_{4n}}\cdot\sum_{j:|\mathcal{A}_j|>1}\sum_{i\in\mathcal{A}_j}\exp\Big(\frac{\bzin}{\tau^n}\Big)\Big(|\dbijn|+|\dvijn|\Big)\to1.
\end{align*}
This implies that there exists an index $j^*:|\mathcal{A}_j|>1$ (WLOG assume that $j^*=1$) such that 
\begin{align*}
    \frac{1}{\mathcal{D}_{4n}}\cdot\sum_{i\in\mathcal{A}_1}\exp\Big(\frac{\bzin}{\tau^n}\Big)\Big(|\dbione|+|\dvione|\Big)\not\to0.
\end{align*}
Moreover, since
\begin{align*}
    \frac{Z^n_{0,\ell_2,1}}{\mathcal{D}_{4n}}-\frac{Z^n_{1,\ell_2,1}}{\mathcal{D}_{4n}}=\frac{1}{\mathcal{D}_{4n}}\cdot\sum_{i\in\mathcal{A}_1}\sum_{\substack{\alpha_4+2\alpha_5=\ell_2, \\ 1\leq\alpha_4+\alpha_5\leq\brone}}\frac{1}{\alpha_4!\alpha_5!2^{\alpha_5}}(\dbione)^{\alpha_4}(\dvione)^{\alpha_5}\to0,
\end{align*}
for any $1\leq\ell_2\leq\brone$, we obtain that
\begin{align}
    \label{eq:vanish_ratio_tight}
    \frac{\sum_{i\in\mathcal{A}_1}\sum_{\substack{\alpha_4+2\alpha_5=\ell_2, \\ 1\leq\alpha_4+\alpha_5\leq\brone}}\frac{1}{\alpha_4!\alpha_5!2^{\alpha_5}}(\dbione)^{\alpha_4}(\dvione)^{\alpha_5}}{\sum_{i\in\mathcal{A}_1}\exp\Big(\frac{\bzin}{\tau^n}\Big)\Big(|\dbione|+|\dvione|\Big)}\to0.
\end{align}
Let us define $\overline{M}_n:=\max\{|\dbione|,|\dvione|^{1/2}:i\in\mathcal{A}_1\}$ and $\overline{\pi}_n:=\max_{i\in\mathcal{A}_1}\exp(\frac{\bzin}{\tau^n})$. Since the sequence $\exp\Big(\frac{\bzin}{\tau^n}\Big)/\overline{\pi}_n$ is bounded, it is possible to replace it by its subsequence that has a positive limit $q^2_{3i}:=\lim_{n\to\infty}\exp(\frac{\bzin}{\tau^n})/\overline{\pi}_n$. Thus, at least one among $q^2_{3i}$, for $i\in\mathcal{A}_1$, is equal to one. 

In addition, we also define
\begin{align*}
    (\dbione)/\overline{M}_n\to q_{4i},& \quad (\dvione)/[2\overline{M}_n]\to q_{5i}.
\end{align*}
It is worth noting that at least one among $q_{4i}$ and $q_{5i}$ for $i\in\mathcal{A}_1$ is equal to either $1$ or $-1$. Subsequently, we divide both the numerator and the denominator of the ratio in equation~\eqref{eq:vanish_ratio_tight} by $\overline{\pi}_n\overline{M}_n^{\ell_2}$, and then obtain the following system of polynomial equations:
\begin{align*}
    \sum_{i\in\mathcal{A}_1}\sum_{\substack{\alpha_4+2\alpha_5=\ell, \\ 1\leq\alpha_4+\alpha_5\leq\brone}}\frac{q^2_{3i}~q_{4i}^{\alpha_4}~q_{5i}^{\alpha_5}}{\alpha_4!~\alpha_5!}=0,
\end{align*}
for all $1\leq\ell_2\leq\brone$. However, from the definition of $\bar{r}(|\mathcal{A}_1|)$, the above system does not have any non-trivial solutions, which contradicts to the fact that at least one among $q_{4i}$ and $q_{5i}$ for $i\in\mathcal{A}_1$ is non-zero. Therefore, not all the ratios $Z^n_{\ell_1,\ell_2,j}/\mathcal{D}_{4n}$ converge to zero as $n\to\infty$.

\textbf{Step 3.} In this step, we leverage the Fatou's lemma to show a result contradicting to that in Step 2. In particular, by the Fatou's lemma, we have
\begin{align*}
    \lim_{n\to\infty}\frac{ \bbE_X[V(g_{G_n}(\cdot|X),g_{G_*}(\cdot|X))]}{\mathcal{D}_{4n}}\geq\int\liminf_{n\to\infty}\frac{|g_{G_n}(Y|X)-g_{G_*}(Y|X)|}{2\mathcal{D}_{4n}}\dint(X,Y).
\end{align*}
Moreover, recall from the hypothesis in equation~\eqref{eq:ratio_tight_over} that $\bbE_X[V(g_{G_n}(\cdot|X),g_{G_*}(\cdot|X))]/\mathcal{D}_{4n}\to0$ as $n\to\infty$. Therefore, we deduce that $$\frac{|g_{G_n}(Y|X)-g_{G_*}(Y|X)|}{\mathcal{D}_{4n}}\to0,$$ 
for almost surely $(X,Y)$. Since the term $\Big[\sum_{j=1}^{k_*}\exp\Big(\dfrac{\boj X+\bzj}{\tau^*}\Big)\Big]$ is bounded, we also have that $\dfrac{Q_n}{\mathcal{D}_{4n}}\to0$ as $n\to\infty$. Following from the results in equation~\eqref{eq:Q_n_tight_over}, we have
\begin{align}
    \label{eq:zero_tight_limit}
    \sum_{j=1}^{k_*}\sum_{\ell_1+\ell_2=0}^{2\brj}\frac{Z^n_{\ell_1,\ell_2,j}}{\mathcal{D}_{4n}}\cdot X^{\ell_1}F^{(\ell_2)}(Y;X,\omega^*_j)-\sum_{j=1}^{k_*}\sum_{\ell=0}^{1+\mathbf{1}_{|\mathcal{A}_j|>1}}\frac{Z^n_{\ell,0,j}}{\mathcal{D}_{4n}}\cdot X^{\ell}H(Y;X,\boj,\tau^*)\to0.
\end{align}
Therefore, $Z^n_{\ell_1,\ell_2,j}/\mathcal{D}_{4n}$ must be bounded for any $j\in[k_*]$ and $0\leq\ell_1+\ell_2\leq 2\brj$. Indeed, if at least one among them is not bounded, then the ratio $Z^n_{\ell_1,\ell_2,j}/\mathcal{D}_{4n}$ will go to infinity, implying that the left hand side of equation~\eqref{eq:zero_tight_limit} does not go to zero, which is a contradiction. Thus, for each $j\in[k_*]$ and $0\leq\ell_1+\ell_2\leq 2\brj$, we can replace $Z^n_{\ell_1,\ell_2,j}$ by one of its subsequences such that the ratio $Z^n_{\ell_1,\ell_2,j}/\mathcal{D}_{4n}$ has a finite limit as $n\to\infty$. Let us denote $Z^n_{\ell_1,\ell_2,j}/\mathcal{D}_{4n}\to Z^*_{\ell_1,\ell_2,j}$, then it follows from the results in Step 2 that at least one among them is non-zero. Additionally, equation~\eqref{eq:zero_tight_limit} indicates that
\begin{align*}
    \sum_{j=1}^{k_*}\sum_{\ell_1+\ell_2=0}^{2\brj}Z^*_{\ell_1,\ell_2,j}\cdot X^{\ell_1}F^{(\ell_2)}(Y;X,\omega^*_j)-\sum_{j=1}^{k_*}\sum_{\ell=0}^{1+\mathbf{1}_{|\mathcal{A}_j|>1}}Z^*_{\ell,0,j}\cdot X^{\ell}H(Y;X,\boj,\tau^*)=0,
\end{align*}
for almost surely $(X,Y)$. Since the set 
\begin{align*}
     \mathcal{S}=\Big\{X^{\ell_1}F^{(\ell_2)}(Y;X,\omega^*_j), \ X^{\ell}H(Y;X,\boj,\tau^*):j\in[k_*], \ 0\leq\ell_1+\ell_2\leq 2\brj, \ 0\leq\ell\leq 1+\mathbf{1}_{|\mathcal{A}_j|>1}\Big\}
\end{align*}
is linearly independent, we deduce that $Z^*_{\ell_1,\ell_2,j}=0$ for any $j\in[k_*]$ and $0\leq\ell_1+\ell_2\leq 2\brj$, which contradicts the fact that at least one among them is different from zero. Hence, the proof is completed.

\subsection{Proof of Theorem~\ref{theorem:modified_exact}}
\label{appendix:modified_exact}
In this proof, our main goal is to demonstrate the following inequality:
\begin{align}
    \label{eq:universal_inequality}
    \inf_{G\in\mathcal{E}_{k_*}(\Theta)}\bbE_X[V(p_{G}(\cdot|X),p_{G_*}(\cdot|X))]/\mathcal{D}_5(G,G_*)>0.
\end{align}
For that purpose, we separate the above inequality into local and global parts.

\textbf{Local part:} In this part, we aim to show that
\begin{align}
    \label{eq:local_modified_exact}
    \lim_{\varepsilon\to0}\inf_{G\in\mathcal{E}_{k_*}(\Theta):\mathcal{D}_5(G,G_*)\leq\varepsilon}\bbE_X[V(p_{G}(\cdot|X),p_{G_*}(\cdot|X))]/\mathcal{D}_5(G,G_*)>0.
\end{align}
Assume by contrary that the above claim does not hold true, then we can find a sequence of mixing measures $G_n=\sum_{i=1}^{k_*}\exp(\beta^n_{0i}/\tau^n)\delta_{(\beta^n_{1i},\tau^n,a^n_i,b^n_i,\nu^n_i)}$ in $\mathcal{E}_{k_*}(\Theta)$ that satisfies $\mathcal{D}_{5n}:=\mathcal{D}_5(G_n,G_*)\to0$ and
\begin{align}
    \label{eq:ratio_modified_exact}
    \bbE_X[V(p_{G_n}(\cdot|X),p_{G_*}(\cdot|X))]/\mathcal{D}_{5n}\to0,
\end{align}
as $n\to\infty$. Recall that under the exact-specified settings, each Voronoi cell $\mathcal{A}^n_i=\mathcal{A}_i(G_n)$ has only one element. Therefore, we may assume without loss of generality (WLOG) that $\mathcal{A}^n_i=\{i\}$ for any $i\in[k_*]$. Thus, the loss function $\mathcal{D}_{5n}$ is reduced to
\begin{align}
    \label{eq:new_loss_modified_exact}
    \mathcal{D}_{5n}:=
    \sum_{i=1}^{k_*}\exp\Big(\frac{\bzin}{\tau^n}\Big)\Big[\|\dboin\|+|\dtn|+\|\dain\|+|\dbin|+|\dvin|\Big]+\sum_{i=1}^{k_*}\Big|\exp\Big(\frac{\bzin}{\tau^n}\Big)-\exp\Big(\frac{\bzj}{\tau^*}\Big)\Big|
\end{align}
Since $\mathcal{D}_{5n}\to0$, we get that $(\boin,\tau^n,\ain,\bin,\vin)\to(\boi,\tau^*,\ai,\bi,\vi)$ and $\exp(\bzin/\tau^n)\to\exp(\bzi/\tau^*)$ as $n\to\infty$. Now, we divide the proof of local part into three steps as follows:

\textbf{Step 1.} In this step, we decompose the quantity $Q_n:=\Big[\sum_{i=1}^{k_*}\exp\Big(\dfrac{\sigma((\boi)^{\top}X)+\bzi}{\tau^*}\Big)\Big]\cdot[p_{G_n}(Y|X)-p_{G_*}(Y|X)]$ into a linear combination of linearly independent terms. Firstly, let $\bsigma(X,\beta_1):=\sigma(\beta_1^{\top}X)$, $F(Y;X,\beta_{1},\tau,a,b,\nu):=\exp\Big(\frac{\bsigma(X,\beta_1)}{\tau}\Big)f(Y|a^{\top}X+b,\nu)$ and $H(Y;X,\beta_{1},\tau):=\exp\Big(\frac{\bsigma(X,\beta_1)}{\tau}\Big)p_{G_n}(Y|X)$. Then, it can be verified that
\begin{align*}
    Q_n&=\sum_{i=1}^{k_*}\exp\Big(\frac{\bzin}{\tau^n}\Big)\Big[F(Y;X,\boin,\tau^n,\ain,\bin,\vin)-F(Y;X,\boi,\tau^*,\ai,\bi,\vi)\Big]\\
    &-\sum_{i=1}^{k_*}\exp\Big(\frac{\bzin}{\tau^n}\Big)\Big[H(Y;X,\boin,\tau^n)-H(Y;X,\boi,\tau^*)\Big]\\
    &+\sum_{i=1}^{k_*}\Big[\exp\Big(\frac{\bzin}{\tau^n}\Big)-\exp\Big(\frac{\bzi}{\tau^*}\Big)\Big]\cdot\exp\Big(\frac{\bsigma(X,\boi)}{\tau^*}\Big)f(Y|(\ai)^{\top}X+\bi,\vi)\\
    &-\sum_{i=1}^{k_*}\Big[\exp\Big(\frac{\bzin}{\tau^n}\Big)-\exp\Big(\frac{\bzi}{\tau^*}\Big)\Big]\cdot\exp\Big(\frac{\bsigma(X,\boi)}{\tau^*}\Big)p_{G_n}(Y|X)\\
    :&=A_n-B_n+E_{n,1}-E_{n,2}.
\end{align*}
Next, by means of the first-order Taylor expansion, we get that
\begin{align*}
    A_n&=\sum_{i=1}^{k_*}\exp\Big(\frac{\bzin}{\tau^n}\Big)\sum_{|\alpha|=1}\frac{1}{\alpha!}(\dboin)^{\alpha_1}(\dtn)^{\alpha_2}(\dain)^{\alpha_3}(\dbin)^{\alpha_4}(\dvin)^{\alpha_5}\\
    &\hspace{5cm}\times \frac{\partial F}{\partial\beta_{1}^{\alpha_1}~\partial\tau^{\alpha_2}~\partial a^{\alpha_3}~\partial b^{\alpha_4}~\partial \nu^{\alpha_5}}(Y;X,\boi,\tau^*,\ai,\bi,\vi) + R_1(X,Y),
\end{align*}
where $R_1(X,Y)$ is a Taylor remainder such that $R_1(X,Y)/\mathcal{D}_{5n}\to0$ as $n\to\infty$. Let us denote
\begin{align*}
    F^{(\eta)}(Y;X,\omega^*_i):=\exp\Big(\frac{\bsigma(X,\boi)}{\tau^*}\Big)\frac{\partial^{\eta}f}{\partial h_1^{\eta}}(Y|(\ai)^{\top}X+\bi,\vi),
\end{align*}
for any $\eta\in\mathbb{N}$, where $\omega^*_i:=(\boi,\tau^*,\ai,\bi,\vi)$. Then, the first derivatives of function $F$ w.r.t its parameters are given by
\begin{align}
    \frac{\partial F}{\partial\beta_{1}^{(u)}}(Y;X,\omega^*_i)&=\frac{1}{\tau^*}\cdot\frac{\partial\bsigma}{\partial \beta_1^{(u)}}(X,\boi)F(Y;X,\omega^*_i),\nonumber\\
    \frac{\partial F}{\partial\tau}(Y;X,\omega^*_i)&=-\frac{\bsigma(X,\boi)}{(\tau^*)^2}F(Y;X,\omega^*_i),\nonumber\\
    \frac{\partial F}{\partial a^{(u)}}(Y;X,\omega^*_i)&=X^{(u)}F^{(1)}(Y;X,\omega^*_i),\nonumber\\
    \frac{\partial F}{\partial b}(Y;X,\omega^*_i)&=F^{(1)}(Y;X,\omega^*_i),\nonumber\\
    \label{eq:F_first_derivatives}
    \frac{\partial F}{\partial \nu}(Y;X,\omega^*_i)&=\frac{1}{2}F^{(2)}(Y;X,\omega^*_i),
\end{align}
for any $u\in[d]$ and $i\in[k_*]$. Then, the terms $A_n$ and $E_{n,1}$ can be represented as
\begin{align}
    \label{eq:modified_first_order_exact}
    A_n&=\sum_{i=1}^{k_*}\sum_{\eta=0}^{2}C_{n,\eta,i}(X)F^{(\eta)}(Y;X,\omega^*_i)+R_1(X,Y),
\end{align}
where we define for any $i\in[k_*]$ and $X\in\mathcal{X}$ that
\begin{align}
    \label{eq:modified_coefficients_exact}
    C_{n,0,i}(X)&:=\exp\Big(\frac{\bzin}{\tau^n}\Big)\Big[\sum_{u=1}^{d}\frac{(\dboin)^{(u)}}{\tau^*}\cdot \frac{\partial\bsigma}{\partial \beta_1^{(u)}}(X,\boi)-\frac{(\dtn)}{(\tau^*)^2}\cdot\bsigma(X,\boi)\Big],\nonumber\\
    C_{n,1,i}(X)&:=\exp\Big(\frac{\bzin}{\tau^n}\Big)\Big[\sum_{u=1}^{d}(\dain)^{(u)} X^{(u)}+(\dbin)\Big],\nonumber\\
    C_{n,2,i}(X)&:=\exp\Big(\frac{\bzin}{\tau^n}\Big)\cdot\frac{(\dvin)}{2}.
\end{align}
Thus, we can view the terms $[A_n-R_1(X,Y)]/\mathcal{D}_{5n}$ and $E_{n,1}/\mathcal{D}_{5n}$ as a linear combination of elements from the set $\mathcal{F}:=\cup_{i=1}^{k_*}\cup_{\eta=0}^{2}\mathcal{F}_{i,\eta}$ in which
\begin{align*}
    \mathcal{F}_{i,0}:=&~\Big\{\frac{\partial \bsigma}{\partial \beta_1^{(u)}}(X,\boi)F(Y;X,\omega^*_i):u\in[d]\Big\}
    \cup\Big\{\bsigma(X,\boi)F(Y;X,\omega^*_i)\Big\}
    \cup\Big\{F(Y;X,\omega^*_i)\Big\},\\
    \mathcal{F}_{i,1}:=&~\Big\{X^{(u)}\cdot F^{(1)}(Y;X,\omega^*_i):u\in[d]\Big\}
    \cup\Big\{F^{(1)}(Y;X,\omega^*_i)\Big\},\\
    \mathcal{F}_{i,2}:=&~\Big\{F^{(2)}(Y;X,\omega^*_i)\Big\},
\end{align*}
where $\omega^*_i:=(\boi,\tau^*,\ai,\bi,\vi)$, for any $i\in[k_*]$. Subsequently, we apply the first-order Taylor expansion to $B_n$ as follows:
\begin{align*}
    B_n:&=\sum_{i=1}^{k_*}\exp\Big(\frac{\bzin}{\tau^n}\Big)\sum_{|\gamma|=1}\frac{1}{\gamma!}(\dboin)^{\gamma_1}(\dtn)^{\gamma_2}\cdot\frac{\partial H}{\partial \beta_1^{\gamma_1}~\partial\tau^{\gamma_2}}(Y;X,\boi,\tau^*) + R_2(X,Y),
\end{align*}
where $R_2(X,Y)$ is a Taylor remainder such that $R_2(X,Y)/\mathcal{D}_{5n}\to0$ as $n\to\infty$. 
Then, the term $B_n$ can be represented as
\begin{align}
    \label{eq:modified_first_order_2}
    B_n&=\sum_{i=1}^{k_*}C_{n,0,i}(X)H(Y;X,\boi,\tau^*)+R_2(X,Y),
\end{align}
where $C_{n,0,i}(X)$ is defined in equation~\eqref{eq:modified_coefficients_exact}.  
Therefore, the terms $[B_n-R_2(X,Y)]/\mathcal{D}_{5n}$ and $E_{n,2}/\mathcal{D}_{5n}$ can be treated as a linear combination of elements from the set $\mathcal{H}:=\cup_{i=1}^{k_*}\mathcal{H}_i$, where we define for $i\in[k_*]$ that
\begin{align*}
    \mathcal{H}_i:=&~\Big\{\frac{\partial\bsigma}{\partial \beta_1}(X,\boi)H(Y;X,\boi,\tau^*):u\in[d]\Big\}
    \cup\Big\{\bsigma(X,\boi)H(Y;X,\boi,\tau^*)\Big\}
    \cup\Big\{H(Y;X,\boi,\tau^*)\Big\}.
\end{align*}
\textbf{Step 2.} In this step, we prove by contradiction that at least one among the coefficients in the representations of $[A_n-R_1(X,Y)]/\mathcal{D}_{5n}$, $[B_n-R_2(X,Y)]/\mathcal{D}_{5n}$, $E_{n,1}/\mathcal{D}_{5n}$ and $E_{n,2}/\mathcal{D}_{5n}$ does not converge to zero when $n\to\infty$. Assume by contrary that all of them go to 0 as $n\to\infty$. In the term $E_{n,1}/\mathcal{D}_{5n}$, by taking the summation of the absolute values of the coefficients of $F(Y;X,\omega^*_i)$, we get that
\begin{align}
    \label{eq:limit_bias_exact}
    \frac{1}{\mathcal{D}_{5n}}\cdot\sum_{i=1}^{k_*}\Big|\exp\Big(\frac{\bzin}{\tau^n}\Big)-\exp\Big(\frac{\bzi}{\tau^*}\Big)\Big|\to0.
\end{align}
Next, by taking the summation of the absolute values of the coefficients associated with 
\begin{itemize}
    \item $\dfrac{\partial \bsigma}{\partial \beta_1^{(u)}}(X,\boi)F(Y;X,\omega^*_i)$ in $\mathcal{F}_{0}$: we have that  $\frac{1}{\mathcal{D}_{5n}}\cdot\sum_{i=1}^{k_*}\exp\Big(\frac{\bzin}{\tau^n}\Big)\|\dboin\|_1\to0$; 
    \item $\bsigma(X,\boi)F(Y;X,\omega^*_i)$ in $\mathcal{F}_{0}$: we have that $\frac{1}{\mathcal{D}_{5n}}\cdot\sum_{i=1}^{k_*}\exp\Big(\frac{\bzin}{\tau^n}\Big)|\dtn|\to0$;
    \item $X^{(u)}F^{(1)}(Y;X,\omega^*_i)$ in $\mathcal{F}_{1}$: we have that $\frac{1}{\mathcal{D}_{5n}}\cdot\sum_{i=1}^{k_*}\exp\Big(\frac{\bzin}{\tau^n}\Big)\|\dain\|_1\to0$;
    \item $F^{(1)}(Y;X,\omega^*_i)$ in $\mathcal{F}_{1}$: we have that $\frac{1}{\mathcal{D}_{5n}}\cdot\sum_{i=1}^{k_*}\exp\Big(\frac{\bzin}{\tau^n}\Big)|\dbin|\to0$;
    \item $F^{(2)}(Y;X,\omega^*_i)$ in $\mathcal{F}_{2}$: we have that $\frac{1}{\mathcal{D}_{5n}}\cdot\sum_{i=1}^{k_*}\exp\Big(\frac{\bzin}{\tau^n}\Big)|\dvin|\to0$;
\end{itemize}
Due to the topological equivalence between $\ell_1$-norm and $\ell_2$-norm, it follows that
\begin{align}
    \label{eq:limit_parameter_exact}
        \frac{1}{\mathcal{D}_{5n}}\cdot\sum_{i=1}^{k_*}\exp\Big(\frac{\bzin}{\tau^n}\Big)\Big(\|\dboin\|+|\dtn|+\|\dain\|+|\dbin|+|\dvin|\Big)\to0.
\end{align}
Putting the results in equations~\eqref{eq:limit_bias_exact} and \eqref{eq:limit_parameter_exact} and the formulation of the loss $\mathcal{D}_{5n}$ in equation~\eqref{eq:new_loss_modified_exact} together, we deduce that $1=\mathcal{D}_{5n}/\mathcal{D}_{5n}\to0$ as $n\to\infty$, which is a contradiction. Consequently, not all the coefficients in the representations of $[A_n-R_1(X,Y)]/\mathcal{D}_{5n}$, $[B_n-R_2(X,Y)]/\mathcal{D}_{5n}$, $E_{n,1}/\mathcal{D}_{5n}$ and $E_{n,2}/\mathcal{D}_{5n}$ converge to zero when $n\to\infty$.

\textbf{Step 3.} In this step, we utilize the Fatou's lemma to demonstrate a result contradicting to that in Step 2. In particular, let us denote $m_n$ as the maximum of the absolute values of the coefficients in the representations of $[A_n-R_1(X,Y)]/\mathcal{D}_{5n}$, $[B_n-R_2(X,Y)]/\mathcal{D}_{5n}$, $E_{n,1}/\mathcal{D}_{5n}$ and $E_{n,2}/\mathcal{D}_{5n}$. From the conclusion of Step 2, we know that $1/m_n\not\to\infty$. Next, we denote
\begin{align*}
    \frac{1}{m_n\mathcal{D}_{5n}}\cdot\Big[\exp\Big(\frac{\bzin}{\tau^n}\Big)-\exp\Big(\frac{\bzi}{\tau^*}\Big)\Big]&\to\phi_{0,i},\qquad 
    \frac{1}{m_n\mathcal{D}_{5n}}\cdot\exp\Big(\frac{\bzin}{\tau^n}\Big)(\dboin)^{(u)}\to\phi^{(u)}_{1,i},\\
    \frac{1}{m_n\mathcal{D}_{5n}}\cdot\exp\Big(\frac{\bzin}{\tau^n}\Big)(\dtn)&\to\phi_{2,i},\qquad
    \frac{1}{m_n\mathcal{D}_{5n}}\cdot\exp\Big(\frac{\bzin}{\tau^n}\Big)(\dain)^{(u)}\to\phi^{(u)}_{3,i},\\
    \frac{1}{m_n\mathcal{D}_{5n}}\cdot\exp\Big(\frac{\bzin}{\tau^n}\Big)(\dbin)&\to\phi_{4,i},\qquad
    \frac{1}{m_n\mathcal{D}_{5n}}\cdot\exp\Big(\frac{\bzin}{\tau^n}\Big)(\dvin)\to\phi_{5,i},
\end{align*}
as $n\to\infty$ for any $u\in[d]$ and $i\in[k_*]$. Note that at least one among the terms $\phi_{0,i},\phi^{(u)}_{1,i},\phi_{2,i},\phi^{(u)}_{3,i},\phi_{4,i}$ and $\phi_{5,i}$ is different from zero. By means of the Fatou's lemma, we have that
\begin{align*}
    \lim_{n\to\infty}\frac{ \bbE_X[V(p_{G_n}(\cdot|X),p_{G_*}(\cdot|X))]}{m_n\mathcal{D}_{5n}}\geq\int\liminf_{n\to\infty}\frac{|p_{G_n}(Y|X)-p_{G_*}(Y|X)|}{2m_n\mathcal{D}_{5n}}\dint(X,Y).
\end{align*}
Recall from equation~\eqref{eq:ratio_modified_exact} that the limit the left hand side is equal to zero, which implies that $\frac{|p_{G_n}(Y|X)-p_{G_*}(Y|X)|}{m_n\mathcal{D}_{5n}}\to0$. 
as $n\to\infty$ for almost surely $(X,Y)$. Thus, we also have that $\frac{Q_n}{m_n\mathcal{D}_{5n}}\to0$ as $n\to\infty$. On the other hand, we have
\begin{align*}
    \frac{Q_n}{m_n\mathcal{D}_{5n}}\to\sum_{i=1}^{k_*}\sum_{\eta=0}^{2}C^*_{\eta,i}(X)\cdot F^{(\eta)}(Y;X,\omega^*_i)-\sum_{i=1}^{k_*}C^*_{0,i}(X)\cdot H(Y;X,\boi,\tau^*),
\end{align*}
for almost surely $(X,Y)$, where we define
\begin{align}
    \label{eq:first_limit_exact}
    C^*_{0,i}(X)&:=\phi_{0,i}+\sum_{u=1}^{d}\frac{\phi^{(u)}_{1,i}}{\tau^*}\cdot \dfrac{\partial \bsigma}{\partial \beta_1^{(u)}}(X,\boi)-\phi_{2,i}\cdot\frac{\bsigma(X,\boi)}{(\tau^*)^2},\\
    \label{eq:second_limit_exact}
    C^*_{1,i}(X)&:=\sum_{u=1}^{d}\phi_{3,i}^{(u)}\cdot X^{(u)}+\phi_{4,i},\\
    \label{eq:third_limit_exact}
    C^*_{2,i}(X)&:=\frac{1}{2}\phi_{5,i},
\end{align}
for any $i\in[k_*]$. As a result, we achieve that
\begin{align*}
    \sum_{i=1}^{k_*}\sum_{\eta=0}^{2}C^*_{\eta,i}(X)\cdot F^{(\eta)}(Y;X,\omega^*_i)-\sum_{i=1}^{k_*}C^*_{0,i}(X)\cdot H(Y;X,\boi,\tau^*)=0,
\end{align*}
for almost surely $(X,Y)$. Since the following set is linearly independent w.r.t $Y$:
\begin{align*}
    \left\{F^{(\eta)}(Y;X,\omega^*_i), \ H(Y;X,\boi,\tau^*):0\leq\eta\leq 2, \  i\in[k_*]\right\},
\end{align*}
it leads to $C^*_{\eta,i}(X)=0$ for any $0\leq\eta\leq 2$ and $i\in[k_*]$ for almost surely $X$. 
\begin{itemize}
    \item When $C^*_{0,i}(X)=0$ for almost surely $X$: as the function $\sigma$ satisfies the conditions in Definition~\ref{def:modified_function_exact}, i.e. the set $$\Big\{\frac{\partial\bsigma}{\partial \beta_1^{(u)}}(X,\boi), \ \bsigma(X,\boi), \ 1:1\leq u\leq d\Big\}$$ is linearly independent w.r.t $X$, it follows from equation~\eqref{eq:first_limit_exact} that $\phi_{0,i}=\phi^{(u)}_{1,i}=\phi_{2,i}=0$ for any $u\in[d]$ and $i\in[k_*]$. 
    \item When $C^*_{1,i}(X)=0$ for almost surely $X$: since the set $\{X^{(u)},1:u\in[d]\}$ is linearly independent w.r.t $X$, equation~\eqref{eq:second_limit_exact} indicates that $\phi^{(u)}_{3,i}=\phi_{4,i}=0$ for any $u\in[d]$ and $i\in[k_*]$. 
    \item When $C^*_{2,i}(X)=0$ for almost surely $X$: it can be seen from equation~\eqref{eq:third_limit_exact} that $\phi_{5,i}=0$ for any $i\in[k_*]$.
\end{itemize}
However, the above results contradict the fact that at least one among the terms $\phi_{0,i},\phi^{(u)}_{1,i},\phi_{2,i},\phi^{(u)}_{3,i},\phi_{4,i}$ and $\phi_{5,i}$ is non-zero. Hence, we reach the conclusion of the local part in equation~\eqref{eq:local_modified_exact}, which means that there exists a constant $\varepsilon'>0$ such that
\begin{align*}
    \inf_{G\in\mathcal{E}_{k_*}(\Theta):\mathcal{D}_5(G,G_*)\leq\varepsilon'}\bbE_X[V(p_{G}(\cdot|X),p_{G_*}(\cdot|X))]/\mathcal{D}_5(G,G_*)>0.
\end{align*}
\textbf{Global part.} As a consequence, it suffices to demonstrate the following inequality:
\begin{align}
    \label{global_modified_exact}
    \inf_{G\in\mathcal{E}_{k_*}(\Theta):\mathcal{D}_5(G,G_*)>\varepsilon'}\bbE_X[V(p_{G}(\cdot|X),p_{G_*}(\cdot|X))]/\mathcal{D}_5(G,G_*)>0.
\end{align}
Assume by contrary that the above claim does not hold true, then we can seek a sequence of mixing measures $G'_n\in\mathcal{E}_{k_*}(\Omega)$ such that $\mathcal{D}_5(G'_n,G_*)>\varepsilon'$ and
\begin{align*}
    \lim_{n\to\infty}\frac{\bbE_X[V(p_{G'_n}(\cdot|X),p_{G_*}(\cdot|X))]}{\mathcal{D}_5(G'_n,G_*)}=0,
\end{align*}
which directly implies that $\bbE_X[V(p_{G'_n}(\cdot|X),p_{G_*}(\cdot|X))]\to0$ as $n\to\infty$. Recall that $\Theta$ is a compact set, therefore, we can replace the sequence $G'_n$ by one of its subsequences that converges to a mixing measure $G'\in\mathcal{E}_{k_*}(\Omega)$. Since $\mathcal{D}_5(G'_n,G_*)>\varepsilon'$, this result induces that $\mathcal{D}_5(G',G_*)>\varepsilon'$. 

Next, by invoking the Fatou's lemma, it follows that
\begin{align*}
    0=\lim_{n\to\infty}\bbE_X[2V(p_{G'_n}(\cdot|X),p_{G_*}(\cdot|X))]\geq \int\liminf_{n\to\infty}\Big|p_{G'_n}(Y|X)-p_{G_*}(Y|X)\Big|~\dint(X,Y).
\end{align*}
Thus, we get that $p_{G'}(Y|X)=p_{G_*}(Y|X)$ for almost surely $(X,Y)$. From Proposition~\ref{prop:modified_identifiability}, we know that the model~\eqref{eq:modified_density} is identifiable, which indicates that $G'\equiv G_*$. As a consequence, we have that $\mathcal{D}_5(G',G_*)=0$, contradicting the fact that $\mathcal{D}_5(G',G_*)>\varepsilon'>0$. 

Hence, the proof is completed.


\subsection{Proof of Theorem~\ref{theorem:modified_over}}
\label{appendix:modified_over}

Similar to the proof of Theorem~\ref{theorem:modified_over} in Appendix~\ref{appendix:modified_exact}, we aim to prove the following inequality:
\begin{align}
    \label{eq:modified_inequality_over}
    \inf_{G\in\mathcal{G}_{k}(\Theta)}\bbE_X[V(p_{G}(\cdot|X),p_{G_*}(\cdot|X))]/\mathcal{D}_6(G,G_*)>0.
\end{align}
Moreover, we also divide the above inequality into local and global parts. Since the global part can be argued in the same fashion as in Appendix~\ref{appendix:modified_exact}, we will demonstrate only the local part, that is
\begin{align}
    \label{eq:local_modified_over}
    \lim_{\varepsilon\to0}\inf_{G\in\mathcal{G}_{k}(\Theta):\mathcal{D}_6(G,G_*)\leq\varepsilon}\bbE_X[V(p_{G}(\cdot|X),p_{G_*}(\cdot|X))]/\mathcal{D}_6(G,G_*)>0.
\end{align}
Assume by contrary that the above claim does not hold true, then we can find a sequence of mixing measures $G_n=\sum_{i=1}^{k_*}\exp(\beta^n_{0i}/\tau^n)\delta_{(\beta^n_{1i},\tau^n,a^n_i,b^n_i,\nu^n_i)}$ in $\mathcal{G}_{k}(\Theta)$ that satisfies $\mathcal{D}_{6n}:=\mathcal{D}_6(G_n,G_*)\to0$ and
\begin{align}
    \label{eq:ratio_modified_over}
    \bbE_X[V(p_{G_n}(\cdot|X),p_{G_*}(\cdot|X))]/\mathcal{D}_{6n}\to0,
\end{align}
as $n\to\infty$. Let us denote $\mathcal{A}^n_j=\mathcal{A}_j(G_n)$, then the loss function $\mathcal{D}_{6n}$ is reduced to
\begin{align}
    \label{eq:new_loss_modified_over}
    &\mathcal{D}_{6n}:=\sum_{j:|\mathcal{A}_j|>1}\sum_{i\in\mathcal{A}_j}\exp\Big(\frac{\bzin}{\tau^n}\Big)\Big[\|\dboijn\|^2+|\dtn|^2+\|\daijn\|^2+|\dbijn|^{\brj}+|\dvijn|^{\brj/2}\Big]\nonumber\\
    &+\sum_{j:|\mathcal{A}_j|=1}\sum_{i\in\mathcal{A}_j}\exp\Big(\frac{\bzin}{\tau^n}\Big)\Big[\|\dboijn\|+|\dtn|+\|\daijn\|+|\dbijn|+|\dvijn|\Big]+\sum_{j=1}^{k_*}\Big|\sum_{i\in\mathcal{A}_j}\exp\Big(\frac{\bzin}{\tau^n}\Big)-\exp\Big(\frac{\bzj}{\tau^*}\Big)\Big|
\end{align}
As $\mathcal{D}_{6n}\to0$, we get that $(\boin,\tau^n,\ain,\bin,\vin)\to(\boj,\tau^*,\aj,\bj,\vj)$ and $\sum_{i\in\mathcal{A}_j}\exp(\bzin/\tau^n)\to\exp(\bzj/\tau^*)$ as $n\to\infty$ for any $i\in\mathcal{A}_j$ and $j\in[k_*]$. Now, we divide the proof of local part into three steps as follows:

\textbf{Step 1.} In this step, we decompose the quantity $Q_n:=\Big[\sum_{j=1}^{k_*}\exp\Big(\dfrac{\sigma((\boj)^{\top}X)+\bzj}{\tau}\Big)\Big]\cdot[p_{G_n}(Y|X)-p_{G_*}(Y|X)]$ into a linear combination of linearly independent terms. Firstly, let $\bsigma(X,w):=\sigma(w^{\top}X)$, $F(Y;X,\beta_{1},\tau,a,b,\nu):=\exp\Big(\frac{\bsigma(X,\beta_1)}{\tau}\Big)f(Y|a^{\top}X+b,\nu)$ and $H(Y;X,\beta_{1},\tau):=\exp\Big(\frac{\bsigma(X,\beta_1)}{\tau}\Big)p_{G_n}(Y|X)$. Then, it can be verified that
\begin{align}
    \label{eq:Q_n_formulation}
    Q_n&=\sum_{j=1}^{k_*}\sum_{i\in\mathcal{A}_j}\exp\Big(\frac{\bzin}{\tau^n}\Big)\Big[F(Y;X,\boin,\tau^n,\ain,\bin,\vin)-F(Y;X,\boj,\tau^*,\aj,\bj,\vj)\Big]\nonumber\\
    &-\sum_{j=1}^{k_*}\sum_{i\in\mathcal{A}_j}\exp\Big(\frac{\bzin}{\tau^n}\Big)\Big[H(Y;X,\boin,\tau^n)-H(Y;X,\boj,\tau^*)\Big]\nonumber\\
    &+\sum_{j=1}^{k_*}\Big[\sum_{i\in\mathcal{A}_j}\exp\Big(\frac{\bzin}{\tau^n}\Big)-\exp\Big(\frac{\bzi}{\tau^*}\Big)\Big]\cdot\exp(\bsigma(X,\boj))f(Y|(\aj)^{\top}X+\bj,\vj)\nonumber\\
    &-\sum_{j=1}^{k_*}\Big[\sum_{i\in\mathcal{A}_j}\exp\Big(\frac{\bzin}{\tau^n}\Big)-\exp\Big(\frac{\bzi}{\tau^*}\Big)\Big]\cdot\exp(\bsigma(X,\boj))p_{G_n}(Y|X)\nonumber\\
    :&=A_n-B_n+E_{n,1}-E_{n,2}.
\end{align}
Next, we continue to separate $A_n$ into two terms as follows:
\begin{align*}
    A_n:&=\sum_{j:|\mathcal{A}_j|=1}\sum_{i\in\mathcal{A}_j}\exp\Big(\frac{\bzin}{\tau^n}\Big)\Big[F(Y;X,\boin,\tau^n,\ain,\bin,\vin)-F(Y;X,\boj,\tau^*,\aj,\bj,\vj)\Big]\\
    &+\sum_{j:|\mathcal{A}_j|>1}\sum_{i\in\mathcal{A}_j}\exp\Big(\frac{\bzin}{\tau^n}\Big)\Big[F(Y;X,\boin,\tau^n,\ain,\bin,\vin)-F(Y;X,\boj,\tau^*,\aj,\bj,\vj)\Big]\\
    :&=A_{n,1}+A_{n,2}.
\end{align*}
Let us denote
\begin{align*}
    F^{(\eta)}(Y;X,\boj,\tau^*,\aj,\bj,\vj):=\exp\Big(\frac{\bsigma(X,\boj)}{\tau^*}\Big)\frac{\partial^{\eta}f}{\partial h_1^{\eta}}(Y|(\aj)^{\top}X+\bj,\vj),
\end{align*}
for any $\eta\in\mathbb{N}$. Then, by applying the first-order Taylor expansion as in equation~\eqref{eq:modified_first_order_exact}, the term $A_{n,1}$ can be decomposed as
\begin{align}
    \label{eq:A_1_over}
    A_{n,1}&=\sum_{j:|\mathcal{A}_j|=1}\sum_{i\in\mathcal{A}_j}\exp\Big(\frac{\bzin}{\tau^n}\Big)\sum_{|\alpha|=1}\frac{1}{\alpha!}(\dboijn)^{\alpha_1}(\dtn)^{\alpha_2}(\daijn)^{\alpha_3}(\dbijn)^{\alpha_4}(\dvijn)^{\alpha_5}\nonumber\\
    &\hspace{5cm}\times \frac{\partial F}{\partial\beta_{1}^{\alpha_1}~\partial\tau^{\alpha_2}~\partial a^{\alpha_3}~\partial b^{\alpha_4}~\partial \nu^{\alpha_5}}(Y;X,\boj,\tau^*,\aj,\bj,\vj) + R_1(X,Y)\nonumber\\
    &=\sum_{j:|\mathcal{A}_j|=1}\sum_{\eta=0}^{2}C_{n,\eta,j}(X)F^{(\eta)}(Y;X,\boj,\tau^*,\aj,\bj,\vj)+R_1(X,Y),\\
\end{align}
where $R_1(X,Y)$ is a Taylor remainder such that $R_1(X,Y)/\mathcal{D}_{6n}\to0$ as $n\to\infty$ and 
\begin{align}
    C_{n,0,j}(X)&:=\sum_{i\in\mathcal{A}_j}\exp\Big(\frac{\bzin}{\tau^n}\Big)\Big[\frac{\dboijn}{\tau^*}\cdot \frac{\partial\bsigma}{\partial \beta_1^{(u)}}(X,\boj)-\frac{\dtn}{(\tau^*)^2}\bsigma(X,\boj)\Big],\nonumber\\
    C_{n,1,j}(X)&:=\sum_{i\in\mathcal{A}_j}\exp\Big(\frac{\bzin}{\tau^n}\Big)\Big[\daijn\cdot X^{(u)}+\dbijn\Big],\nonumber\\
    \label{eq:modified_coefficients_over}
    C_{n,2,j}(X)&:=\sum_{i\in\mathcal{A}_j}\exp\Big(\frac{\bzin}{\tau^n}\Big)\cdot\frac{\dvijn}{2},
\end{align}
for any $j\in[k_*]:|\mathcal{A}_j|=1$.
Meanwhile, for each $j:|\mathcal{A}_j|>1$, by means of the Taylor expansion of order $\brj$, we can rewrite $A_{n,2}$ as 
\begin{align*}
    A_{n,2}&=\sum_{j:|\mathcal{A}_j|>1}\sum_{i\in\mathcal{A}_j}\exp\Big(\frac{\bzin}{\tau^n}\Big)\sum_{|\alpha|=1}^{\brj}\frac{1}{\alpha!}(\dboijn)^{\alpha_1}(\dtn)^{\alpha_2}(\daijn)^{\alpha_3}(\dbijn)^{\alpha_4}(\dvijn)^{\alpha_5}\\
    &\hspace{5cm}\times \frac{\partial^{|\alpha_1|+\alpha_2+|\alpha_3|+\alpha_4+\alpha_5} F}{\partial\beta_{1}^{\alpha_1}~\partial\tau^{\alpha_2}~\partial a^{\alpha_3}~\partial b^{\alpha_4}~\partial \nu^{\alpha_5}}(Y;X,\boj,\tau^*,\aj,\bj,\vj) + R_2(X,Y),\\
    &=\sum_{j:|\mathcal{A}_j|>1}\sum_{i\in\mathcal{A}_j}\exp\Big(\frac{\bzin}{\tau^n}\Big)\sum_{|\alpha|=1}^{\brj}\frac{1}{\alpha!}(\dboijn)^{\alpha_1}(\dtn)^{\alpha_2}(\daijn)^{\alpha_3}(\dbijn)^{\alpha_4}(\dvijn)^{\alpha_5}\\
    &\hspace{2.2cm}\times \frac{X^{\alpha_3}}{2^{\alpha_5}}\cdot\frac{\partial^{|\alpha_1|+\alpha_2}L}{\partial\beta_1^{\alpha_1}\partial\tau^{\alpha_2}}(X,\boj,\tau^*)\cdot\frac{\partial^{|\alpha_3|+\alpha_4+2\alpha_5}f}{\partial h_1^{|\alpha_3|+\alpha_4+2\alpha_5}}(Y;X,\boj,\tau^*,\aj,\bj,\vj)+R_2(X,Y).
\end{align*}
where $L(X,\beta_1,\tau):=\exp\Big(\frac{\bsigma(X,\beta_1)}{\tau}\Big)$ and  $R_2(X,Y)$ is a Taylor remainder such that $R_2(X,Y)/\mathcal{D}_{6n}\to0$ as $n\to\infty$. For each $j:|\mathcal{A}_j|>1$, by letting $|\alpha_1|+\alpha_2=s$, where $0\leq s\leq \brj$, then we have $1-s\leq |\alpha_3|+\alpha_4+\alpha_5\leq \brj-s$. Next, we denote $\alpha_4+2\alpha_5=\ell$, where $0\leq\ell\leq 2(\brj-s-|\alpha_3|)$. Then, $A_{n,2}$ can be represented as
\begin{align}
    \label{eq:A_2_over}
    A_{n,2}&=\sum_{j:|\mathcal{A}_j|>1}\sum_{s=0}^{\brj}\sum_{|\alpha_3|=0}^{\brj-s}\sum_{\ell=0}^{2(\brj-s-|\alpha_3|)}T_{n,s,\alpha_3,\ell,j}(X) \cdot X^{\alpha_3}{F}^{(|\alpha_3|+\ell)}(Y;X,\boj,\tau^*,\aj,\bj,\vj)+R_2(X,Y),
\end{align}
where
\begin{align*}
    T_{n,s,\alpha_3,\ell,j}(X)&:=\sum_{\substack{\alpha_4+2\alpha_5=\ell, \\ 1-s\leq\alpha_4+\alpha_5\leq \brj-s}}\sum_{i\in\mathcal{A}_j}\exp\Big(\frac{\bzin}{\tau^n}\Big)\frac{1}{2^{\alpha_5}\alpha_3!\alpha_4!\alpha_5!}(\daijn)^{\alpha_3}(\dbijn)^{\alpha_4}(\dvijn)^{\alpha_5}\\
&\hspace{2cm}\times \Big[\sum_{|\alpha_1|+\alpha_2=s}\frac{1}{\alpha_1!\alpha_2!}(\dboijn)^{\alpha_1}(\dtn)^{\alpha_2}\cdot\frac{\partial^{|\alpha_1|+\alpha_2}L}{\partial\beta_1^{\alpha_1}\partial\tau^{\alpha_2}}(X,\boj,\tau^*)\cdot\frac{1}{L(X,\boj,\tau^*)}\Big].
\end{align*}
Now, we provide the explicit formulations of $T_{n,s,\alpha_3,\ell,j}(X)$ for $s\in\{0,1,2\}$. First, the term $T_{n,0,\alpha_3,\ell,j}(X)$ is given by:

\begin{align*}
    T_{n,0,\alpha_3,\ell,j}(X):=\sum_{\substack{\alpha_4+2\alpha_5=\ell, \\ 1\leq\alpha_4+\alpha_5\leq \brj}}\sum_{i\in\mathcal{A}_j}\exp\Big(\frac{\bzin}{\tau^n}\Big)\frac{1}{2^{\alpha_5}\alpha_3!\alpha_4!\alpha_5!}(\daijn)^{\alpha_3}(\dbijn)^{\alpha_4}(\dvijn)^{\alpha_5}
\end{align*}
For the term $T_{n,1,\alpha_3,\ell,j}(X)$, let us derive the first derivatives of function $L$ w.r.t its parameters as
\begin{align*}
     \frac{\partial L}{\partial\beta_{1}^{(u)}}(X,\boj,\tau^*)&=\frac{1}{\tau^*}\cdot\frac{\partial\bsigma}{\partial \beta_1^{(u)}}(X,\boj)L(X,\boj,\tau^*),\nonumber\\
    \frac{\partial L}{\partial\tau}(X,\boj,\tau^*)&=-\frac{\bsigma(X,\boj)}{(\tau^*)^2}L(X,\boj,\tau^*).
\end{align*}
Thus, the formulation of $T_{n,1,\alpha_3,\ell,j}(X)$ reads as
\begin{align*}
T_{n,1,\alpha_3,\ell,j}(X)&=\sum_{\substack{\alpha_4+2\alpha_5=\ell, \\ 1-1\leq\alpha_4+\alpha_5\leq \brj-1}}\sum_{i\in\mathcal{A}_j}\exp\Big(\frac{\bzin}{\tau^n}\Big)\frac{1}{2^{\alpha_5}\alpha_3!\alpha_4!\alpha_5!}(\daijn)^{\alpha_3}(\dbijn)^{\alpha_4}(\dvijn)^{\alpha_5}\\
&\hspace{4cm} \times\Big[\sum_{u=1}^{d}\frac{(\dboijn)^{(u)}}{\tau^*}\cdot \frac{\partial\bsigma}{\partial \beta_1^{(u)}}(X,\boj)-\frac{(\dtn)}{(\tau^*)^2}\cdot\bsigma(X,\boj)\Big].
\end{align*}
Similarly, for the term $T_{n,2,\alpha_3,\ell,j}$, we derive the second derivatives of function $L$ w.r.t its parameters as
\begin{align*}
    \frac{\partial^2L}{\partial\beta_1^{(u)}\partial\beta_1^{(v)}}(X,\boj,\tau^*)&=\frac{1}{\tau^*}\cdot \frac{\partial^2\sigma}{\partial \beta_1^{(u)}\partial\beta_1^{(v)}}((\boj)^{\top}X)L(X,\boj,\tau^*) \\
    &\hspace{3cm}+\frac{1}{(\tau^*)^2}\cdot \Big[\frac{\partial\bsigma}{\partial \beta_1^{(u)}}(X,\boj)\Big]\Big[\frac{\partial\sigma}{\partial \beta_1^{(v)}}((\boj)^{\top}X)\Big]L(X,\boj,\tau^*),\\
    \frac{\partial^2L}{\partial\tau^2}(X,\boj,\tau^*)&=\frac{2}{(\tau^*)^3}\cdot\bsigma(X,\boj)L(X,\boj,\tau^*) + \frac{1}{(\tau^*)^4}\cdot\bsigma^2(X,\boj)L(X,\boj,\tau^*),\\
    \frac{\partial^2L}{\partial\beta_1^{(u)}\partial \tau}(X,\boj,\tau^*)&=-\frac{1}{(\tau^*)^2}\cdot \frac{\partial\bsigma}{\partial \beta_1^{(u)}}(X,\boj)L(X,\boj,\tau^*) \\
    &\hspace{5cm}-\frac{1}{(\tau^*)^3}\cdot \bsigma(X,\boj)\frac{\partial\bsigma}{\partial \beta_1^{(u)}}(X,\boj)L(X,\boj,\tau^*).
\end{align*}
Then, $T_{n,2,\alpha_3,\ell,j}$ can be written as
\begin{align*}
    &T_{n,2,\alpha_3,\ell,j}(X):=\sum_{\substack{\alpha_4+2\alpha_5=\ell, \\ 0\leq\alpha_4+\alpha_5\leq \brj-2}}\sum_{i\in\mathcal{A}_j}\exp\Big(\frac{\bzin}{\tau^n}\Big)\frac{1}{2^{\alpha_5}\alpha_3!\alpha_4!\alpha_5!}(\daijn)^{\alpha_3}(\dbijn)^{\alpha_4}(\dvijn)^{\alpha_5}\\
    &\times\Bigg\{\sum_{1\leq u,v\leq d}\frac{(\dboijn)^{(u)}(\dboijn)^{(v)}}{1+\mathbf{1}_{\{u=v\}}}\Big[\frac{1}{\tau^*}\cdot \frac{\partial^2\bsigma}{\partial\beta_1^{(u)}\partial\beta_1^{(v)}}+\frac{1}{(\tau^*)^2}\cdot X^{(u)}X^{(v)}\Big(\frac{\partial\sigma}{\partial g}((\boj)^{\top}X)\Big)^2\Big]\nonumber\\
    &\hspace{6cm}+\frac{1}{2}(\dtn)^2\Big[\frac{2}{(\tau^*)^3}\cdot\bsigma(X,\boj)+\frac{1}{(\tau^*)^4}\cdot \bsigma^2(X,\boj)\Big]\nonumber\\
    &\hspace{2cm}-\sum_{u=1}^{d}(\dboijn)^{(u)}(\dtn)\Big[\frac{1}{(\tau^*)^2}\cdot \frac{\partial\bsigma}{\partial\beta_1^{(u)}}+\frac{1}{(\tau^*)^3}\cdot \bsigma(X,\boj)\frac{\partial\bsigma}{\partial \beta_1^{(u)}}(X,\boj)\Big]\Bigg\}.
\end{align*}
Thus, the term $[A_n-R_1(X,Y)-R_2(X,Y)]/\mathcal{D}_{6n}$ can be viewed as a linear combination of elements from the union of the following sets:
\begin{align*}
    \mathcal{F}_{1,0}&:=\Big\{\frac{\partial \bsigma}{\partial \beta_1^{(u)}}(X,\boj)F(Y;X,\omega^*_j):u\in[d], \ j:|\mathcal{A}_j|=1\Big\}
    \cup\Big\{\bsigma(X,\boj)F(Y;X,\omega^*_j): j:|\mathcal{A}_j|=1\Big\},\\
    \mathcal{F}_{1,1}&:=\Big\{X^{(u)}F^{(1)}(Y;X,\omega^*_j): u\in[d], \  j:|\mathcal{A}_j|=1\Big\}\cup\Big\{F^{(1)}(Y;X,\omega^*_j):  j:|\mathcal{A}_j|=1\Big\},\\
    \mathcal{F}_{1,2}&:=\Big\{F^{(2)}(Y;X,\omega^*_j):j:|\mathcal{A}_j|=1\Big\},\\
    \mathcal{F}_{2,0}&:=\Big\{X^{\alpha_3}F^{(|\alpha_3|+\ell)}(Y;X,\omega^*_j):j:|\mathcal{A}_j|>1,\ 0\leq|\alpha_3|\leq\brj, \ 0\leq\ell\leq 2(\brj-|\alpha_3|)\Big\},\\
    \mathcal{F}_{2,1}&:=\Big\{X^{\alpha_3}\cdot\frac{\partial\bsigma}{\partial \beta_1^{(u)}}(X,\boj)F^{(|\alpha_3|+\ell)}(Y;X,\omega^*_j), \ X^{\alpha_3}\cdot\bsigma(X,\boj)F^{(|\alpha_3|+\ell)}(Y;X,\omega^*_j):\\
    &\hspace{5cm}u\in[d], \ j:|\mathcal{A}_j|>1,\ 0\leq|\alpha_3|\leq\brj-1, \ 0\leq\ell\leq 2(\brj-1-|\alpha_3|)\Big\},\\
    \mathcal{F}_{2,2}&:=\Big\{X^{\alpha_3}\cdot\frac{\partial^2\sigma}{\partial \beta_1^{(u)}\partial\beta_1^{(v)}}((\boj)^{\top}X)F^{(|\alpha_3|+\ell)}(Y;X,\omega^*_j), \\
    &\qquad X^{\alpha_3}\cdot\Big[\frac{\partial\bsigma}{\partial \beta_1^{(u)}}(X,\boj)\Big]\Big[\frac{\partial\sigma}{\partial \beta_1^{(v)}}((\boj)^{\top}X)\Big]F^{(|\alpha_3|+\ell)}(Y;X,\omega^*_j),\\
    &\qquad X^{\alpha_3}\bsigma(X,\boj)F^{(|\alpha_3|+\ell)}(Y;X,\omega^*_j), \ X^{\alpha_3}\bsigma^2(X,\boj)F^{(|\alpha_3|+\ell)}(Y;X,\omega^*_j),\\
    &\qquad X^{\alpha_3}\frac{\partial \bsigma}{\partial \beta_1^{(u)}}(X,\boj)F^{(|\alpha_3|+\ell)}(Y;X,\omega^*_j), \ X^{\alpha_3}\bsigma(X,\boj)\frac{\partial \bsigma}{\partial \beta_1^{(u)}}(X,\boj)F^{(|\alpha_3|+\ell)}(Y;X,\omega^*_j):\\
    &\hspace{6cm} u,v\in[d], \ j:|\mathcal{A}_j|>1,\ 0\leq|\alpha_3|\leq\brj-2, \ 0\leq\ell\leq 2(\brj-2-|\alpha_3|)\Big\},\\
    \mathcal{F}_{2,s}&:=\Big\{\sum_{|\alpha_1|+\alpha_2=s}X^{\alpha_3}\cdot\frac{\partial^{|\alpha_1|+\alpha_2}L}{\partial\beta_1^{\alpha_1}\partial\tau^{\alpha_2}}(X,\boj,\tau^*)\cdot\frac{1}{L(X,\boj,\tau^*)}F^{(|\alpha_3|+\ell)}(Y;X,\omega^*_j):\\
    &\hspace{7cm} j:|\mathcal{A}_j|>1,\ 0\leq|\alpha_3|\leq\brj-s, \ 0\leq\ell\leq 2(\brj-s-|\alpha_3|)\Big\}.
\end{align*}
for any $s\geq 3$, where $\omega^*_j:=(\boj,\tau^*,\aj,\bj,\vj)$ for any $j\in[k_*]$. Additionally, it is also worth noting that $E_{n,1}/\mathcal{D}_{6n}$ can be seen as a linear combination of elements from the set $\{F(Y;X,\omega^*_j):j\in[k_*]\}$.

Similarly, we also decompose $B_n$ into two terms as follows:
\begin{align*}
    B_n&:=\sum_{j:|\mathcal{A}_j|=1}\sum_{i\in\mathcal{A}_j}\exp\Big(\frac{\bzin}{\tau^n}\Big)\Big[H(Y;X,\boin,\tau^n)-H(Y;X,\boj,\tau^*)\Big]\\
    &+\sum_{j:|\mathcal{A}_j|>1}\sum_{i\in\mathcal{A}_j}\exp\Big(\frac{\bzin}{\tau^n}\Big)\Big[H(Y;X,\boin,\tau^n)-H(Y;X,\boj,\tau^*)\Big]\\
    &:=B_{n,1}+B_{n,2}.
\end{align*}
Subsequently, we apply the first-order Taylor expansion to $B_{n,1}$ as in equation~\eqref{eq:modified_first_order_2}, and get that
\begin{align}
    \label{eq:B_1_over}
    B_{n,1}&=\sum_{j:|\mathcal{A}_j|=1}\sum_{i\in\mathcal{A}_j}\exp\Big(\frac{\bzin}{\tau^n}\Big)\sum_{|\gamma|=1}\frac{1}{\gamma!}(\dboijn)^{\gamma_1}(\dtn)^{\gamma_2}\cdot\frac{\partial H}{\partial \beta_1^{\gamma_1}~\partial\tau^{\gamma_2}}(Y;X,\boj,\tau^*) + R_3(X,Y)\nonumber\\
    &=\sum_{j:|\mathcal{A}_j|=1}C_{n,0,j}(X)\cdot H(Y;X,\boj,\tau^*)+R_3(X,Y),
\end{align}
where $C_{n,0,j}(X)$ is defined in equation~\eqref{eq:modified_coefficients_exact} and $R_3(X,Y)$ is a Taylor remainder such that $R_3(X,Y)/\mathcal{D}_{6n}\to0$ as $n\to\infty$. On the other hand, by means of the second-order Taylor expansion, we rewrite $B_{n,2}$ as
\begin{align}
    \label{eq:B_2_over}
    B_{n,2}&=\sum_{j:|\mathcal{A}_j|>1}\sum_{i\in\mathcal{A}_j}\exp\Big(\frac{\bzin}{\tau^n}\Big)\sum_{|\gamma|=1}^{2}\frac{1}{\gamma!}(\dboijn)^{\gamma_1}(\dtn)^{\gamma_2}\cdot\frac{\partial^{|\gamma_1|+\gamma_2} H}{\partial \beta_1^{\gamma_1}~\partial\tau^{\gamma_2}}(Y;X,\boj,\tau^*) + R_4(X,Y)\nonumber\\
    &=\sum_{j:|\mathcal{A}_j|>1}S_{n,j}(X)\cdot H(Y;X,\boj,\tau^*)+R_4(X,Y),
\end{align}
where $R_4(X,Y)$ is a Taylor remainder such that $R_4(X,Y)/\mathcal{D}_{6n}\to0$ as $n\to\infty$ and
\begin{align}
    \label{eq:modified_coefficients_over_2}
    S_{n,j}(X)&:=\sum_{i\in\mathcal{A}_j}\exp\Big(\frac{\bzin}{\tau^n}\Big)\Bigg\{\sum_{u=1}^{d}\frac{(\dboijn)^{(u)}}{\tau^*}\cdot \frac{\partial\bsigma}{\partial \beta_1^{(u)}}(X,\boj)-\frac{\dtn}{(\tau^*)^2}\cdot\bsigma(X,\boj)\nonumber\\
    &-\sum_{u=1}^{d}(\dboijn)^{(u)}(\dtn)\Big[\frac{1}{(\tau^*)^2}\cdot \frac{\partial\bsigma}{\partial\beta_1^{(u)}}+\frac{1}{(\tau^*)^3}\cdot \bsigma(X,\boj)\frac{\partial\bsigma}{\partial \beta_1^{(u)}}(X,\boj)\Big]\nonumber\\
    &+\sum_{1\leq u,v\leq d}\frac{(\dboijn)^{(u)}(\dboijn)^{(v)}}{1+\mathbf{1}_{\{u=v\}}}\Big[\frac{1}{\tau^*}\cdot \frac{\partial^2\bsigma}{\partial\beta_1^{(u)}\partial\beta_1^{(v)}}+\frac{1}{(\tau^*)^2}\cdot \frac{\partial\bsigma}{\partial\beta_1^{(u)}}(X,\boj)\frac{\partial\bsigma}{\partial\beta_1^{(v)}}(X,\boj)\Big]\nonumber\\
    &+\frac{1}{2}(\dtn)^2\Big[\frac{2}{(\tau^*)^3}\cdot\bsigma(X,\boj)+\frac{1}{(\tau^*)^4}\cdot \bsigma^2(X,\boj)\Big]\Bigg\},
\end{align}
for any $j:|\mathcal{A}_j|>1$. Therefore, the term $[B_n-R_3(X,Y)-R_4(X,Y)]/\mathcal{D}_{6n}$ can be treated as a linear combination of elements from the following set:
\begin{align*}
    \mathcal{H}&:=\Big\{\frac{\partial \bsigma}{\partial \beta_1^{(u)}}(X,\boj)H(Y;X,\boj,\tau^*), \ \bsigma(X,\boj)H(Y;X,\boj,\tau^*):u\in[d], j\in[k_*]\Big\}\\
    &~\cup\Big\{\bsigma(X,\boj)\frac{\partial \bsigma}{\partial \beta_1^{(u)}}(X,\boj)H(Y;X,\boj,\tau^*), \ \bsigma^2(X,\boj)H(Y;X,\boj,\tau^*):u\in[d], \ j:|\mathcal{A}_j|>1\Big\}\\
    &~\cup\Big\{\frac{\partial^2\bsigma}{\partial \beta_1^{(u)}\partial\beta_1^{(v)}}(X,\boj)H(Y;X,\boj,\tau^*), \ \Big[\frac{\partial\bsigma}{\partial \beta_1^{(u)}}(X,\boj)\Big]\Big[\frac{\partial\bsigma}{\partial \beta_1^{(u)}}(X,\boj)\Big]H(Y;X,\boj,\tau^*)\\
    &\hspace{12cm}:u,v\in[d], \ j:|\mathcal{A}_j|>1\Big\}.
\end{align*}
In addition, we can view the term $E_{n,2}/\mathcal{D}_{6n}$ as a linear combination of elements from the set $\{H(Y;X,\boj,\tau^*):j\in[k_*]\}$.

\textbf{Step 2.} In this step, we prove by contradiction that at least one among the coefficients in the representations of $[A_n-R_1(X,Y)-R_2(X,Y)]/\mathcal{D}_{6n}$, $[B_n-R_3(X,Y)-R_4(X,Y)]/\mathcal{D}_{6n}$, $E_{n,1}/\mathcal{D}_{6n}$ and $E_{n,2}/\mathcal{D}_{6n}$ does not converge to zero when $n\to\infty$. Assume that all of them go to 0 as $n\to\infty$. 
By using the same arguments for showing the results in equations~\eqref{eq:limit_bias_exact} and \eqref{eq:limit_parameter_exact}, we get that
\begin{align}
    \label{eq:limit_bias_over}
    \frac{1}{\mathcal{D}_{6n}}\cdot\sum_{j=1}^{k_*}\Big|\sum_{i\in\mathcal{A}_j}\exp\Big(\frac{\bzin}{\tau^n}\Big)-\exp\Big(\frac{\bzj}{\tau^*}\Big)\Big|\to0,
\end{align}
and
\begin{align}
    \label{eq:limit_parameter_over_1}
    \frac{1}{\mathcal{D}_{6n}}\cdot\sum_{j:|\mathcal{A}_j|=1}\sum_{i\in\mathcal{A}_j}\exp\Big(\frac{\bzin}{\tau^n}\Big)\Big(\|\dboijn\|+|\dtn|+\|\daijn\|+|\dbijn|+|\dvijn|\Big)\to0.
\end{align}
Next, by taking the summation of the absolute values of the coefficients associated with 
\begin{itemize}
    \item $\dfrac{\partial^2\sigma}{\partial [\beta_1^{(u)}]^2}(X,\boj)H(Y;X,\boj,\tau^*)$ in $\mathcal{H}$: we have that $\frac{1}{\mathcal{D}_{6n}}\cdot\sum_{j:|\mathcal{A}_j|>1}\sum_{i\in\mathcal{A}_j}\exp\Big(\frac{\bzin}{\tau^n}\Big)\|\dboijn\|^2\to0$;
    \item $\bsigma^2(X,\boj)H(Y;X,\boj,\tau^*)$ in $\mathcal{H}$: we have that $\frac{1}{\mathcal{D}_{6n}}\cdot\sum_{j:|\mathcal{A}_j|>1}\sum_{i\in\mathcal{A}_j}\exp\Big(\frac{\bzin}{\tau^n}\Big)|\dtn|^2\to0$;
    \item $[X^{(u)}]^2F^{(2)}(Y;X,\omega^*_j)$ in $\mathcal{F}_{2,0}$: we have that $\frac{1}{\mathcal{D}_{6n}}\cdot\sum_{j:|\mathcal{A}_j|>1}\sum_{i\in\mathcal{A}_j}\exp\Big(\frac{\bzin}{\tau^n}\Big)\|\daijn\|^2\to0$.
\end{itemize}
As a result, we obtain that
\begin{align}
    \label{eq:limit_parameter_over_2}
    \frac{1}{\mathcal{D}_{6n}}\cdot\sum_{j:|\mathcal{A}_j|>1}\sum_{i\in\mathcal{A}_j}\exp\Big(\frac{\bzin}{\tau^n}\Big)\Big[\|\dboijn\|^2+|\dtn|^2+\|\daijn\|^2\Big]\to0
\end{align}
From the results in equations~\eqref{eq:limit_bias_over}, \eqref{eq:limit_parameter_over_1} and \eqref{eq:limit_parameter_over_2}, we deduce that 
\begin{align*}
    \frac{1}{\mathcal{D}_{6n}}\cdot\sum_{j:|\mathcal{A}_j|>1}\sum_{i\in\mathcal{A}_j}\exp\Big(\frac{\bzin}{\tau^n}\Big)\Big[|\dbijn|^{\brj}+|\dvijn|^{\brj/2}\Big]\to1,
\end{align*}
which means that there exists an index $j:|\mathcal{A}_j|>1$, which can be assumed WLOG to be $j=1$, such that 
\begin{align}
    \label{eq:non_zero_limit}
    \frac{1}{\mathcal{D}_{6n}}\cdot\sum_{i\in\mathcal{A}_1}\exp\Big(\frac{\bzin}{\tau^n}\Big)\Big[|\dbione|^{\brone}+|\dvione|^{\brone/2}\Big]\not\to0.
\end{align}
Moreover, since the coefficients of elements $F^{(|\alpha_3|+\ell)}(Y;X,\omega^*_j)$, for $j=1$, $\alpha_3=\zerod$ and $0\leq\ell\leq 2\brj$ in the set $\mathcal{F}_{2,0}$ converges to zero, i.e.
\begin{align}
    \label{eq:zero_limits}
    \frac{1}{\mathcal{D}_{6n}}\cdot\sum_{i\in\mathcal{A}_1}\sum_{\substack{\alpha_4+2\alpha_5=\ell, \\ 1\leq\alpha_4+\alpha_5\leq \brone}}\frac{\exp\Big(\frac{\bzin}{\tau^n}\Big)}{2^{\alpha_5}\alpha_3!\alpha_4!\alpha_5!}(\dbijn)^{\alpha_4}(\dvijn)^{\alpha_5}\to0,
\end{align}
for any $1\leq\ell\leq\brone$. Then, we divide the left hand side of equation~\eqref{eq:zero_limits} by that of equation~\eqref{eq:non_zero_limit}, and achieve that
\begin{align}
    \label{eq:vanish_ratio}
    \dfrac{\sum_{i\in\mathcal{A}_1}\sum_{\substack{\alpha_4+2\alpha_5=\ell, \\ 1\leq\alpha_4+\alpha_5\leq \brone}}\dfrac{\exp\Big(\dfrac{\bzin}{\tau^n}\Big)}{2^{\alpha_5}\alpha_4!\alpha_5!}(\dbijn)^{\alpha_4}(\dvijn)^{\alpha_5}}{\sum_{i\in\mathcal{A}_1}\exp\Big(\frac{\bzin}{\tau^n}\Big)\Big[|\dbione|^{\brone}+|\dvione|^{\brone/2}\Big]}\to0,
\end{align}
for any $1\leq\ell\leq\brone$.

Let us define $\overline{M}_n:=\max\{|\dbione|,|\dvione|^{1/2}:i\in\mathcal{A}_1\}$ and $\overline{\pi}_n:=\max_{i\in\mathcal{A}_1}\exp(\frac{\bzin}{\tau^n})$. Since the sequence $\exp\Big(\frac{\bzin}{\tau^n}\Big)/\overline{\pi}_n$ is bounded, it is possible to replace it by its subsequence that has a positive limit $q^2_{3i}:=\lim_{n\to\infty}\exp(\frac{\bzin}{\tau^n})/\overline{\pi}_n$. Thus, at least one among $q^2_{3i}$, for $i\in\mathcal{A}_1$, is equal to one. 

In addition, we also define
\begin{align*}
    (\dbione)/\overline{M}_n\to q_{4i},& \quad (\dvione)/[2\overline{M}_n]\to q_{5i}.
\end{align*}
It is worth noting that at least one among $q_{4i}$ and $q_{5i}$ for $i\in\mathcal{A}_1$ is equal to either $1$ or $-1$. Subsequently, we divide both the numerator and the denominator of the ratio in equation~\eqref{eq:vanish_ratio} by $\overline{\pi}_n\overline{M}_n^{\ell}$, and then obtain the following system of polynomial equations:
\begin{align*}
    \sum_{i\in\mathcal{A}_1}\sum_{\substack{\alpha_4+2\alpha_5=\ell, \\ 1\leq\alpha_4+\alpha_5\leq\brone}}\frac{q^2_{3i}~q_{4i}^{\alpha_4}~q_{5i}^{\alpha_5}}{\alpha_4!~\alpha_5!}=0,
\end{align*}
for all $1\leq\ell\leq\brone$. However, from the definition of $\bar{r}(|\mathcal{A}_1|)$, the above system does not have any non-trivial solutions, which contradicts to the fact that at least one among $q_{4i}$ and $q_{5i}$ for $i\in\mathcal{A}_1$ is non-zero. Therefore, not all the coefficients in the representations of $[A_n-R_1(X,Y)-R_2(X,Y)]/\mathcal{D}_{6n}$, $[B_n-R_3(X,Y)-R_4(X,Y)]/\mathcal{D}_{6n}$, $E_{n,1}/\mathcal{D}_{6n}$ and $E_{n,2}/\mathcal{D}_{6n}$ converge to zero as $n\to\infty$.

\textbf{Step 3.} In this step, we use the Fatou's lemma to show that all the coefficients in the representations of $[A_n-R_1(X,Y)-R_2(X,Y)]/\mathcal{D}_{6n}$, $[B_n-R_3(X,Y)-R_4(X,Y)]/\mathcal{D}_{6n}$, $E_{n,1}/\mathcal{D}_{6n}$ and $E_{n,2}/\mathcal{D}_{6n}$ converge to zero as $n\to\infty$, which leads to a contradiction to the results in Step 2. In particular, let us denote $m_n$ as the maximum of the absolute values of those coefficients. It follows from the claim in Step 2 that $1/m_n\not\to\infty$. Next, we denote
\begin{align*}
    \frac{1}{m_n\mathcal{D}_{6n}}\cdot\Big[\sum_{i\in\mathcal{A}_j}\exp\Big(\frac{\bzin}{\tau^n}\Big)-\exp\Big(\frac{\bzj}{\tau^*}\Big)\Big]&\to\phi_{0,j}, \qquad 
    \frac{1}{m_n\mathcal{D}_{6n}}\cdot\sum_{i\in\mathcal{A}_j}\exp\Big(\frac{\bzin}{\tau^n}\Big)(\dboijn)^{(u)}\to\phi^{(u)}_{1,j},\\
    \frac{1}{m_n\mathcal{D}_{6n}}\cdot\sum_{i\in\mathcal{A}_j}\exp\Big(\frac{\bzin}{\tau^n}\Big)(\dtn)&\to\phi_{2,j},\qquad 
    \frac{1}{m_n\mathcal{D}_{6n}}\cdot\sum_{i\in\mathcal{A}_j}\exp\Big(\frac{\bzin}{\tau^n}\Big)(\daijn)^{(u)}\to\phi^{(u)}_{3,j},\\
    \frac{1}{m_n\mathcal{D}_{6n}}\cdot\sum_{i\in\mathcal{A}_j}\exp\Big(\frac{\bzin}{\tau^n}\Big)(\dbijn)&\to\phi_{4,j},\qquad 
    \frac{1}{m_n\mathcal{D}_{6n}}\cdot\sum_{i\in\mathcal{A}_j}\exp\Big(\frac{\bzin}{\tau^n}\Big)(\dvijn)\to\phi_{5,j},
\end{align*}
as $n\to\infty$ for any $u\in[d]$ and $j\in[k_*]$. 
By means of the Fatou's lemma, we have that
\begin{align*}
    \lim_{n\to\infty}\frac{ \bbE_X[V(p_{G_n}(\cdot|X),p_{G_*}(\cdot|X))]}{m_n\mathcal{D}_{6n}}\geq\int\liminf_{n\to\infty}\frac{|p_{G_n}(Y|X)-p_{G_*}(Y|X)|}{2m_n\mathcal{D}_{6n}}\dint(X,Y).
\end{align*}
Recall from equation~\eqref{eq:ratio_modified_over} that the limit the left hand side is equal to zero, which implies that $\frac{|p_{G_n}(Y|X)-p_{G_*}(Y|X)|}{m_n\mathcal{D}_{6n}}\to0$. 
as $n\to\infty$ for almost surely $(X,Y)$. Thus, we also get that $\frac{Q_n}{m_n\mathcal{D}_{6n}}\to0$ as $n\to\infty$, which implies that
\begin{align}
    \label{eq:zero_limit}
    \lim_{n\to\infty}\frac{1}{m_n\mathcal{D}_{6n}}\cdot[A_{n,1}+A_{n,2}-B_{n,1}-B_{n,2}+E_{n,1}-E_{n,2}]=\lim_{n\to\infty}\frac{Q_n}{m_n\mathcal{D}_{6n}}=0,
\end{align}
for almost surely $(X,Y)$. Now, we derive the limits of terms in the above right hand side. In particular, from the formulations of
\begin{itemize}
    \item $E_{n,1}$ and $E_{n,2}$ in equation~\eqref{eq:Q_n_formulation}, we have 
\begin{align}
    \label{eq:limit_E}
    \frac{E_{n,1}}{m_n\mathcal{D}_{6n}}\to\sum_{j=1}^{k_*}\phi_{0,j}F(Y;X,\omega^*_j), \qquad \frac{E_{n,2}}{m_n\mathcal{D}_{6n}}\to\sum_{j=1}^{k_*}\phi_{0,j}H(Y;X,\boj,\tau^*).
\end{align}
    \item $A_{n,1}$ in equation~\eqref{eq:A_1_over}, we deduce that
\begin{align}
    \label{eq:limit_A_1}
    \frac{A_{n,1}}{m_n\mathcal{D}_{6n}}\to \sum_{j:|\mathcal{A}_j|=1}\sum_{\eta=0}^{2}C^*_{\eta,j}(X)\cdot F^{(\eta)}(Y;X,\omega^*_j),
\end{align}
where 
\begin{align*}
    C^*_{0,j}(X)&:=\sum_{u=1}^{d}\frac{\phi^{(u)}_{1,j}}{\tau^*}\cdot \frac{\partial \bsigma}{\partial \beta_1^{(u)}}(X,\boj)-\phi_{2,j}\cdot\frac{\bsigma(X,\boj)}{(\tau^*)^2},\\
    C^*_{1,j}(X)&:=\sum_{u=1}^{d}\phi_{3,j}^{(u)}\cdot X^{(u)}+\phi_{4,j},\\
    C^*_{2,j}(X)&:=\frac{1}{2}\phi_{5,j},
\end{align*}
for any $j:|\mathcal{A}_j|=1$.
    \item $B_{n,1}$ in equation~\eqref{eq:B_1_over}, we get
    \begin{align}
    \label{eq:limit_B_1}
    \frac{B_{n,1}}{\mathcal{D}_{6n}}\to\sum_{j:|\mathcal{A}_j|=1}C^*_{0,j}(X)H(Y;X,\boj,\tau^*).
    \end{align}
    \item $A_{n,2}$ in equation~\eqref{eq:A_2_over}, we have
    \begin{align*}
        \lim_{n\to\infty}\frac{A_{n,2}}{\mathcal{D}_{6n}}=\sum_{j:|\mathcal{A}_j|>1}\sum_{s=0}^{\brj}\sum_{|\alpha_3|=0}^{\brj-s}\sum_{\ell=0}^{2(\brj-s-|\alpha_3|)}\lim_{n\to\infty}\frac{T_{n,s,\alpha_3,\ell,j}(X)}{\mathcal{D}_{6n}} \cdot X^{\alpha_3}{F}^{(|\alpha_3|+\ell)}(Y;X,\omega^*_j).
    \end{align*}
    From the arguments in Step 2, we deduce that the value of $m_n$ is the ratio between one element of the following set and the loss $\mathcal{D}_{6n}$:
    \begin{align}
        \label{eq:non_vanish_elements}
        \Big\{\exp\Big(\frac{\bzin}{\tau^n}\Big)|\dboijn|, \exp\Big(\frac{\bzin}{\tau^n}\Big)|\dtn|,  \exp\Big(\frac{\bzin}{\tau^n}\Big)|\daijn|,  \exp\Big(\frac{\bzin}{\tau^n}\Big)|\dbijn|, \exp\Big(\frac{\bzin}{\tau^n}\Big)|\dvijn|, i\in\mathcal{A}_j, j:|\mathcal{A}_j|>1\Big\}\nonumber\\
        \cup~\Big\{\sum_{i\in\mathcal{A}_j}(\dboijn)^2, \sum_{i\in\mathcal{A}_j}(\dtn)^2, \sum_{i\in\mathcal{A}_j}(\daijn)^2, j:|\mathcal{A}_j|>1\Big\}.
    \end{align}
    Thus, the associated coefficients $T_{n,s,\alpha_3,\ell,j}/\mathcal{D}_{6n}$ in the representation of $A_{n,2}/\mathcal{D}_{6n}$ converge to zero as $n\to\infty$ for any $s\geq 3$. Therefore, we consider only the limits of $T_{n,s,\alpha_3,\ell,j}/\mathcal{D}_{6n}$ for $s\in\{0,1,2\}$. In particular, let us denote
    \begin{align*}
       \frac{1}{m_n\mathcal{D}_{6n}}\cdot\sum_{i\in\mathcal{A}_j} \sum_{\substack{\alpha_4+2\alpha_5=\ell, \\ 1\leq\alpha_4+\alpha_5\leq \brj}}\exp\Big(\frac{\bzin}{\tau^n}\Big)\frac{1}{2^{\alpha_5}\alpha_3!\alpha_4!\alpha_5!}(\daijn)^{\alpha_3}(\dbijn)^{\alpha_4}(\dvijn)^{\alpha_5}\to\psi_{0,\alpha_3,\ell,j},\\
        \frac{1}{m_n\mathcal{D}_{6n}}\cdot\sum_{i\in\mathcal{A}_j} \sum_{\substack{\alpha_4+2\alpha_5=\ell, \\ 1\leq\alpha_4+\alpha_5\leq \brj}}\exp\Big(\frac{\bzin}{\tau^n}\Big)\frac{1}{2^{\alpha_5}\alpha_3!\alpha_4!\alpha_5!}(\daijn)^{\alpha_3}(\dbijn)^{\alpha_4}(\dvijn)^{\alpha_5}(\dboijn)^{(u)}\to\psi^{(u)}_{1,\alpha_3,\ell,j},\\
        \frac{1}{m_n\mathcal{D}_{6n}}\cdot\sum_{i\in\mathcal{A}_j} \sum_{\substack{\alpha_4+2\alpha_5=\ell, \\ 1\leq\alpha_4+\alpha_5\leq \brj}}\exp\Big(\frac{\bzin}{\tau^n}\Big)\frac{1}{2^{\alpha_5}\alpha_3!\alpha_4!\alpha_5!}(\daijn)^{\alpha_3}(\dbijn)^{\alpha_4}(\dvijn)^{\alpha_5}(\dtn)\to\psi_{2,\alpha_3,\ell,j},\\
        \frac{1}{m_n\mathcal{D}_{6n}}\cdot\sum_{i\in\mathcal{A}_j} \sum_{\substack{\alpha_4+2\alpha_5=\ell, \\ 1\leq\alpha_4+\alpha_5\leq \brj}}\exp\Big(\frac{\bzin}{\tau^n}\Big)\frac{1}{2^{\alpha_5}\alpha_3!\alpha_4!\alpha_5!}(\daijn)^{\alpha_3}(\dbijn)^{\alpha_4}(\dvijn)^{\alpha_5}(\dboijn)^{(u)}(\dboijn)^{(v)}\to\psi^{(u,v)}_{3,\alpha_3,\ell,j},\\
        \frac{1}{m_n\mathcal{D}_{6n}}\cdot\sum_{i\in\mathcal{A}_j} \sum_{\substack{\alpha_4+2\alpha_5=\ell, \\ 1\leq\alpha_4+\alpha_5\leq \brj}}\exp\Big(\frac{\bzin}{\tau^n}\Big)\frac{1}{2^{\alpha_5}\alpha_3!\alpha_4!\alpha_5!}(\daijn)^{\alpha_3}(\dbijn)^{\alpha_4}(\dvijn)^{\alpha_5}(\dboijn)^{(u)}(\dtn)\to\psi^{(u)}_{4,\alpha_3,\ell,j},\\
        \frac{1}{m_n\mathcal{D}_{6n}}\cdot\sum_{i\in\mathcal{A}_j} \sum_{\substack{\alpha_4+2\alpha_5=\ell, \\ 1\leq\alpha_4+\alpha_5\leq \brj}}\exp\Big(\frac{\bzin}{\tau^n}\Big)\frac{1}{2^{\alpha_5}\alpha_3!\alpha_4!\alpha_5!}(\daijn)^{\alpha_3}(\dbijn)^{\alpha_4}(\dvijn)^{\alpha_5}(\dtn)^2\to\psi_{5,\alpha_3,\ell,j},
    \end{align*}
    for any $j:|\mathcal{A}_j|>1$ and $u,v\in[d]$. Then, we have that
    \begin{align}
        \label{eq:limit_A_2}
        \frac{A_{n,2}}{\mathcal{D}_{6n}}\to\sum_{j:|\mathcal{A}_j|>1}\sum_{s=0}^{2}\sum_{|\alpha_3|=0}^{\brj-s}\sum_{\ell=0}^{2(\brj-s-|\alpha_3|)}T^*_{s,\alpha_3,\ell,j}\cdot X^{\alpha_3}{F}^{(|\alpha_3|+\ell)}(Y;X,\omega^*_j),
    \end{align}
    where
    \begin{align*}
    T^*_{0,\alpha_3,\ell,j}&:=\psi_{0,\alpha_3,\ell,j},\\
        T^*_{1,\alpha_3,\ell,j}&:=\sum_{u=1}^{d}\frac{\psi^{(u)}_{1,\alpha_3,\ell,j}}{\tau^*}\cdot \frac{\partial\bsigma}{\partial \beta_1^{(u)}}(X,\boj)-\frac{\psi_{2,\alpha_3,\ell,j}}{(\tau^*)^2}\cdot\bsigma(X,\boj),\\
        T^*_{2,\alpha_3,\ell,j}&:=\sum_{1\leq u,v\leq d}\frac{\psi^{(u,v)}_{3,\alpha_3,\ell,j}}{1+\mathbf{1}_{\{u=v\}}}\Big[\frac{1}{\tau^*}\cdot \frac{\partial^2\sigma}{\partial \beta_1^{(u)}\partial \beta_1^{(v)}}((\boj)^{\top}X)+\frac{1}{(\tau^*)^2}\cdot \Big(\frac{\partial\bsigma}{\partial \beta_1^{(u)}}(X,\boj)\Big)\Big(\frac{\partial\sigma}{\partial \beta_1^{(v)}}((\boj)^{\top}X)\Big)\Big]\nonumber\\
    &\hspace{5.2cm}+\frac{1}{2}\psi_{5,\alpha_3,\ell,j}\Big[\frac{2}{(\tau^*)^3}\cdot\bsigma(X,\boj)+\frac{1}{(\tau^*)^4}\cdot \bsigma^2(X,\boj)\Big]\nonumber\\
    &\hspace{2.8cm}-\sum_{u=1}^{d}\psi^{(u)}_{4,\alpha_3,\ell,j}\Big[\frac{1}{(\tau^*)^2}\cdot \frac{\partial\bsigma}{\partial \beta_1^{(u)}}(X,\boj)+\frac{1}{(\tau^*)^3}\cdot \bsigma(X,\boj)\frac{\partial\bsigma}{\partial \beta_1^{(u)}}(X,\boj)\Big].
    \end{align*}
    \item $B_{n,2}$ in equation~\eqref{eq:B_2_over}, by denoting 
    \begin{align*}
        \frac{1}{m_n\mathcal{D}_{6n}}\cdot\sum_{i\in\mathcal{A}_j}\exp\Big(\frac{\bzin}{\tau^n}\Big)(\dboijn)^{(u)}(\dboijn)^{(v)}\to\varphi^{(u,v)}_{0,j},\\
        \frac{1}{m_n\mathcal{D}_{6n}}\cdot\sum_{i\in\mathcal{A}_j}\exp\Big(\frac{\bzin}{\tau^n}\Big)(\dboijn)^{(u)}(\dtn)\to\varphi^{(u)}_{1,j},\\
        \frac{1}{m_n\mathcal{D}_{6n}}\cdot\sum_{i\in\mathcal{A}_j}\exp\Big(\frac{\bzin}{\tau^n}\Big)(\dtn)^2\to\varphi_{2,j},
    \end{align*}
    we have
    \begin{align}
        \label{eq:limit_B_2}
        \frac{B_{n,2}}{\mathcal{D}_{6n}}\to\sum_{j:|\mathcal{A}_j|>1}S^*_{j}(X)\cdot H(Y;X,\boj,\tau^*),
    \end{align}
    where
    \begin{align*}
        S^*_{j}(X)&:=\Bigg\{\sum_{u=1}^{d}\frac{\phi^{(u)}_{1,j}}{\tau^*}\cdot \frac{\partial\bsigma}{\partial \beta_1^{(u)}}(X,\boj)-\frac{\phi_{2,j}}{(\tau^*)^2}\cdot\bsigma(X,\boj)-\sum_{u=1}^{d}\varphi_{1,j}^{(u)}\Big[\frac{1}{(\tau^*)^2}\cdot \frac{\partial\bsigma}{\partial \beta_1^{(u)}}(X,\boj)\nonumber\\
    &+\frac{1}{(\tau^*)^3}\cdot \bsigma(X,\boj)\frac{\partial\bsigma}{\partial \beta_1^{(u)}}(X,\boj)\Big]+\sum_{1\leq u,v\leq d}\frac{\varphi_{0,j}^{(u,v)}}{1+\mathbf{1}_{\{u=v\}}}\Big[\frac{1}{\tau^*}\cdot \frac{\partial^2\sigma}{\partial \beta_1^{(u)}\partial\beta_1^{(v)}}((\boj)^{\top}X)\nonumber\\
    &+\frac{1}{(\tau^*)^2}\cdot \Big(\frac{\partial\bsigma}{\partial \beta_1^{(u)}}(X,\boj)\Big)\Big(\frac{\partial\sigma}{\partial \beta_1^{(v)}}((\boj)^{\top}X)\Big)\Big]+\frac{1}{2}\varphi_{2,j}\Big[\frac{2}{(\tau^*)^3}\cdot\bsigma(X,\boj)+\frac{1}{(\tau^*)^4}\cdot \bsigma^2(X,\boj)\Big]\Bigg\}.
    \end{align*}
\end{itemize}
Recall that not all the ratios between elements in the set~\eqref{eq:non_vanish_elements} and the loss $\mathcal{D}_{6n}$ converge to zero as $n\to\infty$. Thus, at least one element of the following set union is non-zero:
\begin{align}
    \label{eq:at_least_one_zero}
   \Big\{\phi_{0,j},\phi^{(u)}_{1,j},\phi_{2,j},\phi^{(u)}_{3,j},\phi_{4,j},\phi_{5,j}:u\in[d],j:|\mathcal{A}_j|=1\Big\}\cup\Big\{\phi_{0,j},\psi_{0,2\mathbf{e}_{u},0,j}, \varphi^{(u,u)}_{0,j},\varphi_{2,j}:u\in[d],j:|\mathcal{A}_j|>1\Big\}
\end{align} 
Now, we show that all elements in the union~\eqref{eq:at_least_one_zero} must be zero. Indeed, putting the results in equations~\eqref{eq:zero_limit}, \eqref{eq:limit_E}, \eqref{eq:limit_A_1}, \eqref{eq:limit_B_1}, \eqref{eq:limit_A_2} and \eqref{eq:limit_B_2}, we obtain that
\begin{align}
    \label{eq:linearly_independent}
    &\sum_{j=1}^{k_*}\phi_{0,j}F(Y;X,\omega^*_j)+\sum_{j:|\mathcal{A}_j|=1}\sum_{\eta=0}^{2}C^*_{\eta,j}(X)F^{(\eta)}(Y;X,\omega^*_j)-\sum_{j:|\mathcal{A}_j|=1}(C^*_{0,j}(X)+\phi_{0,j})H(Y;X,\boj,\tau^*)\nonumber\\
    &+\sum_{j:|\mathcal{A}_j|>1}\sum_{s=0}^{2}\sum_{|\alpha_3|=0}^{\brj-s}\sum_{\ell=0}^{2(\brj-s-|\alpha_3|)}T^*_{s,\alpha_3,\ell,j}(X)\cdot X^{\alpha_3}{F}^{(|\alpha_3|+\ell)}(Y;X,\omega^*_j)\nonumber\\
    &\hspace{8cm}-\sum_{j:|\mathcal{A}_j|>1}(S^*_{j}(X)+\phi_{0,j})\cdot H(Y;X,\boj,\tau^*)=0,
\end{align}
for almost surely $(X,Y)$. It can be verified that the set
\begin{align*}
    &\Big\{F^{(\eta)}(Y;X,\omega^*_j), \ H(Y;X,\boj,\tau^*):0\leq\eta\leq 2, j:|\mathcal{A}_j|=1\Big\}\\
    \cup~&\Big\{X^{\alpha_3}F^{(|\alpha_3|+\ell)}(Y;X,\omega^*_j), \ H(Y;X,\boj,\tau^*):j:|\mathcal{A}_j|>1, \ 0\leq\alpha_3\leq \brj, \ 0\leq\ell\leq2(\brj-|\alpha_3|)\Big\}.
\end{align*}
is linearly independent w.r.t $Y$. Thus, in the left hand side of equation~\eqref{eq:linearly_independent}, the coefficients associated with the following terms must be zero. 
\begin{itemize}
    \item $F(Y;X,\omega^*_j)$, where $j:|\mathcal{A}_j|=1$: $\phi_{0,j}+C^*_{0,j}(X)=0$. More explicitly, we have
\begin{align*}
    \phi_{0,j}+\sum_{u=1}^{d}\frac{\phi^{(u)}_{1,j}}{\tau^*}\cdot \frac{\partial \bsigma}{\partial \beta_1^{(u)}}(X,\boj)-\phi_{2,j}\cdot\frac{\bsigma(X,\boj)}{(\tau^*)^2}=0,
\end{align*}
for any $j:|\mathcal{A}_j|=1$, for almost surely $X$. Since the function $\sigma$ satisfies the conditions in Definition~\ref{def:modified_function_over}, we deduce that $\phi_{0,j}=\phi_{1,j}^{(u)}=\phi_{2,j}=0$, for any $j:|\mathcal{A}_j|=1$.
    \item $F^{(1)}(Y;X,\omega^*_j)$, where $j:|\mathcal{A}_j|=1$: $C^*_{1,j}(X)=0$. More explicitly, we have 
    \begin{align*}
        \sum_{u=1}^{d}\phi_{3,j}^{(u)}\cdot X^{(u)}+\phi_{4,j}=0,
    \end{align*}
    for almost surely $X$. Since the set $\{X^{(u)},1:u\in[d]\}$ is linearly independent, the above equation implies that $\phi_{3,j}^{(u)}=\phi_{4,j}=0$ for any $j:|\mathcal{A}_j|=1$.
    \item $F^{(2)}(Y;X,\omega^*_j)$, where $j:|\mathcal{A}_j|=1$: $C^*_{2,j}(X)=0$, or equivalently, $\phi_{5,j}=0$ for any $j:|\mathcal{A}_j|=1$.
    \item $H(Y;X,\boj,\tau^*)$, where $j:|\mathcal{A}_j|>1$: $S^*_{j}(X)+\phi_{0,j}=0$. More explicitly, we have
    \begin{align*}
        &\phi_{0,j}+\Bigg\{\sum_{u=1}^{d}\frac{\phi^{(u)}_{1,j}}{\tau^*}\cdot \frac{\partial\bsigma}{\partial \beta_1^{(u)}}(X,\boj)-\frac{\phi_{2,j}}{(\tau^*)^2}\cdot\bsigma(X,\boj)-\sum_{u=1}^{d}\varphi_{1,j}^{(u)}\Big[\frac{1}{(\tau^*)^2}\cdot \frac{\partial\bsigma}{\partial \beta_1^{(u)}}(X,\boj)\nonumber\\
    &+\frac{1}{(\tau^*)^3}\cdot \bsigma(X,\boj)\frac{\partial\bsigma}{\partial \beta_1^{(u)}}(X,\boj)\Big]+\sum_{1\leq u,v\leq d}\frac{\varphi_{0,j}^{(u,v)}}{1+\mathbf{1}_{\{u=v\}}}\Big[\frac{1}{\tau^*}\cdot \frac{\partial^2\sigma}{\partial \beta_1^{(u)}\partial\beta_1^{(v)}}((\boj)^{\top}X)\nonumber\\
    &+\frac{1}{(\tau^*)^2}\cdot \Big(\frac{\partial\bsigma}{\partial \beta_1^{(u)}}(X,\boj)\Big)\Big(\frac{\partial\sigma}{\partial \beta_1^{(v)}}((\boj)^{\top}X)\Big)\Big]+\frac{1}{2}\varphi_{2,j}\Big[\frac{2}{(\tau^*)^3}\cdot\bsigma(X,\boj)+\frac{1}{(\tau^*)^4}\cdot \bsigma^2(X,\boj)\Big]\Bigg\}=0,
    \end{align*}
    for almost surely $X$. As the function $\sigma$ meets the conditions in Definition~\eqref{def:modified_function_over}, i.e. the set
    \begin{align*}
        &\Big\{1, \ \bsigma(X,\boj), \ \bsigma^2(X,\boj),\ \frac{\partial\bsigma}{\partial \beta_1^{(u)}}(X,\boj), \ \bsigma(X,\boj)\frac{\partial\bsigma}{\partial \beta_1^{(u)}}(X,\boj),\\
        &\hspace{3cm}\frac{\partial^2\sigma}{\partial \beta_1^{(u)}\partial\beta_1^{(v)}}((\boj)^{\top}X), \ \Big(\frac{\partial\bsigma}{\partial \beta_1^{(u)}}(X,\boj)\Big)\Big(\frac{\partial\sigma}{\partial \beta_1^{(v)}}((\boj)^{\top}X)\Big): u,v\in[d]\Big\}
    \end{align*}
    is linearly independent, the coefficients associated with $1$, $\bsigma^2(X,\boj)$, $\bsigma(X,\boj)\frac{\partial\bsigma}{\partial \beta_1^{(u)}}(X,\boj)$ and $\frac{\partial^2\sigma}{\partial \beta_1^{(u)}\partial\beta_1^{(v)}}((\boj)^{\top}X)$ must be zero, i.e. $\phi_{0,j}=\varphi_{2,j}=\varphi_{1,j}^{(u)}=\varphi_{0,j}^{(u,v)}=0$ for any $u,v\in[d]$ and $j:|\mathcal{A}_j|>1$. 
    \item $X^{\alpha_3}F^{(|\alpha_3|+\ell)}(Y;X,\boj,\tau^*)$, where $\alpha_3=2\mathbf{e}_{u}$, $\ell=0$, $j:|\mathcal{A}_j|>1$: $T^*_{0,\alpha_3,\ell,j}(X)+T^*_{1,\alpha_3,\ell,j}(X)+T^*_{2,\alpha_3,\ell,j}(X)=0$. More explicitly, we have
    \begin{align*}
        &\psi_{0,\alpha_3,\ell,j}+
        \sum_{u=1}^{d}\frac{\psi^{(u)}_{1,\alpha_3,\ell,j}}{\tau^*}\cdot \frac{\partial\bsigma}{\partial \beta_1^{(u)}}(X,\boj)-\frac{\psi_{2,\alpha_3,\ell,j}}{(\tau^*)^2}\cdot\bsigma(X,\boj),\\
        &+\sum_{1\leq u,v\leq d}\frac{\psi^{(u,v)}_{3,\alpha_3,\ell,j}}{1+\mathbf{1}_{\{u=v\}}}\Big[\frac{1}{\tau^*}\cdot \frac{\partial^2\sigma}{\partial \beta_1^{(u)}\partial\beta_1^{(v)}}((\boj)^{\top}X)+\frac{1}{(\tau^*)^2}\cdot \Big(\frac{\partial\bsigma}{\partial \beta_1^{(u)}}(X,\boj)\Big)\Big(\frac{\partial\sigma}{\partial \beta_1^{(v)}}((\boj)^{\top}X)\Big)\Big]\nonumber\\
    &\hspace{5.2cm}+\frac{1}{2}\psi_{5,\alpha_3,\ell,j}\Big[\frac{2}{(\tau^*)^3}\cdot\bsigma(X,\boj)+\frac{1}{(\tau^*)^4}\cdot \bsigma^2(X,\boj)\Big]\nonumber\\
    &\hspace{2.8cm}-\sum_{u=1}^{d}\psi^{(u)}_{4,\alpha_3,\ell,j}\Big[\frac{1}{(\tau^*)^2}\cdot \frac{\partial\bsigma}{\partial \beta_1^{(u)}}(X,\boj)+\frac{1}{(\tau^*)^3}\cdot \bsigma(X,\boj)\frac{\partial\bsigma}{\partial \beta_1^{(u)}}(X,\boj)\Big]=0,
    \end{align*}
    for almost surely $X$. Since the function $\sigma$ satisfies the conditions in Definition~\ref{def:modified_function_over}, we deduce that $\psi_{0,\alpha_3,\ell,j}=\psi_{0,2\mathbf{e}_{u},0,j}=0$, for any $u\in[d]$ and $j:|\mathcal{A}_j|>1$.
\end{itemize}
Gather the above results, we see that all elements in the set~\eqref{eq:non_vanish_elements} are equal to zero, which is a contradiction. Hence, we reach the conclusion of the theorem.

\section{Identifiability of the (Activation) Dense-to-sparse Gating Gaussian Mixture of Experts}
\label{sec:identifiability}
\begin{proposition}
    \label{prop:standard_identifiability}
    Assume that $G$ is a mixing measure in $\mathcal{O}_k(\Theta)$ that satisfy $g_{G}(Y|X)=g_{G_*}(Y|X)$ for almost surely $(X,Y)$. Then, we obtain that $G\equiv G_*(\lambda)$, where $G_*(\lambda):=\sum_{i=1}^{k_*}\exp(\bzi/\tau^*)\delta_{(\lambda\boi,\lambda\tau^*,\ai,\bi,\vi)}$, for some $\lambda\neq 0$.
\end{proposition}
\begin{proof}[Proof of Proposition~\ref{prop:standard_identifiability}]
    Firstly, let us recall that two mixing measures $G$ and $G_*$ admit the following forms:
    \begin{align*}
        G=\sum_{i=1}^{k'}\exp\Big(\frac{\beta_{0i}}{\tau}\Big)\delta_{(\beta_{1i},\tau,a_i,b_i,\nu_i)}, \qquad G_*=\sum_{i=1}^{k_*}\exp\Big(\frac{\bzi}{\tau^*}\Big)\delta_{(\boi,\tau^*,\ai,\bi,\vi)},
    \end{align*}
    where $k'\leq k$. Since $g_G(Y|X)=g_{G_*}(Y|X)$ for almost surely $(X,Y)$, we have
    \begin{align}
        \label{eq:identifiable_equation}
        \sum_{i=1}^{k}\softmax\Big(\frac{(\beta_{1i})^{\top}X+\beta_{0i}}{\tau}\Big)\cdot f(Y|a_i^{\top}X+b_i,\nu_i)=\sum_{i=1}^{k'}\softmax\Big(\frac{(\boi)^{\top}X+\bzi}{\tau^*}\Big)\cdot f(Y|(\ai)^{\top}+\bi,\vi).
    \end{align}
    As the mixture of location-scale Gaussian distributions is identifiable \cite{Teicher-1960,Teicher-1961,Teicher-63}, it follows that $k'=k_*$ and
    \begin{align*}
        \Big\{\softmax\Big(\frac{(\beta_{1i})^{\top}X+\beta_{0i}}{\tau}\Big):i\in[k']\Big\}=\Big\{\softmax\Big(\frac{(\boi)^{\top}X+\bzi}{\tau^*}\Big):i\in[k_*]\Big\},
    \end{align*}
    for almost surely $X$. WLOG, we may assume that 
    \begin{align}
        \label{eq:soft-soft}
        \softmax\Big(\frac{(\beta_{1i})^{\top}X+\beta_{0i}}{\tau}\Big)=\softmax\Big(\frac{(\boi)^{\top}X+\bzi}{\tau^*}\Big),
    \end{align}
    for almost surely $X$ for any $i\in[k_*]$. It is worth noting that the $\softmax$ function is invariant to translations, then equation~\eqref{eq:soft-soft} indicates that $\beta_{1i}/\tau=\boi/\tau^*+v_1$ and $\beta_{0i}/\tau=\bzi/\tau^*+v_0$ for some $v_1\in\mathbb{R}^d$ and $v_0\in\mathbb{R}$. However, from the assumptions $\beta_{1k}=\beta^*_{1k}=\zerod$ and $\beta_{0k}=\beta^*_{0k}=0$, we deduce that $v_1=\zerod$ and $v_0=0$. Consequently, we get that $\beta_{1i}/\tau=\boi/\tau^*$ and $\beta_{0i}/\tau=\bzi/\tau^*$ for any $i\in[k]$. Thus, we deduce that $\beta_{1i}=\lambda\boi$ and $\tau=\lambda\tau^*$, for some $\lambda\neq 0$.
    
    Then, equation~\eqref{eq:identifiable_equation} can be rewritten as
    \begin{align}
        \label{eq:new_identifiable_equation}
        \sum_{i=1}^{k_*}\exp\Big(\frac{\beta_{0i}}{\tau}\Big)\exp\Big(\frac{(\beta_{1i})^{\top}X}{\tau}\Big)f(Y|(a_i)^{\top}X+b_i,\nu_i)=\sum_{i=1}^{k_*}\exp\Big(\frac{\bzi}{\tau^*}\Big)\exp\Big(\frac{(\boi)^{\top}X}{\tau^*}\Big)f(Y|(\ai)^{\top}X+\bi,\vi),
    \end{align}
    for almost surely $(X,Y)$. Next, we denote $J_1,J_2,\ldots,J_m$ as a partition of the index set $[k_*]$, where $m\leq k$, such that $\exp(\beta_{0i}/\tau)=\exp(\beta^*_{0i'}/\tau^*)$ for any $i,i'\in J_j$ and $j\in[k_*]$. On the other hand, when $i$ and $i'$ do not belong to the same set $J_j$, we let $\exp(\beta_{0i}/\tau)\neq\exp(\beta_{0i'}/\tau^*)$. Thus, we can reformulate equation~\eqref{eq:new_identifiable_equation} as
    \begin{align*}
        \sum_{j=1}^{m}\sum_{i\in{J}_j}\exp\Big(\frac{\beta_{0i}}{\tau}\Big)\exp\Big(\frac{(\beta_{1i})^{\top}X}{\tau}\Big)f(Y|(a_i)^{\top}X+b_i,\nu_i)=\sum_{j=1}^{m}\sum_{i\in{J}_j}\exp\Big(\frac{\bzi}{\tau^*}\Big)\exp\Big(\frac{(\boi)^{\top}X}{\tau^*}\Big)f(Y|(\ai)^{\top}X+\bi,\vi),
    \end{align*}
    for almost surely $(X,Y)$. This results leads to $\{((a_i)^{\top}X+b_i,\nu_i):i\in J_j\}\equiv\{((\ai)^{\top}X+\bi,\vi):i\in J_j\}$, for almost surely $X$ for any $j\in[m]$. Therefore, we have
    \begin{align*}
        \{(a_i,b_i,\nu_i):i\in J_j\}\equiv\{(\ai,\bi,\vi):i\in J_j\},
    \end{align*}
    for any $j\in[m]$. As a consequence, 
    \begin{align*}
        G=\sum_{j=1}^{m}\sum_{i\in J_j}\exp\Big(\frac{\beta_{0i}}{\tau}\Big)\delta_{(\beta_{1i},\tau,a_i,b_i,\nu_i)}=\sum_{j=1}^{m}\sum_{i\in J_j}\exp\Big(\frac{\beta_{0i}}{\tau}\Big)\delta_{(\lambda\boi,\lambda\tau^*,\ai,\bi,\vi)}=G_*(\lambda).
    \end{align*}
    Hence, we reach the conclusion of this proposition.
\end{proof}

\begin{proposition}
    \label{prop:modified_identifiability}
    Assume that $G$ is a mixing measure in $\mathcal{O}_k(\Theta)$ that satisfy $p_{G}(Y|X)=p_{G_*}(Y|X)$ for almost surely $(X,Y)$. Then, we obtain that $G\equiv G_*$.
\end{proposition}
\begin{proof}[Proof of Proposition~\ref{prop:modified_identifiability}]
    Firstly, let us recall that two mixing measures $G$ and $G'$ admit the following forms:
    \begin{align*}
        G=\sum_{i=1}^{k'}\exp\Big(\frac{\beta_{0i}}{\tau}\Big)\delta_{(\beta_{1i},\tau,a_i,b_i,\nu_i)}, \qquad G_*=\sum_{i=1}^{k_*}\exp\Big(\frac{\bzi}{\tau^*}\Big)\delta_{(\boi,\tau^*,\ai,\bi,\vi)},
    \end{align*}
    where $k'\leq k$. Since $p_G(Y|X)=p_{G_*}(Y|X)$ for almost surely $(X,Y)$, we have
    \begin{align}
        \label{eq:identifiable_equation_modified}
        \sum_{i=1}^{k}\softmax\Big(\frac{\sigma((\beta_{1i})^{\top}X)+\beta_{0i}}{\tau}\Big)\cdot f(Y|a_i^{\top}X+b_i,\nu_i)=\sum_{i=1}^{k'}\softmax\Big(\frac{\sigma((\boi)^{\top}X)+\bzi}{\tau^*}\Big)\cdot f(Y|(\ai)^{\top}+\bi,\vi).
    \end{align}
    As the mixture of location-scale Gaussian distributions is identifiable \cite{Teicher-1960,Teicher-1961,Teicher-63}, it follows that $k'=k_*$ and
    \begin{align*}
        \Big\{\softmax\Big(\frac{\sigma((\beta_{1i})^{\top}X)+\beta_{0i}}{\tau}\Big):i\in[k']\Big\}=\Big\{\softmax\Big(\frac{\sigma((\boi)^{\top}X)+\bzi}{\tau^*}\Big):i\in[k_*]\Big\},
    \end{align*}
    for almost surely $X$. WLOG, we may assume that 
    \begin{align}
        \label{eq:soft-soft_modified}
        \softmax\Big(\frac{\sigma((\beta_{1i})^{\top}X)+\beta_{0i}}{\tau}\Big)=\softmax\Big(\frac{\sigma((\boi)^{\top}X)+\bzi}{\tau^*}\Big),
    \end{align}
    for almost surely $X$ for any $i\in[k_*]$. It is worth noting that the $\softmax$ function is invariant to translations, then equation~\eqref{eq:soft-soft_modified} indicates that $[\sigma((\beta_{1i})^{\top}X)+\beta_{0i}]/\tau=[\sigma((\boi)^{\top}X)+\bzi]/\tau^*+v$ for some $v\in\mathbb{R}$. However, from the assumptions $\sigma((\beta^*_{1k_*})^{\top}X)+\beta^*_{0k_*}=0$, we deduce $v=0$. Consequently, we get that
    \begin{align*}
        [\sigma((\beta_{1i})^{\top}X)+\beta_{0i}]/\tau=[\sigma((\boi)^{\top}X)+\bzi]/\tau^*,
    \end{align*}
    for almost surely $X$ for any $i\in[k]$. Thus, when $X=\zerod$, we deduce that $\beta_{0i}/\tau=\bzi/\tau^*$, which implies that $\sigma((\beta_{1i})^{\top}X)/\tau=\sigma((\boi)^{\top}X)/\tau^*$ for almost surely $X$. Again, when $X=\zerod$, since $\sigma(0)\neq 0$, we obtain that $\tau=\tau^*$, and therefore, $\beta_{1i}=\boi$ for any $i\in[k_*]$.
    
    Then, equation~\eqref{eq:identifiable_equation_modified} can be rewritten as
    \begin{align}
        \label{eq:new_identifiable_equation_modified}
        \sum_{i=1}^{k_*}\exp\Big(\frac{\beta_{0i}}{\tau}\Big)\exp\Big(\frac{\sigma((\beta_{1i})^{\top}X)}{\tau}\Big)f(Y|(a_i)^{\top}X+b_i,\nu_i)=\sum_{i=1}^{k_*}\exp\Big(\frac{\bzi}{\tau^*}\Big)\exp\Big(\frac{\sigma((\boi)^{\top}X)}{\tau^*}\Big)f(Y|(\ai)^{\top}X+\bi,\vi),
    \end{align}
    for almost surely $(X,Y)$. Next, we denote $J_1,J_2,\ldots,J_m$ as a partition of the index set $[k_*]$, where $m\leq k$, such that $\exp(\beta_{0i}/\tau)=\exp(\beta^*_{0i'}/\tau^*)$ for any $i,i'\in J_j$ and $j\in[k_*]$. On the other hand, when $i$ and $i'$ do not belong to the same set $J_j$, we let $\exp(\beta_{0i}/\tau)\neq\exp(\beta_{0i'}/\tau^*)$. Thus, we can reformulate equation~\eqref{eq:new_identifiable_equation_modified} as
    \begin{align*}
        &\sum_{j=1}^{m}\sum_{i\in{J}_j}\exp\Big(\frac{\beta_{0i}}{\tau}\Big)\exp\Big(\frac{\sigma((\beta_{1i})^{\top}X)}{\tau}\Big)f(Y|(a_i)^{\top}X+b_i,\nu_i)\\
        &\hspace{3cm}=\sum_{j=1}^{m}\sum_{i\in{J}_j}\exp\Big(\frac{\bzi}{\tau^*}\Big)\exp\Big(\frac{\sigma((\boi)^{\top}X)}{\tau^*}\Big)f(Y|(\ai)^{\top}X+\bi,\vi),
    \end{align*}
    for almost surely $(X,Y)$. This results leads to $\{((a_i)^{\top}X+b_i,\nu_i):i\in J_j\}\equiv\{((\ai)^{\top}X+\bi,\vi):i\in J_j\}$, for almost surely $X$ for any $j\in[m]$. Therefore, we have
    \begin{align*}
        \{(a_i,b_i,\nu_i):i\in J_j\}\equiv\{(\ai,\bi,\vi):i\in J_j\},
    \end{align*}
    for any $j\in[m]$. As a consequence, 
    \begin{align*}
        G=\sum_{j=1}^{m}\sum_{i\in J_j}\exp\Big(\frac{\beta_{0i}}{\tau}\Big)\delta_{(\beta_{1i},\tau,a_i,b_i,\nu_i)}=\sum_{j=1}^{m}\sum_{i\in J_j}\exp\Big(\frac{\beta_{0i}}{\tau}\Big)\delta_{(\boi,\tau^*,\ai,\bi,\vi)}=G_*.
    \end{align*}
    Hence, we reach the conclusion of this proposition.
\end{proof}
\section{Experimental Details}
\label{appendix:experiments}

\input{experiments}

\end{document}

%% file: math_commands.tex
\usepackage{eqnarray,amsmath}

\usepackage{color}

\usepackage{amsthm}
\usepackage{amsmath}
\usepackage{amssymb,bbm}
\usepackage[justification=centering]{caption}
\usepackage{algorithmic}
\usepackage{algorithm}
\usepackage{textcomp}
\usepackage{siunitx}
\usepackage{wrapfig}
\usepackage{algorithmic}
\usepackage{algorithm}
\usepackage{multirow}
\usepackage{multicol}

\newcommand{\bbE}{\mathbb{E}}

\newcommand{\softmax}{\mathrm{Softmax}}
\newcommand{\gelu}{\mathrm{GELU}}
\newcommand{\sigmoid}{\mathrm{sigmoid}}

\newcommand{\dint}{\mathrm{d}}

\newcommand{\dboijn}{\Delta \beta_{1ij}^{n}}

\newcommand{\dtn}{\Delta \tau^n}
\newcommand{\daijn}{\Delta a_{ij}^{n}}
\newcommand{\dbijn}{\Delta b_{ij}^{n}}
\newcommand{\dvijn}{\Delta \nu_{ij}^{n}}

\newcommand{\dboin}{\Delta \beta_{1i}^{n}}

\newcommand{\dain}{\Delta a_{i}^{n}}
\newcommand{\dbin}{\Delta b_{i}^{n}}
\newcommand{\dvin}{\Delta \nu_{i}^{n}}

\newcommand{\norm}[1]{\|#1\|}

\newcommand{\boin}{\beta_{1i}^n}
\newcommand{\bzin}{\beta_{0i}^n}
\newcommand{\ain}{a_i^n}
\newcommand{\bin}{b_i^n}
\newcommand{\vin}{\nu_i^n}

\newcommand{\boj}{\beta_{1j}^*}
\newcommand{\bzj}{\beta_{0j}^*}
\newcommand{\aj}{a_j^*}
\newcommand{\bj}{b_j^*}
\newcommand{\vj}{\nu_j^*}

\newcommand{\boi}{\beta_{1i}^*}
\newcommand{\bzi}{\beta_{0i}^*}
\newcommand{\ai}{a_i^*}
\newcommand{\bi}{b_i^*}
\newcommand{\vi}{\nu_i^*}

\newcommand{\zerod}{\mathbf{0}_d}

\newcommand{\brj}{\bar{r}_j}

\newcommand{\brone}{\bar{r}_1}

\newcommand{\dbione}{\Delta b_{i1}^{n}}
\newcommand{\dvione}{\Delta \nu_{i1}^{n}}

\newcommand{\bsigma}{\bar{\sigma}}

\DeclareMathOperator*{\argmax}{arg\,max}
\DeclareMathOperator*{\argmin}{arg\,min}

%% file: experiments.tex
In this appendix, we provide additional details regarding the experimental setups.

\subsection{Expectation-Maximization (EM) Algorithm}
\label{appendix:exp-alg}
Our approach for parameter estimation in the dense-to-sparse gating Gaussian MoE model employs an Expectation-Maximization (EM) algorithm. We present the Expectation-Maximization (EM) algorithm used for parameter estimation in the context of a Mixture of Experts (MoE) model. The derivation and formulation closely follow \cite{JORDAN19951409}.

The EM algorithm is an iterative optimization technique used for finding maximum likelihood estimates of parameters in models with latent variables. In the context of MoE, the latent variables correspond to the assignment of data points to specific experts. Denote the latent variables as $Z_{ik}$, indicating whether expert $k$ is responsible for the observation $i$, i.e.,
\begin{align*}
    Z_{ij} = 
    \begin{cases}
        1 & \text{if $Y_i$ is generated from the $j$th expert},\\
        0 & \text{Otherwise.}
    \end{cases}
\end{align*}
The algorithm alternates between two key steps: Expectation and Maximization.
\subsubsection{Expectation Step}
In the E-step, the algorithm estimates the posterior probabilities of the latent variables given the observed data and the current parameters. This involves calculating the responsibility of each expert for each data point. We denote the responsibility of the $j$th expert for $(X_i, Y_i)$ at iteration $t$ by
\begin{align*}
    \tau_{ij}^{(t)} 
    &= \mathbb{E}[Z_{i, j}|(X_i, Y_i); \Theta^{(t)}] \\
    &= \frac{\mathrm{Softmax}\left(\frac{\beta_{1j}^{(t)\top} X_i + \beta_{0j}^{(t)}}{\tau^{(t)}}\right) \cdot f\left(Y_i \Bigm\vert a_j^{(t)\top} X_i + b_j^{(t)}, \nu_j^{(t)}\right)}{\sum_{l = 1}^k \mathrm{Softmax}\left(\frac{\beta_{1l}^{(t)\top} X_i + \beta_{0l}^{(t)}}{\tau^{(t)}}\right) \cdot f\left(Y_i \Bigm\vert a_l^{(t)\top} X_i + b_l^{(t)}, \nu_l^{(t)}\right)}.
\end{align*}
Here, $\Theta^{(t)}$ represents the set of all parameters in the Mixture of Experts (MoE) model at iteration $t$.
\subsubsection{Maximization Step}
In the Maximization step, the model parameters are updated to maximize the expected log-likelihood obtained from the latent variable distribution derived in the Expectation step. This involves updating the parameters of both the expert models and the gating network based on the responsibilities computed in the E-step.

\textbf{M-step for expert parameters.} The update equations for expert parameters are given by:
\begin{align*}
    \nu_j^{(t+1)} &= \frac{1}{\sum_{i=1}^n \tau_{ij}^{(t)}} \sum_{i=1}^n \tau_{ij}^{(t)} \left(Y_i - (a_j^{(t)\top}X_i + b_j^{(t)})\right)^2, \\
    \theta_j^{(t+1)} &= \left(\sum_{i=1}^n \tau_{ij}^{(t)} \cdot \frac{\widetilde{X}_i \widetilde{X}_i^\top}{\nu_j^{(t)}}\right)^{-1} \left(\sum_{i=1}^n \tau_{ij}^{(t)} \cdot \frac{Y_i \widetilde{X}_i}{\nu_j^{(t)}}\right),
\end{align*}
where $\widetilde{X}_i^\top = (X_i^\top, 1)$, and $\theta_j^{(t)\top} = (a_j^{(t)\top}, b_j^{(t)})$ for $j = 1, 2, ..., k$.

\textbf{M-step for gating parameters.} Finally, the update for gating parameters can be viewed as a specific form of a generalized linear model, specifically a multinomial logit model, as observed by \cite{Jordan-1994}. Efficient fitting of such models is achieved through iteratively reweighted least squares (IRLS), a variant of Newton's method. First, let's denote $\theta_{0j}^{(t)\top} := (\beta_{1j}^{(t)\top}, \beta_{0j}^{(t)}, \tau^{(t)})$. Then, the update rule for the gating network parameters is given by:
\begin{align*}
    \theta_{0j}^{(t+1)} = \theta_{0j}^{(t)} + \eta (R_{j}^{(t)})^{-1} e_j^{(t)},
\end{align*}
where
\begin{align*}
    e_j^{(t)} &= \sum_{i=1}^n \sum_{j=1}^k \left(\tau_{ij}^{(t)} - \mathrm{Softmax}\left(\frac{\beta_{1j}^{(t)\top}X_i + \beta_{0j}^{(t)}}{\tau^{(t)}}\right)\right) g_{ij}^{(t)}, \\
    R_j^{(t)} &= \sum_{i=1}^n \sum_{j=1}^k \mathrm{Softmax}\left(\frac{\beta_{1j}^{(t)\top}X_i + \beta_{0j}^{(t)}}{\tau^{(t)}}\right) \left(1 - \mathrm{Softmax}\left(\frac{\beta_{1j}^{(t)\top}X_i + \beta_{0j}^{(t)}}{\tau^{(t)}}\right)\right) g_{ij}^{(t)}g_{ij}^{(t)\top},\\
    g_{ij}^{(t)} &= \nabla_{\theta_{0j}}\left(\frac{\beta_{1j}^{(t)\top}X_i + \beta_{0j}^{(t)}}{\tau^{(t)}}\right) = \left(X_i^\top, 1, -\frac{\beta_{1j}^{(t)\top}X_i + \beta_{0j}^{(t)}}{(\tau^{(t)})^2}\right)^\top.
\end{align*}

These two steps are iteratively repeated until convergence, providing a framework for estimating the model parameters that maximize the likelihood of the observed data. Furthermore, it is noteworthy to highlight that we choose the convergence criterion as $\epsilon = 10^{-6}$ and execute a maximum of $1000$ iterations for the EM algorithm, with an $100$ iterations for the Iteratively Reweighted Least Squares (IRLS) algorithm at each EM iteration, employing a learning rate of $\eta = 0.01$.

\subsection{Experimental Setup}
\textbf{Synthetic Data.}
We conducted experiments using synthetic datasets generated with the true mixing measure $G_{*}=\sum_{i=1}^{2}\exp(\bzi/\tau^*)\delta_{(\boi,\tau^*,\ai,\bi,\vi)}$ of order $k^* = 2$. We generated i.i.d samples $\{(X_i, Y_i)\}_{i=1}^n$ by initially sampling $X_i$ values from a zero-mean Gaussian distribution with unit variance, followed by sampling $Y_i$ values from the true conditional density $g_{G_*}(Y | X)$ of a Gaussian mixture with softmax gating, comprising $k_{*}=2$ experts:
\begin{align}
    \label{eq:standard_density_exp}
    g_{G_{*}}(Y | X) = \sum_{i=1}^{2}\softmax\Big(\dfrac{(\beta^*_{1i})^{\top}X+\beta^*_{0i}}{\tau^*}\Big)\cdot f(Y|(a^*_i)^{\top}X+b^*_i,\nu^*_i).
\end{align}
The values corresponding to the true parameters are detailed in Table~\ref{tab:true-params}.
\begin{table}[h]
    \centering
    \begin{tabular}{cccc}
    \toprule
                    & Gating parameters & 
                    \multicolumn{2}{c}{Expert parameters}\\
    \midrule
        Expert $1$  & $(\beta_{01}^{*}, \beta_{11}^{*}, \tau^{*}) = (0,-10, 0.5)$ & 
        
                   $(a_{1}^{*}, b_{1}^{*}, \nu_{1}^{*}) = (-1, 2, 0.3)$ &\\

        Expert $2$ & $(\beta_{02}^{*}, \beta_{12}^{*}, \tau^{*}) = (0, 0, 0.5)$ & $ (a_{2}^{*}, b_{2}^{*}, \nu_{2}^{*}) = (1, 2, 0.4)$&\\
    \bottomrule
    \end{tabular}
    \caption{Parameter values of the true model.}
    \label{tab:true-params}
\end{table}

\begin{figure}[!ht]
    \centering
    \includegraphics[scale=0.43]{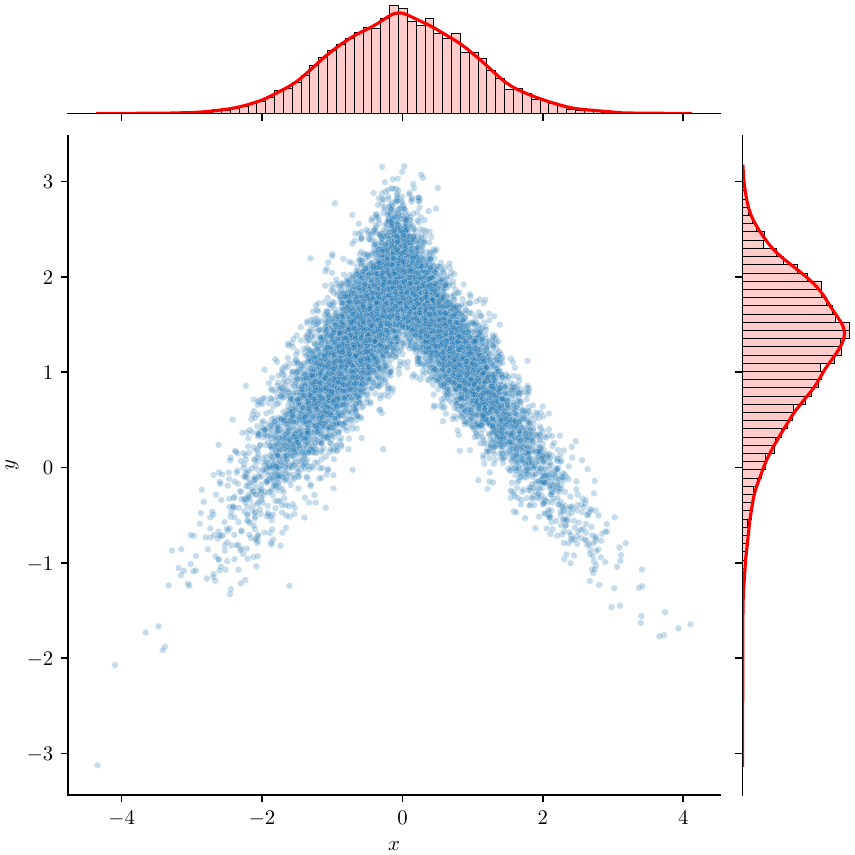}
    \caption{A visual illustration depicting the correlation between variables $X$ and $Y$, along with their individual marginal distributions.}
    \label{fig:regression-plot}
\end{figure}

\textbf{Initialization.}
For each $k\in\{k_{*},k_{*}+1\}$, elements from the set $\{1, 2, ..., k\}$ are randomly distributed among $k_{*}$ Voronoi cells denoted as $\mathcal{C}_1, \mathcal{C}_2, \ldots, \mathcal{C}_{k}$, ensuring that each cell contains at least one element. This process is repeated for each replication. Subsequently, for each $j\in[k]$, all parameters are initialized by sampling from a Gaussian distribution centered around their true counterparts with a small variance, where $i\in \mathcal{C}_j$.